\documentclass[10pt,fullpage,letterpaper]{article}

\usepackage[linesnumbered,ruled,vlined]{algorithm2e}
\usepackage{enumerate}
\usepackage{amsmath, amsfonts, amssymb}
\usepackage{amsthm}
\usepackage{mathrsfs}
\usepackage{bm}
\usepackage{graphicx}
\usepackage{alltt}
\usepackage{tikz}
\usepackage{float}
\usepackage{booktabs}
\usepackage{hyperref}
\usepackage{array}
\usepackage{dsfont}
\usepackage{nicefrac}
\usepackage{comment}
\usepackage{anyfontsize}
\usepackage{longtable}
\usepackage{natbib}
\usepackage{wrapfig}
\allowdisplaybreaks

\DeclareMathOperator*{\argmax}{argmax}

\newcommand{\Sp}[1]{\left(#1\right)}
\newcommand{\Mp}[1]{\left[#1\right]}
\newcommand{\Bp}[1]{\left\{#1\right\}}
\newcommand{\abs}[1]{\left|#1\right|}
\newcommand{\Norm}[1]{\left\|#1\right\|}
\newcommand{\ve}[1]{\mathbf{#1}}

\newcommand{\inner}[1]{\left\langle#1\right\rangle}

\newcommand{\E}{\mathbb{E}}
\renewcommand{\P}{\mathbb{P}}
\newcommand{\Pp}{\mathcal{P}}
\renewcommand{\S}{\mathcal{S}}
\newcommand{\R}{\mathbb{R}}
\newcommand{\Rr}{\mathcal{R}}
\newcommand{\A}{\mathcal{A}}
\newcommand{\G}{\mathcal{G}}

\newcommand{\bV}{\overline{V}}

\newcommand{\hsa}{\left(h, s, a\right)}
\newcommand{\bbV}{\mathbb{V}}
\newcommand{\ogk}{\mathds{1}\left\{\mathcal{G}_k\right\}}
\newcommand{\sumogk}{\sum_{k=1}^{K}\sum_{h=1}^{H-1}\mathds{1}\left\{\mathcal{G}_k\right\}}
\newcommand{\Be}{\mathrm{Be}}
\newcommand{\Ho}{\mathrm{Ho}}
\newcommand{\ty}{\mathrm{ty}}
\newcommand{\ob}{\overline{b}}
\newcommand{\ub}{\underline{b}}

\newtheorem{theorem}{Theorem}
\newtheorem{lemma}{Lemma}
\newtheorem{definition}{Definition}

\newcommand{\Oo}{\mathcal{O}}
\newcommand{\hz}{\hat{z}}
\newcommand{\bv}{\overline{V}}
\newcommand{\ov}{\overline{V}}
\newcommand{\uv}{\underline{V}}

\newcommand{\oh}{\overline{\mathcal{H}}}

\newcommand{\pik}{{\pi^k}}
\newcommand{\hk}{{h,k}}

\newcommand{\hsahk}{{h,s^k_h,a^k_h}}

\newcommand{\uw}{\underline{w}}
\newcommand{\od}{\overline{\delta}}
\newcommand{\ud}{\underline{\delta}}

\newcommand{\hR}{\hat{R}}
\newcommand{\hP}{\hat{P}}
\newcommand{\M}{\mathcal{M}}
\newcommand{\um}{\underline{M}}

\newcommand{\upm}{\overline{\overline{M}}}
\newcommand{\upv}{\overline{\overline{V}}}
\newcommand{\upw}{\overline{w}}
\newcommand{\upd}{\overline{\overline{\delta}}}

\newcommand{\clip}{\mathrm{clip}}

\newcommand{\algonamefull}{Single Seed Randomization}
\newcommand{\algoname}{\textsf{SSR}}

\usepackage{color}
\definecolor{ForestGreen}{rgb}{0.1333,0.5451,0.1333}
\hypersetup{colorlinks,
	linkcolor=ForestGreen,
	citecolor=ForestGreen,
	urlcolor=black,
	linktocpage,
	plainpages=false}

\title{Near-Optimal Randomized Exploration for Tabular Markov Decision Processes}
\author{Zhihan Xiong\footnote{Equal contribution}\\ \url{zhihanx@cs.washington.edu} \and Ruoqi Shen\footnotemark[\value{footnote}] \\ \url{shenr3@cs.washington.edu} \and Qiwen Cui\footnotemark[\value{footnote}] \\ \url{qwcui@cs.washington.edu} \and Maryam Fazel \\ \url{mfazel@uw.edu} \and Simon S. Du \\\url{ssdu@cs.washington.edu}}
\date{}

\begin{document}
\maketitle

\renewcommand*{\thefootnote}{\arabic{footnote}}

\begin{abstract}
%This paper shows randomized least-square value iteration (RLSVI) is near-optimal for time-inhomogeneous tabular Markov Decision Processes.
We study algorithms using randomized value functions for exploration in reinforcement learning. This type of algorithms enjoys appealing empirical performance. 
We show that when we use 1) a single random seed in each episode, and 2) a Bernstein-type magnitude of noise, we obtain a worst-case $\widetilde{O}\left(H\sqrt{SAT}\right)$ regret bound for episodic time-inhomogeneous Markov Decision Process where $S$ is the size of state space, $A$ is the size of action space, $H$ is the planning horizon and $T$ is the number of interactions. 
This bound polynomially improves all existing bounds for algorithms based on randomized value functions, and for the first time, matches the $\Omega\left(H\sqrt{SAT}\right)$ lower bound up to logarithmic factors. 
Our result highlights that randomized exploration can be near-optimal, which was previously achieved only by optimistic algorithms. 
To achieve the desired result, we develop 1) a new clipping operation to ensure both the probability of being optimistic and the probability of being pessimistic are lower bounded by a constant, and 2) a new recursive  formula for the absolute value of estimation errors to analyze the regret.

% \simon{New main analysis techniques: 1. both optimism and pessimism, 2.a new recursion of the differences terms..}
% \qiwen{The analysis contains a new clipping technique, the first pessimism argument in online RL and novel recursion formulas on the absolute value of the estimation, which might shed new light on analyzing randomized exploration algorithms.}
%We consider regret minimization of a classical Thompson Sampling (TS)-based algorithm, randomized least-square value iteration (RLSVI), for reinforcement learning under tabular finite-horizon Markov Decision Process. Previously, a high-probability worst-case regret of $\tilde{O}(H^2S\sqrt{AT})$ was obtained by introducing clipping technique to RLSVI. In this paper, we further propose two new modifications, uniform random source and Bernstein-like magnitude of noise, to Clipping-RLSVI and obtain a high-probability worse-case regret of $\tilde{O}(H\sqrt{SAT})$ when $T$ is sufficiently large. This result matches the lower bound up to a logarithmic factor and to the best of our knowledge, this is the first provably optimal TS-based algorithm for tabular reinforcement learning.

\end{abstract}

\section{Introduction}
%episodic MDP, OFU
This paper concerns learning in tabular Markov Decision Processes (MDP), arguably the most fundamental model for reinforcement learning (RL).
Existing algorithms that achieve the near-optimal minimax $\widetilde{O}\left(H\sqrt{SAT}\right)$ regret bound are based on the principle of  \emph{Optimism in the face of Uncertainty} (OFU), such as upper confidence bound (UCB)~\citep{azar2017minimax,zanette2019tighter,dann2019policy,zhang2020model,zhang2020reinforcement}.\footnote{This bound is for time-inhomogeneous MDP with each reward bounded by $1$ and $T$ is sufficiently large. }
Here $S$ is the number of states, $A$ is the number of actions, $H$ is the planning horizon, and $T$ is the total number of interactions between the agent and the environment.

Another broad category is algorithms with randomized exploration such as Thompson Sampling~\citep{osband2013more,agrawal2017optimistic,osband2014generalization}.
These algorithms inject (carefully tuned) random noise to  value function to encourage exploration. UCB-type algorithms enjoy well-established theoretical guarantees but suffer from difficult implementation since an upper confidence bound is usually infeasible for many practical models like neural networks.
Instead, practitioners prefer randomized exploration such as noisy networks in \cite{fortunato2018noisy}, and algorithms with randomized exploration have been widely used in practice  ~\citep{osband2017deep,chapelle2011empirical,burda2018exploration,osband2018randomized}. However, how to design randomized exploration algorithms in a principled way and perform randomized exploration optimally is far from clear. 
% \ruoqi{'perform randomized exploration in a principled way' what's unclear?} 
While randomized exploration can have great performance in practice, theoretically, the best known worst-case regret bound for algorithms with randomized exploration is $\widetilde{O}\left(H^{2}S\sqrt{AT}\right)$ \citep{agrawal2020improved}, which is far worse than that of the UCB-type algorithms.
In this paper, we introduce a new randomized exploration algorithm and show it enjoys a near-optimal $\widetilde{O}\left(H\sqrt{SAT}\right)$ worst-case regret bound, thus closing the gap. 
Our work sheds new light on randomized exploration on both the algorithmic side and the theoretical side. 
% }
% \simon{Add some reference: e.g., noisy network?}
%
\paragraph{Our Contributions.}
Our contributions are summarized below:

% \begin{itemize}
$\bullet$ We propose a new algorithm, \algonamefull~(\algoname), which incorporates a crucial  algorithmic idea: using a single random seed for the entire episode, in contrast to previous methods of randomized exploration which use one seed for each time step. \algoname~is able to explore more efficiently than previous methods by avoiding having noise at different time steps canceling with each other.  
Theoretically, we show, thanks to this new idea, if one uses a Hoeffding-type magnitude of noise, \algoname~achieves an $\widetilde{O}\left(H^{1.5}\sqrt{SAT}\right)$ regret bound, improving upon the best existing result on randomized exploration algorithm~\citep{agrawal2020improved}.

$\bullet$ We further design a new Bernstein-type magnitude of noise for our algorithm, and achieve an $\widetilde{O}\left(H\sqrt{SAT}\right)$ regret bound, resolving an open problem raised in \cite{agrawal2020improved}.
To our knowledge, this is the first time that a Bernstein-type bound is used in randomized exploration.
More importantly, our upper bound matches the $\Omega\left(H\sqrt{SAT}\right)$ minimax lower bound up to logarithmic factors.

We note that our goal is not to show randomized exploration is better than optimistic algorithms \citep{azar2017minimax} in the tabular setting. Instead, we aim to provide a solid theoretical understanding of a practically relevant algorithm. Indeed, understanding randomized exploration itself is an important theoretical research direction and has attracted much interest in the community \citep{agrawal2012analysis, agrawal2017near, agrawal2017posterior, russo2019worst, zanette2020frequentist, pmlr-v108-vaswani20a, agrawal2020improved, osband2013more, osband2014generalization, osband2017deep, osband2018randomized}.

% Therefore, our result conveys an important conceptual message:
% \begin{center}
% \textbf{Randomized exploration can be near-optimal in reinforcement learning.}
% \end{center}
% Previously, near-optimal regret bounds can only be achieved by UCB-based approaches.
% \end{itemize}

\paragraph{Main Challenge and Technical  Overview.}
% \qiwen{
Besides the aforementioned algorithmic ideas (single random seed and Bernstein-type magnitude of noise), we also need additional ideas in analysis to prove the desired regret bound.
The main challenge is that unlike UCB-type algorithms, the estimated value in algorithm with randomized exploration, is not an upper bound of the true optimal value. This leads to the failure of directly utilizing their analysis, which only need to analyze the one-sided error in estimation. 
We instead work on the \emph{absolute value} of the estimation error, whose analysis is more complicated than that for the one-sided error in UCB-type algorithms. 
Working with absolute value forces us to ensure that both the probability that the estimated value is optimistic and the probability that the estimated value is pessimistic are lower bounded.
However, the clipping strategy in existing algorithm cannot maintain pessimism.
% Another challenge is that previous clipping strategy cannot maintain pessimism or the number of clipping happening can not be bounded.
To tackle with this issue, we develop a new clipping method. Below we list our technical contributions.

% \begin{enumerate}
% \item 
{\bf 1.} First, we propose a new clipping strategy to constrain the estimated value function (cf. Eqn.~\eqref{eq:clipping}). Previous clipping strategies in \citep{zanette2020frequentist,agrawal2020improved} are based on uncertainty and can only maintain optimism. Our clipping strategy directly works on the value function, which is similar to those used in UCB-type algorithms \citep{azar2017minimax,jin2018q,zhang2020model}. Our clipping strategy can maintain both the optimism and pessimism. In addition, the number of times that the clipping is used can still be bounded. 

% \item 
{\bf 2.} Second, we prove that the single seed randomization ensures that the estimated value function can both be optimistic or pessimistic with constant probability at all states and timesteps.
This is stronger than previous randomized exploration algorithms that are only shown to be optimistic at the initial state with constant probability. With this property, we can then bound the difference between the optimal value function and estimated value function from both above and below, which results in a bound on its absolute value. See Section~\ref{sec:po_concentration}, Appendix~\ref{sec:optimism} and Appendix~\ref{sec:pessimism}.

% \item 
{\bf 3.} Third, we prove a novel recursion argument on the absolute value of the policy estimation error. As mentioned in \citep{agrawal2020improved}, the recursion in UCB-type algorithms can not be directly utilized because our estimated value function is not a high-probability upper bound of the true optimal value function. With the bound of absolute value, we are able to prove new recursion formulas and together we can control the policy estimation error. See Section~\ref{sec:po_regdecomp} and Appendix~\ref{sec:regret_decomp}.

% \item 
{\bf 4.} At last, we bound the sum of variance in a novel manner. In \citep{azar2017minimax}, the UCB-type estimation guarantees that the policy estimation error is always positive so the difference of the variance can be directly bounded. We generalize the argument to the absolute value of the estimation error to bound the sum of variance. See Section~\ref{sec:po_sumvar} and Appendix~\ref{sec:sum_variance}.

% \end{enumerate}

% }

\section{Related Work}
\label{sec:related}
%ucb-based algorithm
In this section we review existing provably efficient algorithms for tabular MDP.
There is a long list of sample complexity guarantees for tabular MDP~\citep{kearns2002near,brafman2002r,kakade2003sample,strehl2006pac,strehl2008analysis,kolter2009near,bartlett2009regal,jaksch2010near,szita2010model,lattimore2012pac,osband2013more,dann2015sample,azar2017minimax,dann2017unifying,osband2017posterior,agrawal2017optimistic,jin2018q,fruit2018near,talebi2018variance,dann2019policy,dong2019q,simchowitz2019non,russo2019worst,zhang2019regret,cai2019provably,zhang2020almost,yang2020q,pacchiano2020optimism,neu2020unifying,zhang2020reinforcement,wang2020long,agrawal2020improved,russo2019worst,agrawal2017optimistic,domingues2021episodic,menard2021ucb,li2021breaking}.
The state-of-the-art methods are based on upper confidence bound (UCB)~\citep{azar2017minimax,zanette2019tighter,dann2019policy,zhang2020model,zhang2020reinforcement,menard2021ucb,li2021breaking}.
For the setting considered in this paper where the transition is  time-inhomogeneous and the reward is bounded by $1$, one can achieve an $\widetilde{O}\left(H\sqrt{SAT}\right)$ in the regime where $T$ is sufficiently large.

%TS-algorithm
%bandit
Algorithms with randomized exploration have been proved to enjoy favorable regret bounds in bandit problems~\citep{lai1985asymptotically,agrawal2012analysis,kaufmann2012thompson,bubeck2014prior,agrawal2017near}.
In certain settings, randomized exploration can match the worst-case regret bound of UCB-based approaches and achieve nearly minimax optimal regret bounds~\citep{jin2020mots,agrawal2017near}.
%rl
However, for RL, existing theory for randomized exploration are far from optimal~\citep{agrawal2020improved,russo2019worst,agrawal2017optimistic,xu2019worst,zanette2020frequentist}.
For the setting considered in this paper, the sharpest existing regret bound among algorithms with randomized exploration is 
$\widetilde{O}\left(H^2S\sqrt{AT}\right)$ proved in \citep{agrawal2020improved}.
% Comparing with the state-of-the-art UCB-based methods and the lower bound, there is still a $\widetilde{\Omega}\left(H\sqrt{S}\right)$ gap. \qiwen{In addition, we found that their proof is wrong because they misused the Bellman backup on the clipped value function. Besides this bound, the only existing bound is
% $\widetilde{O}\Sp{H^{5/2}S^{3/2}\sqrt{AT}}$.
% which is far from optimal.}
Our paper closes this gap and thus deepens our understanding about randomized exploration.
% \simon{Check ICML 2021 reviews.}

\section{Preliminaries}
\label{sec:setting}
We consider time-inhomogeneous finite-horizon MDP $M=\Sp{H, \mathcal{S}, \A, P, R, s_1}$,  where $\abs{\mathcal{S}}=S$ and $\abs{\A}=A$. Here, $\mathcal{S} = \{1,...,S\}$ is the finite state space. $\A = \{1,...,A\}$ is the finite action space. $H$ is the length of an episode. For convenience, we take $s_1$ to be the fixed initial state, although a more general initial distribution will not change the conclusion.  $P:  \S \times\A \times [H] \rightarrow \triangle(\S)$ is the transition function, where if the agent stays at state $s$ and takes action $a$ at time $h$, it transits to state $s'$ with probability $P_{h,s,a}(s')\in\Mp{0, 1}$. $R: \S \times\A \times [H] \rightarrow [0,1]$ is the reward function, where if the agent stays at $s$ and takes action $a$ at time $h$, it will receive reward $r_{h, s, a}\in\Mp{0, 1}$ such that $\E\Mp{r_{h, s, a}}=R_{h, s, a}$.

A deterministic policy for such a MDP is defined as a tuple $\pi=\Sp{\pi_1, \dots, \pi_H}$, where $\pi_h:\mathcal{S}\mapsto\A$. The associated value function at state $s\in\mathcal{S}$ and level $h\in\Bp{1, \dots, H}$ is recursively defined as
$$V_h^{\pi}\Sp{s}=R_{h, s, \pi_h\Sp{s}}+\sum_{s'\in\mathcal{S}}P_{h, s, \pi_h(s)}\Sp{s'}V^{\pi}_{h+1}\Sp{s'}.$$

For convenience, we set $V^{\pi}_{H+1}=\ve{0}\in\R^S$. The corresponding optimal value function is $V^{*}_h\Sp{s}=\max_{\pi\in\Pi}V^{\pi}_h\Sp{s}$, where $\Pi$ is the set of all possible deterministic policies. For a particular algorithm $\mathsf{Alg}$, let $\pi^k$ denote the policy that $\mathsf{Alg}$ employs during episode $k$. Then, the regret of running $\mathsf{Alg}$ on MDP $M$ for $K$ episodes is defined as
\begin{align}
\mathrm{Reg}\Sp{M, K, \mathsf{Alg}}=\sum_{k=1}^{K}\Sp{V_1^{*}(s_1)-V_1^{\pi^k}(s_1)}. \label{eq:regret}
\end{align}
Note that the regret, $\mathrm{Reg}\Sp{M, K, \mathsf{Alg}}$, is a random variable due to randomness in state transition and the algorithm, $\mathsf{Alg}$. In this paper, we show the regret of our proposed algorithm can be upper bounded with high probability, and the upper bound matches the known lower bound up to logarithmic factors. 
% Unlike many previous works which only studies the expectation of the regret, our bound studies high probability bound in the worst case and is strictly stronger than expectation bounds. 

To facilitate our later analysis, we introduce some notations for empirical estimation. At episode $k$, we collect a trajectory $(s_1^k,a_1^k,r_1^k,\cdots,s_H^k,a_H^k,r_H^k)$ as specified in Algorithm \ref{algo:cub-rlsvi}. Let $n_k\Sp{h, s, a}=\sum_{l=1}^{k-1}\mathds{1}\{(s_h^l, a_h^l)=(s, a)\}$ be the number of times action $a$ is taken at state $s$ and time $h$ before episode $k$, where $\mathds{1}\Bp{\cdot}$ is the indicator function. We define
\begin{align}
    \hat{R}^k_{h, s, a} & =\frac{\sum_{l=1}^{k-1}\mathds{1}\{(s_h^l, a_h^l)=(s, a)\}r^l_{h, s_h^l, a_h^l}}{n_k\Sp{h, s, a}+1},\label{eq:premR}\\
   \hat{P}^k_{h, s, a}\Sp{s'} &=\frac{\sum_{l=1}^{k-1}\mathds{1}\{(s_h^l, a_h^l, s_{h+1}^l)=(s, a, s')\}}{n_k\Sp{h, s, a}+1}.\label{eq:premP}
\end{align}

Then, define empirical MDP based on our observation and estimation before episode $k$ as the tuple $\hat{M}^k=(H, \S, \A, \hat{P}^k, \hat{R}^k, s_1)$.  
%Note that the denominator $n_k\Sp{h, s, a}+1$ is uncommon for estimation, whose meaning will be clear when we introduce the model-based view of our algorithm.
Since $\hat{P}^k_{h, s, a}$ is not a valid distribution over $\mathcal{S}$, for being rigorous, we can imagine there is an additional virtual absorbing state that every state will transit to with remaining probability.
% We further define
% \begin{align*}
% \tilde{P}^k_{h, s, a}\Sp{s'} &=\frac{\sum_{l=1}^{k-1}\mathds{1}\{(s_h^l, a_h^l, s_{h+1}^l)=(s, a, s')\}}{\max\Bp{n_k\Sp{h, s, a}, 1}},
% \end{align*}
% which will be used as the estimated transition when computing the random value perturbation in the algorithm. 

In addition to the above notations, let $\widetilde{O}\Sp{\cdot}, \widetilde{\Theta}\Sp{\cdot}$ and $\widetilde{\Omega}\Sp{\cdot}$ be asymptotic notations ignoring all poly-logarithmic terms. For distribution $D\in\Delta^S$ and value function $V\in\R^S$, let $\bbV\Sp{D, V}$ denote the variance of $V$ under distribution $D$, which is defined as $\bbV\Sp{D, V}=\sum_{s\in\S}D(s)\Sp{V(s)-\inner{D, V}}^2$. For constant $a>0$, we define the corresponding clipping function as $\clip_a(\cdot)=\max\Bp{-a,\min\Bp{a,\cdot}}$. Immediately we have $\abs{\clip_a(x)}\leq a$ for any $a>0$. We introduce the definitions of other notations when used. In appendix, we summarize the notations and definitions used in this paper. 

\section{Main Results}
\label{sec:main_result}

\subsection{Algorithm}

\begin{algorithm}[tb]
    \caption{\algonamefull~(\algoname)}
	\label{algo:cub-rlsvi}
	\SetAlgoLined
	\KwIn{$H, S, A$, perturbation type $\mathrm{ty}\in\Bp{\Ho, \Be}$}
	\For{episodes $k=1, 2, \dots, K$}{
	    Sample $\hat{z}_k\sim\mathcal{N}\Sp{0, 1}$\\
	    Define terminal value function $\overline{Q}_{H+1, k}=\ve{0}\in\R^{SA}$ and $\overline{V}_{H+1, k}=\ve{0}\in\R^S$\\
	    \For{time periods $h=H, \dots, 1$}{
	        $\overline{Q}_{h, k}\Sp{s, a}\leftarrow\hat{R}^k_{h, s, a}+\inner{\hat{P}^k_{h, s, a}, \overline{V}_{h+1, k}}+\sigma^k_{\mathrm{ty}}\Sp{h, s, a}\hat{z}_k$\tcp*{\textcolor{red}{$\sigma^k_{\mathrm{ty}}\Sp{h, s, a}$ is defined in \eqref{equ:def_ho_noise} and \eqref{equ:def_be_noise}.}}
	        Define $\overline{V}_{h, k}\Sp{s}=\clip_{2(H-h+1)}\Sp{\max_{a\in\A}\overline{Q}_{h, k}(s, a)}$ for all $s\in\S$
	    }
	    Agent takes actions $a_h^k=\argmax_{a\in\A}\overline{Q}_{h, k}(s_h^k, a)$ throughout this episode\\
	    Observe data $s_1^k, a_1^k, r_1^k, \dots, s_H^k, a_H^k, r_H^k$ and compute $\hat{R}^{k+1}_{h, s, a}$, $\hat{P}^{k+1}_{h, s, a}$ and $n_{k+1}\Sp{h, s, a}$ for all $\Sp{h, s, a}\in\Mp{H}\times\S\times\A$
	}
\end{algorithm}

The main contribution of this paper is that we show algorithm with randomized value functions can achieve regret that matches the known lower bound $\Omega\Sp{H\sqrt{SAT}}$ \citep{jaksch2010near,domingues2021episodic} up to logarithmic factors in the tabular setting. To facilitate exploration, this
type of algorithms uses random value perturbation instead of deterministic bonus. The algorithm we consider is summarized in Algorithm \ref{algo:cub-rlsvi}. In our algorithm, \algoname, the random perturbation ensures that optimism/pessimism can be obtained with constant probability in each episode. Moreover, randomized value function has its origin from posterior sampling for reinforcement learning (Thompson sampling). The randomized perturbation can be interpreted as approximate sampling from the posterior distribution of the value function on randomized training data \citep{russo2019worst}. 

We first give an overview of \algoname. In Algorithm \ref{algo:cub-rlsvi}, the policy used at episode $k$ is computed using the empirical MDP, $\hat{M}^k=(H, \S, \A, \hat{P}^k, \hat{R}^k, s_1)$, which is based on observation and estimation before episode $k$. However, instead of directly choosing optimal policy for $\hat{M}^k$, we add a small random perturbation when computing the value of each state and action pair. To be more precise, at each episode $k$, we first estimate the reward and transition function for each state $s$ and action $a$ based on \eqref{eq:premR} and \eqref{eq:premP}. Then, we compute the value function for state $s$ and action $a$,
\begin{align*}
    \overline{Q}_{h, k}\Sp{s, a}\leftarrow\hat{R}^k_{h, s, a}+\inner{\hat{P}^k_{h, s, a}, \overline{V}_{h+1, k}}+\sigma^k_{\mathrm{ty}}\Sp{h, s, a}\hat{z}_k.
\end{align*}
Here, $\hat{z}_k\sim \mathcal{N}(0, 1)$ is a standard Gaussian random variable sampled once every episode. The magnitude of the perturbation, $\sigma^k_{\mathrm{ty}}$ depends on how many samples $n_k(h, s, a)$ we have observed and how confident we are on the estimations $\hR_{h,s,a}^k$ and $\hP^k_{h,s,a}$. We will discuss more about the choice of the magnitude later in this section.

% \qiwen{
In order to prevent estimated value function from behaving badly, we add a clipping to the value function:
% \small
\begin{equation}
    \label{eq:clipping}
    \begin{split}
        \ov_{h,k}(s)&=\clip_{2(H-h+1)}\Sp{\max_{a\in\A}\overline{Q}_{h,k}(s,a)}
        % &=\max\Bp{-2(H-h+1),\min\Bp{2(H-h+1),\max_{a\in\A}\overline{Q}_{h,k}(s,a)}}.
    \end{split}
\end{equation}
% \begin{align}
% \ov_{h,k}(s)=\clip_{2(H-h+1)}\Sp{\max_{a\in\A}\overline{Q}_{h,k}(s,a)}=\max\Bp{-2(H-h+1),\min\Bp{2(H-h+1),\max_{a\in\A}\overline{Q}_{h,k}(s,a)}}. \label{eq:clipping}
% \end{align}
% \normalsize
As our analysis will show, this kind of clipping can bound the value function, maintain optimism and pessimism and also guarantee that clipping will not happen for a lot of times.
% }
% \qiwen{
The constant $2$ (instead of $1$) plays a crucial role because it means the value function grows at an additive rate of $2$ from $h=H$ to $h=1$. 
If we do not consider the added noise, then the value function should at most grow $1$ at each timestep because the reward is at most $1$. For our clipping technique, if a clip is triggered, there exists a timestep such that the added noise is more than $1$, which is equivalent to a small number of visits (cf. Definition~\ref{def:event_no_clip} and Lemma~\ref{lemma:q_bar_bound}).
% }
% \simon{ @Qiwen: add a sentence and refer a lemma about why constant $2$ is important here.}
\begin{comment}
In order to enforce explicit optimism and prevent estimated value function from behaving badly at early stage, we add a clipping to $\overline{Q}_{h, k}$ so that the value is set to $H-h+1$ when the number of sample observed $n^k(h, s, a)$ is smaller than $\alpha_k$, a technique first introduced into TS-like algorithm in \cite{zanette2020frequentist}. In \algoname, we take 
\begin{equation}
    \label{equ:alpha_k}
    \alpha_k=200H^2\log\Sp{2HSAk^2}\log\Sp{40k^4}=\widetilde{\Theta}\Sp{H^2}.
\end{equation}
% $\alpha_k  = \tilde{\Theta}\Sp{H^2}$.
\end{comment}
As our later analysis will show, the clipping only affects the lower-order term and will not compromise the long-term performance of the algorithm. Finally, after computing the value function and clipping, \algoname\ chooses the action $a_h^k$ that maximizes $\overline{Q}_{h, k}(s_h^k,a)$ at each time step, $h=1, ..., H$, throughout the episode.

Note that from a Bayesian perspective, when there is no clipping, in Algorithm \ref{algo:cub-rlsvi}, $\overline{Q}_{h, k}$ follows distribution
\begin{align*}
    \overline{Q}_{h, k}(s, a) \mid \overline{V}_{h+1, k} \sim \mathcal{N}\Sp{\hat{R}^k_{h, s, a}+\inner{\hat{P}^k_{h, s, a}, \overline{V}_{h+1, k}}, \Sp{\sigma^k_{\mathrm{ty}}(h, s, a)}^2}.
\end{align*}
This resembles posterior sampling because when estimating some parameter $\theta^* \sim \mathcal{N}\Sp{0,\beta^2}$ based on noisy observations $\theta_1,...,\theta_n \sim \mathcal{N}(\theta,\beta^2)$, the posterior distribution of $\theta^*$ given $\{\theta_i\}_{i=1}^n$ is $\theta^* \mid \{\theta_i\}_{i=1}^n \sim \mathcal{N}\Sp{\frac{1}{n+1}\sum_{i=1}^n\theta_i, \frac{\beta^2}{n+1}}$. Although exact posterior sampling may not be possible in complex reinforcement learning settings, in \algoname, $\sigma^k_{\mathrm{ty}}(h, s, a)$ is chosen at scale $\widetilde{\Theta}\Sp{1/\sqrt{n_k(h,s,a)}}$ and therefore can be interpreted as doing approximate posterior sampling. Moreover, \algoname\ can be viewed as a variant of Randomized Least Square Value Iteration (RLSVI). The major differences are at the clipping function and a single random seed used in each episode instead of different random seeds at different tuples $(h, s, a)$. We will discuss more about the choice of the random seed later in this section. We refer to \cite{osband2017deep} and \cite{russo2019worst} for a more detailed discussion on the relationship among RLSVI, posterior sampling and randomized value function. 

In the following paragraphs, we discuss in more details about the three major algorithmic innovations:

\paragraph{Single Random Seed in Each Episode.} \algoname\ is similar to the algorithms analyzed in \cite{russo2019worst} and \cite{agrawal2020improved}. The major difference is that in the algorithm we propose, we use a single random seed $\hz_k$ to generate the perturbations for all time steps $h = 1,...,H$ in an episode $k$. 

When using different random seeds in an episode, the algorithm can be optimistic in some time step while being pessimistic in others. Then, the effects of the perturbations at different time steps will cancel with each other. As a result, to ensure sufficient exploration, the magnitude of the perturbation has to large.
% can be optimistic enough. 
This issue was also pointed out in \cite{agrawal2020improved,abeille2017linear}.

A large perturbation magnitude can increase the instability of the algorithm and worsen the algorithm's performance. When a single random seed is used, a small perturbation magnitude is enough to guarantee that the algorithm is optimistic with constant probability in any episode. We are able to show that using a single random seed can significantly increase the stability of the algorithm and therefore enjoy much smaller regret. Coincidentally, \cite{pmlr-v108-vaswani20a} also uses a similar single randomization in bandit problems to build a near-optimal randomized exploration algorithm and our work can be treated as its natural extension to RL problems.

\paragraph{Clipping.} To obtain a tight regret bound, the estimated value function needs to be well bounded. In \citep{russo2019worst}, no clipping is used and the estimated value function is at the order of $\widetilde{O}(H^{5/2}S)$, which results in a suboptimal regret bound. Generally there are two types of clipping methods. The first one is uncertainty-based, i.e. the value is clipped to $H-h+1$ at timestep $h$ whenever the uncertainty is large \citep{zanette2020frequentist,agrawal2020improved}. However this type of clipping cannot maintain pessimism which is critical in our analysis. The other kind of clipping is value-based, mostly in UCB-type algorithms \citep{jin2019provably}. These algorithms truncate estimated value greater than a certain threshold, i.e. $H-h+1$ at time step $h$. The problem here is that the number of clippings cannot be bounded because if the true value function is close to $H-h+1$ at timestep $h$, the clipping will happen with some constant probability.

Our clipping method leverages both type of clipping methods in the existing literature. Though our clipping is based on the value function, we show that whenever the clip is triggered, the estimation error must be large, which implies that the uncertainty at that state is large. This clipping method inherits the desired properties from both uncertainty-based and value-based clipping, i.e. the optimism/pessimism is maintained and the number of clippings can be bounded.

\paragraph{Magnitude of Perturbation.} 
% For the magnitude of perturbation $\sigma^k_{\mathrm{ty}}\Sp{h, s, a}$, we have two choices. 
A large magnitude of perturbation can encourage exploration, but at the same time increase instability. In our algorithm, the magnitudes are chosen as the smallest values so that the algorithm can be optimistic with constant probability. Since the value function can roughly be bounded by $O(H)$, a naive choice of the perturbation magnitude can be $\Theta\Sp{H/\sqrt{n_k(h,s,a)}}$. In this way, by Hoeffding's inequality, as long as the random Gaussian variable sampled $\hz_k$ is bigger than a constant, which happens with constant probability, the estimated value function will be optimistic. By similar reasoning, we can see that the estimated value function will also be pessimistic with constant probability.

To make the magnitude even smaller, inspired by \citep{azar2017minimax} who showed one can use an (empirical) Bernstein's inequality to derive a sharp exploration bonus for UCB-based algorithms, we propose a new choice of perturbation magnitude based on Bernstein's inequality. The Bernstein-based perturbation uses the empirical variance of the value function, which makes it smaller than the Hoeffiding-based one mostly, but still maintains optimism with constant probability.

In our paper, we study both types of magnitudes. In particular, we show that the regret of \algoname\ based on Bernstein's inequality matches the known lower bound $\Omega\Sp{H\sqrt{SAT}}$. Following are the two choices:
\fontsize{8.5}{8.5}
\begin{equation}
    \label{equ:def_ho_noise}
    \sigma^k_{\mathrm{Ho}}\Sp{h, s, a}=H\sqrt{\frac{\log\Sp{2HSAk^2}}{n_k\Sp{h, s, a}+1}}+\frac{H}{n_k\Sp{h, s, a}+1},
\end{equation}
\begin{equation}
    \label{equ:def_be_noise}
    \begin{split}
        \sigma^k_{\mathrm{Be}}\Sp{h, s, a}=\sqrt{\frac{16\mathbb{V}\Sp{\tilde{P}^k_{h, s, a}, \overline{V}_{k, h+1}}\log\Sp{2HSAk^2}}{n_k\Sp{h, s, a}+1}}+\frac{65H\log\Sp{2HSAk^2}}{n_k\Sp{h, s, a}+1}+\sqrt{\frac{\log\Sp{2HSAk^2}}{n_k\Sp{h, s, a}+1}},
    \end{split}
\end{equation}
\normalsize
where subscript ``Ho'' represents that the perturbation is based on Hoeffding's inequality and ``Be'' represents Bernstein's inequality, correspondingly. Here, for proof convenience, $\tilde{P}^k_{h, s, a}$ is defined by replacing the denominator in $\hat{P}^k_{h, s, a}$ by $\max\Bp{n_k(h, s, a), 1}$.
% Here, $\bbV\Sp{P, V}$ is the variance of $V$ under distribution $P$, $\bbV\Sp{P, V}=\sum_{s\in\S}P(s)\Sp{V(s)-\inner{P, V}}^2$. 
To clarify, when subscript ``ty'' is used, which stands for ``type'' as a placeholder for ``Ho'' or ``Be'', it means that there is no need to write two copies of expressions for Hoeffding-based and Bernstein-based noises separately.

\paragraph{Practical Considerations.}
Here, we explain why randomized exploration is widely used in practice and why our algorithmic formulation practically has advantage over UCB-type algorithms.
In randomized exploration, there are usually two important components: (1) the algorithm (e.g., Algorithm \ref{algo:cub-rlsvi}) and (2) the noise magnitude ($\sigma_{\ty}$). In practice, the main advantage of randomized exploration lies in the algorithm component. The generalization from the tabular setting to the function approximation setting is straightforward: one can just add a random regularization term in the value estimation step, whose details can be found in \citep{osband2018randomized}. On the other hand, the generalization of optimistic algorithms from the tabular setting to the function approximation setting is more non-trivial because it often requires an explicit construction of the confidence set. For the second component, although generalizing our strategy of tuning noise magnitude to the real-world function approximation setting is indeed not straightforward, it is often set as a hyper-parameter in practice.

% \begin{comment}
% Further, $\alpha_k$ is chosen to be
% \begin{equation}
%     \label{def:alpha_k}
%     \alpha_k=200H^2\log(2HSAk^2)\log(40K/\delta).
% \end{equation}
% From a model-based view, as discussed in \cite{agrawal2020improved}, when the clipping step does not happen in episode $k$, the greedy policy of $\hat{Q}^k_h\Sp{s, a}$ forms an optimal policy of the perturbed empirical MDP $\overline{M}^k_{\mathrm{ty}}=(H, \S, \A, \hat{P}^k, \hat{R}^k+w^k_{\mathrm{ty}}, s_1)$, where $w^k_{\mathrm{ty}}\Sp{h, s, a}=\sigma^k_{\mathrm{ty}}\Sp{h, s, a}\hat{z}_k$ depending on the type of perturbation. Notice that $\hat{z}_k$ keeps the same for all tuple $\Sp{h, s, a}$. 
% \end{comment}

\subsection{Regret Analysis}
We analyze the regret, defined in \eqref{eq:regret}, of our algorithm \algoname\ using both types of perturbations. Our main theorems are presented in Theorem \ref{theo:main_1} and \ref{theo:main_2}. 
% The regret bounds we get are high probability bounds and are strictly stronger than expectation bound (by choosing $\delta \leq \frac{1}{\sqrt{T}}$). 
In particular, Theorem~\ref{theo:main_2} shows \algoname\ with Bernstein-based perturbation can achieve the regret that matches the known lower bound $\Omega\Sp{H\sqrt{SAT}}$ up to logarithmic factors. We sketch the proof of Theorem~\ref{theo:main_1} and Theorem~\ref{theo:main_2} in Section~\ref{sec:proofoutline}.

\begin{theorem}
	\label{theo:main_1}
	If the Hoeffding-type noise \eqref{equ:def_ho_noise} is used, then for any MDP $M=\Sp{H, \S, \A, P, R, s_1}$, with probability at least $1-\delta$, Algorithm \ref{algo:cub-rlsvi} satisfies
	$$\mathrm{Reg}(M, K, \mathsf{\algoname}_\Ho)\leq\widetilde{O}\Sp{H^{1.5}\sqrt{SAT}+H^4S^2A}.$$
	In particular, when $T\geq\widetilde{\Omega}\Sp{H^5S^3A}$, it holds that $\mathrm{Reg}(M, K, \mathsf{\algoname}_\Ho)\leq\widetilde{O}\Sp{H^{1.5}\sqrt{SAT}}$.
\end{theorem}

\begin{theorem}
	\label{theo:main_2}
	If the Bernstein-type noise \eqref{equ:def_be_noise} is used, then when $T\geq\widetilde{\Omega}\Sp{H^5S^2A}$, for any MDP $M=\Sp{H, \S, \A, P, R, s_1}$, with probability at least $1-\delta$, Algorithm \ref{algo:cub-rlsvi} satisfies
	$$\mathrm{Reg}(M, K, \mathsf{\algoname}_\Be)\leq\widetilde{O}\Sp{H\sqrt{SAT}+H^4S^2A}.$$
	In particular, if we further have $T\geq\widetilde{\Omega}\Sp{H^6S^3A}$, it then holds that $\mathrm{Reg}(M, K, \mathsf{\algoname}_\Be)\leq\widetilde{O}\Sp{H\sqrt{SAT}}$.
\end{theorem}

We give a brief comparison between \algoname\ and other related works. \cite{russo2019worst} shows that RLSVI, an algorithm similar to \algoname, can achieve $\tilde{O}\Sp{H^{2.5}S^{1.5}\sqrt{AT}}$ regret in expectation over the randomness of MDP and the algorithm. 
% \cite{agrawal2020improved} introduced clipping to RLSVI and proposed an algorithm C-RLSVI, which can achieve regret $\tilde{O}\Sp{H^{2}S\sqrt{AT}}$ with high probability.
In \citep{agrawal2020improved}, an improved high probability regret bound $\widetilde{O}\left(H^2S\sqrt{AT}\right)$ is proposed, which is the sharpest bound for randomized algorithms prior to this work.
% In addition, \cite{zanette2020frequentist} showed a $\tilde{O}\Sp{H^2d^2\sqrt{T}}$ frequentist regret bound of RLSVI with linear function approximation. Converting their bound to the tabular setting directly will give a $\tilde{O}\Sp{H^2S^2A^2\sqrt{T}}$ bound. 
Our paper closes the gap between those previous bounds and the lower bound in tabular setting. 
% That is, we show randomized exploration can be near-optimal.

We also run numerical simulations to empirically compare \algoname\ and RLSVI in the deep-sea environment, which is commonly used as a benchmark to test an algorithm's ability to explore. The results show that \algoname\ significantly outperforms RLSVI as predicted by our regret analysis. More details about our experiment can be found in Appendix \ref{sec:simulation}.
\section{Proof Outline}
\label{sec:proofoutline}

In this section, we present an proof outline of Theorem \ref{theo:main_1} and \ref{theo:main_2}. Since their proofs follow the same framework, we will present an unified outline and explain the individual steps particularly for each case when necessary. The details of complete proof are deferred to the appendix.

% We present the first step in Section~\ref{sec:po_concentration}, which shows the concentration property of the estimated value functions based on $\hR$ and $\hP$ and the perturbations. We show that in most time, the deviation of the estimated functions from the true ones can be bounded by the perturbation magnitude. Therefore, \algoname\ can be both optimistic and pessimistic with constant probability at each episode. Moreover, we show that the perturbation is bounded with high probability.  

% In Section~\ref{sec:po_regdecomp}, we decompose the regret into the pessimism term and the estimation error term. While the pessimism term studies the regret induced by the algorithm's underestimation of the optimal value, the estimation error term studies the difference between policy $\pi^k$'s estimated value and true value. We bound these two terms separately.

% Both pessimism term and estimation error term are reduced to a sum of perturbation magnitude and related lower-order terms in Section~\ref{sec:po_combine}. The key difference between Hoeffding-type and Bernstein-type perturbation is that the latter one requires to bound the sum of variance of the value functions. In Section~\ref{sec:po_sumvar}, we address this additional technical difficulty by bounding the sum of the variance, which completes the last piece of the proof.

\paragraph{Notation} For the ease of exposure, we will use a simplified notations during this sketch. Specifically, let $x=\hsa$ and $x_h^k=(\hsahk)$.

\subsection{Concentration and Optimism/Pessimism}
\label{sec:po_concentration}
We start by introducing a set of MDPs $\M^k_\ty$ as a confidence set such that the empirical MDP $\hat{M}^k$ belongs to it with high probability, meaning that we have a good estimation of the true MDP. Specifically, with $M'=\Sp{H, \S, \A, P', R', s_1}$, we define
\[\M^k_\ty:=\Bp{M':\forall x=\hsa,\abs{\Sp{R'_{x}-R_{x}}+\inner{P'_x-P_x, V^*_{h+1}}}\leq\sqrt{e^k_\ty(x)}},\]
where $\sqrt{e^k_\Ho(x)}=\sigma^k_\Ho(x)$ and $\sqrt{e^k_\Be(x)}\approx\sigma^k_\Be(x)$.
% where $\mathrm{err}(x)=\Sp{R'_{x}-R_{x}}+\inner{P'_x-P_x, V^*_{h+1}}$ and
% \begin{equation}
% 	\label{equ:e_simple}
% 	\sqrt{e^k_\Ho(x)}=\sigma^k_\Ho(x),\quad\sqrt{e^k_\Be(x)}\approx\sigma^k_\Be(x).
% \end{equation}
% Here, since $\sqrt{e^k_\Be(x)}$ and $\sigma^k_\Be(x)$ are approximately in the same order, we use ``$\approx$'' to avoid technical details. 

Define the event $\mathcal{C}^k_\ty:=\Bp{\hat{M}^k\in\M^k_\ty}$. Then, by applying Hoeffding's inequality or Bernstein's inequality, for both types of perturbation, it is possible to show that
$$\sum_{k=1}^{\infty}\P\Sp{\Sp{\mathcal{C}^k_\ty}^c}=\sum_{k=1}^{\infty}\P\Sp{\hat{M}^k\notin\M^k_\ty}\leq\frac{\pi^2}{3}.$$
Since the value function is bounded in $\Mp{0, H}$, this inequality tells us that the regret incurred by bad estimation is at most $\widetilde{O}\Sp{H}$. To be precise, it holds with high probability that
\begin{align}
	\sum_{k=1}^{K}\mathds{1}\Bp{\Sp{\mathcal{C}^k_\ty}^c}\Sp{V^*_{1}-V_{1, k}^\pik}(s_1^k)\leq\widetilde{O}\Sp{H}.\label{equ:regret_part1}
\end{align}

% Note that although Gaussian randomness in value function encourages exploration, its actual effect can both be optimistic or pessimistic because Gaussian random variable takes values from the whole real line. Therefore, conceptually, a mechanism that prevents Gaussian noise from behaving badly will help improve the algorithm and clipping is such a mechanism. From a technical viewpoint, clipping effectively bounds the value function and allows tighter concentration arguments.

% \qiwen{
Then, to better control the estimated value function, we need it to be bounded, which requires us to clip it.
Specifically, we will use two crucial properties of our clipping method. First, if $\overline{Q}_{h,k}(s,a)\geq Q_h^*(s,a)$, $\forall (s,a)\in\S\times\A$, then we have $\ov_{h,k}(s)\geq V_h^*(s),\forall s\in\S$. Similarly if $\overline{Q}_{h,k}(s,a)\leq Q_h^*(s,a)$, $\forall (s,a)\in\S\times\A$, then we have $\ov_{h,k}(s)\leq V_h^*(s),\forall s\in\S$. 

In addition, we can prove that whenever a clip is triggered for $s_h^k$, we have $n_k(h,s_h^k,a_h^k)\leq\alpha_k$ with $\alpha_k=\widetilde{O}(H^2)$. 
% The proof is deferred to the appendix. The high level intuition is that the value function is bounded by $2(H-h)$ at timestep $h+1$, so the estimated value function without noise $\sigma_\ty^k(h,s,a)\hat{z}_k$ should be bounded by $2(H-h)+1$ because the reward at timestep $h$ is bounded by 1. So if the value function at $s_h^k$ is clipped, we have that $\abs{\sigma_\ty^k(h,s,a)\hat{z}_k}\geq 1$, which implies that $n_k(h,s_h^k,a_h^k)$ is small. 
As a result, it is possible to show that the total regret incurred by clipping is at most $\widetilde{O}\Sp{H^4SA}$, which is a lower-order term when $T$ is sufficiently large. That is, let $\mathcal{E}^{cum}_{H, k}$ denote the event that there is no clipping during episode $k$. Then, it holds with high probability that\footnote{Technically, this is not precisely how we bound the regret incurred by clipping, but it aligns better with the intuition. Full technical details can be found in Appendix.}
% \fontsize{9}{9}
\begin{equation}
	\label{equ:regret_part2}
	\sum_{k=1}^{K}\mathds{1}\Bp{\mathcal{C}^k_\ty\cap\Sp{\mathcal{E}^{cum}_{H, k}}^c}\Sp{V^*_{1}-V_{1, k}^\pik}(s_1^k)\leq\widetilde{O}\Sp{H^4S^2A}.	
\end{equation}
% }

% \begin{comment}
% Specifically, clipping not only enforces explicit optimism at early stage, but also makes estimated value function $\overline{Q}_{h, k}$ accurate when there is no clipping. In other words, when $\hat{M}^k\in\M^k_\ty$ and no clipping happens, it holds with high probability that $|\overline{Q}_{k, h}(s, a)-Q^*_h(s, a)|\leq H-h+1$ for any tuple $\hsa$. This property of being bounded will help  in later regret analysis. Moreover, since $\alpha_k=\widetilde{O}\Sp{H^2}$, it is possible to show that the total regret incurred by clipping is at most $\widetilde{O}\Sp{H^4SA}$, which is a lower-order term when $T$ is sufficiently large. To be precise, let $\mathcal{E}^{th}_k$ denote the event that there is no clipping during episode $k$. Then, it holds with high probability that
% % \fontsize{9}{9}
% \begin{equation}
% 	\label{equ:regret_part2}
% 	\sum_{k=1}^{K}\mathds{1}\Bp{\mathcal{C}^k_\ty\cap\Sp{\mathcal{E}^{th}_k}^c}\Sp{V^*_{1}-V_{1, k}^\pik}(s_1^k)\leq\widetilde{O}\Sp{H^4S^2A}.	
% \end{equation}
% % \normalsize
% \end{comment}

As claimed before, because of the randomness in Gaussian noise, our algorithm \algoname\ will encourage exploration and it takes effect when there is no clipping and the estimation is not too bad. In other words, it can be optimistic. However, also because of this randomness, its optimism only holds in a probabilistic sense. In precise, it is possible to show that
% To be precise, by further defining the good event $\mathcal{G}_k\subseteq\mathcal{C}^k_\ty\cap\mathcal{E}^{th}_k$ that holds with high probability, it is possible to show that 
% if $\hat{M}^k\in\M^k_\ty$ and event $\mathcal{E}^{th}_k$ holds, then for any $h\in\Mp{H}$, we have
\begin{equation}
	\label{equ:stochastic_opt}
	\P\Sp{\overline{V}_{h, k}(s)\geq V^*_{h}(s), \forall h\in[H], s\in\S\mid\mathcal{C}^k_\ty}\geq C_\ty,
\end{equation}
where the value of constant $C_\ty$ depends on the type of noise we choose. Meanwhile, we can also prove a very similar probabilistic pessimism, which means to have $\overline{V}_{h, k}(s)\leq V^*_{h}(s), \forall h\in[H], s\in\S$ with constant probability. The property of optimism and pessimism will help us upper bound the absolute value of $V^*_{1}(s_1^k)-\overline{V}_{1, k}(s_1^k)$, which will be discussed soon. 
% \qiwen{Similarly, we can prove that the estiamted value is pessimistic with constant probability, which can help us lower bound $V^*_{1}(s_1^k)-\overline{V}_{1, k}(s_1^k)$.}

% We note that the chosen magnitude of noise is close to the confidence width $\sqrt{e^k_\ty(x)}$ as specified in equation \eqref{equ:e_simple}. Therefore, it is intuitive to consider that when $\hat{z}_k$ happens to not be too small, which holds with constant probability for Gaussian, the estimated value function $\overline{V}_{h, k}$ will be optimistic as long as the estimation is not too bad. From this perspective, we can also see the advantage of using single random source in a episode over using independent random sources for all different $(h, s, a)$. The reason is that when independent random sources are used, we need most of them to be not too small for being optimistic, which will hold with small probability since independence makes probability to multiplicate. Therefore, to achieve optimism, a larger magnitude of noise is necessary and it thus contributes to a larger regret at the end.

\subsection{Regret Decomposition}
\label{sec:po_regdecomp}
Now, given equations \eqref{equ:regret_part1} and \eqref{equ:regret_part2}, we can see that for each episode $k$, it only remains to bound $\mathds{1}\Bp{\mathcal{C}^k_\ty\cap\mathcal{E}^{cum}_{H, k}}\Sp{V^*_{1}-V^\pik_{1, k}}(s_1^k)$. Technically, the further defined the good event $\mathcal{G}_k$ will help make $\overline{V}_{h, k}$ better-behaved. Its precise definition will be given in the appendix. Therefore, it is sufficient to bound $\ogk\Sp{V^*_1-V^\pik_{1, k}}(s_1^k)$, which means to have
\begin{equation}
	\label{equ:simple_regret_decom}
	\mathrm{Reg}\Sp{M, K, \algoname_\ty}\leq \sum_{k=1}^K\ogk(\underbrace{V^*_1-\overline{V}_{1, k}}_{\text{pessimism}}+\underbrace{\overline{V}_{1, k}-V^\pik_{1, k}}_{\text{estimation error}})(s_1^k)+\widetilde{O}\Sp{H^4SA}.
\end{equation}

To proceed, we need to define two auxiliary value functions $\uv_{h, k}$ and $\upv_{h, k}$, which are obtained by virtually running policy $\pi^k$ on some deliberately perturbed MDPs. In particular, they are designed such that $\uv_{h, k}\leq\overline{V}_{h, k}\leq\upv_{h,k}$ holds under the good event $\mathcal{G}_k$.

\paragraph{Pessimism Term} Here, as a technical novelty, we bound the pessimism term's absolute value. Meanwhile, different from \cite{zanette2020frequentist} and \cite{agrawal2020improved}, by applying both optimism and pessimism, we do not resort to an independent copy of the perturbed MDP to bound the pessimism term and give a conceptually simpler analysis. In particulary, by defining $C_1=1/\min\Bp{C_\Ho, C_\Be}$. it is possible to show that
\fontsize{9}{9}
\begin{equation}
\label{equ:simple_pessi}
    \ogk\abs{V^*_{h,k}(s_h^k)-\ov_{h,k}(s_h^k)}  \leq \ogk C_1\Sp{\abs{\upv^\pik_{h,k}(s_h^k)-V^\pik_{h,k}(s_h^k)}+\abs{\uv^\pik_{h,k}(s_h^k)-V^\pik_{h,k}(s_h^k)}}.
\end{equation}
\normalsize
The full proof is given in Appendix under Lemma \ref{lem:pessimism}.

% \subsubsection{The Estimation Error Term}
% Note that all the three kinds of $\delta$ are the value difference under the sample policy $\pik$ but different underlying MDP and we call them (policy) estimation error. 
% For simplicity, we will use $\M_{h, k}$ to denote arbitrary martingale difference sequence (MDS) terms at period $h$, episode $k$, without specifying them precisely.

\paragraph{Estimation Error Term} The sum of pessimism term and estimation error term can be further bounded via the techniques of recursion used in \cite{azar2017minimax}. However, we want to emphasize the difference that in their algorithm, the estimated value is optimistic with high probability, which makes $\overline{V}_{h, k}(s_h^k)-V^*_{h}(s_h^k)$ always positive. Instead, since our optimism only holds with constant probability, we use absolute value to keep the estimation error terms positive. As a result, we show that
\fontsize{9}{9}
\begin{equation}
    \label{equ:simple_esti}
    \begin{split}
        \abs{\ov_{1, k}-V^\pik_{1, k}}(s_1^k)+\abs{\upv^\pik_{1,k}-V^\pik_{1,k}}(s_1^k)+\abs{\uv^\pik_{1,k}-V^\pik_{1,k}}(s_1^k)\lesssim & e^{3C}\sum_{h=1}^{H}\Sp{L\sigma^k_\ty(x_h^k)+\M_{h, k}},
    \end{split}
    % \Sp{\abs{\od^\pik_{1, k}}+\abs{\upd^\pik_{1, k}}+\abs{\ud^\pik_{1, k}}}(s_1^k)\lesssim e^{3C}\sum_{h=1}^{H}\Sp{\sqrt{e^k_\ty}+L\sigma^k_\ty}(x_h^k)+\sum_{h=1}^{H}\M_{h, k},
\end{equation}
\normalsize
where $L$ denotes some poly-logarithmic term and $\M_{h, k}$ denotes some martingale difference sequence term at period $h$, episode $k$. The full proof is given in Appendix under Lemma \ref{lem:sum of delta},

% \fontsize{8}{8}

\begin{comment}
\begin{equation}
	\label{equ:simple_esti}
	\begin{split}
		&\ogk(\ov_{1, k}-V^\pik_{1, k})(s_1^k)\leq\widetilde{O}\Sp{H\sqrt{T}} +e^{3C}\sum_{h=1}^{H}\ogk\Sp{\sqrt{e^k_\ty}+L\sigma^k_\ty}(x_h^k).
	\end{split}
\end{equation}
% \normalsize

As for correction term, since $\uv_{h, k}$ is obtained by running $\pik$ on MDP $\underline{M}^k_\ty$, which means that $\uv_{h, k}(s_h^k)=\underline{Q}_{h, k}(s_h^k, a_h^k)$, we can obtain
% \fontsize{8}{8}
\begin{equation}
	\label{equ:simple_correc}
	\begin{split}
		&\ogk(V^\pik_{1, k}-\uv_{1, k})(s_1^k)\leq \widetilde{O}\Sp{H\sqrt{T}}+\sum_{h=1}^{H}\ogk\Sp{\sqrt{e^k_\ty}+L\sigma^k_\ty+\ub^k}(x_h^k),
	\end{split}
\end{equation}
% \normalsize
where $\ub^k(x_h^k)$ is defined through replacing $\ov$ by $\uv$ in $\ob^k(x_h^k)$.
\end{comment}

\subsection{Combining Different Terms}
\label{sec:po_combine}

By combining equations \eqref{equ:simple_regret_decom}, \eqref{equ:simple_pessi} and \eqref{equ:simple_esti} and applying concentration inequalities to MDP $\M_{h, k}$, it is possible to show that
\fontsize{9.5}{9.5}
\begin{equation}
\label{equ:reg}
    \mathrm{Reg}\Sp{M, K, \mathsf{\algoname}_\ty}\leq e^{3C_1}\sum_{k=1}^{K}\sum_{h=1}^{H}\ogk L\sigma^k_\ty(x_h^k)+\widetilde{O}\Sp{H\sqrt{T}+H^4S^2A}.
\end{equation}
\normalsize

Then, a final high-probability regret bound can be obtained by summing each individual terms over $k, h$ separately. It is well-known among literature that
\begin{align}
	\sum_{k=1}^{K}\sum_{h=1}^{H}\sqrt{\frac{1}{n_k(x_h^k)+1}}\leq\widetilde{O}\Sp{\sqrt{HSAT}},\qquad \sum_{k=1}^{K}\sum_{h=1}^{H}\frac{1}{n_k(x_h^k)+1}\leq\widetilde{O}\Sp{HSA}.\label{equ:inequ_2}
\end{align}
Recall the definition of $\sigma^k_\Ho$ in equation \eqref{equ:def_ho_noise}. By using these two inequalities, the bound in equation \eqref{equ:reg} can be made explicit if we use Hoeffding-type noise. As a result, we have
% $$\sum_{k=1}^{K}\ogk\Sp{V^*_1-V^\pik_{1, k}}(s_1^k)\leq\widetilde{O}\Sp{H^{1.5}\sqrt{SAT}}.$$
% Combining the results in equations \eqref{equ:regret_part1} and \eqref{equ:regret_part2}, we can conclude if Hoeffding-based noise is used, it holds with high probability that
$$\mathrm{Reg}\Sp{M, K, \mathsf{\algoname}_\Ho}\leq\widetilde{O}\Sp{H^{1.5}\sqrt{SAT}+H^4S^2A}.$$

\subsubsection{Bound on Sum of Variance}
\label{sec:po_sumvar}
Analyses become more involved when Bernstein-type noise is used. Specifically, notice that inequalities in \eqref{equ:inequ_2} cannot directly be used to bound $\sum_{k, h}\bbV\Sp{\tilde{P}^k_{x_h^k}, \ov_{h+1, k}}$. Here, we apply some techniques developed in \cite{azar2017minimax}. However, since the optimism only holds with constant probability, the details for specific terms are quite different.

% \paragraph{Notations} For the ease of exposure, we will ignore all constants and define
% \begin{center}
%     \begin{tabular}{ll}
% 	\toprule
% 	$\hat{\mathbb{V}}^*_{h, k}=\mathbb{V}\Sp{\tilde{P}^k_{x_h^k}, V^*_{h}}$ & $\hat{\overline{\mathbb{V}}}_{h, k}=\mathbb{V}\Sp{\tilde{P}^k_{x_h^k}, \overline{V}_{h, k}}$\\
% 	$\mathbb{V}^{\pi^k}_{h, k}=\mathbb{V}\Sp{P_{x_h^k}, V^{\pi^k}_{h, k}}$ & $\hat{\bbV}_{h, k}=\hat{\bbV}^*_{h, k}+\hat{\overline{\bbV}}_{h, k}$\\
% 	$\mathbb{W}=\sum_{k, h}\ogk\hat{\bbV}_{h, k}$ & $\delta^\pik_{h, k}=V^*_h-V^\pik_{h, k}$\\
% %	$U_{h, k, 1}=\sqrt{\frac{\hat{\bbV}_{h, k}L}{n_k(x_h^k)+1}}$ & $U_{h, k, 2}=\sqrt{\frac{\hat{\overline{\bbV}}_{h, k}L}{n_k(x_h^k)+1}}$\\
% %	$U_{h, k, 3}=\sqrt{\frac{\hat{\underline{\bbV}}_{h, k}L}{n_k(x_h^k)+1}}$ & $U=\sum_{k, h}\ogk\sum_{j=1}^{3}U_{h, k, j}$\\
% % 	$\delta^\pik_{h, k}=V^*_h-V^\pik_{h, k}$ & $\ud^\pik_{h, k}=V^\pik_{h, k}-\uv_{h, k}$\\
% 	\bottomrule 
% \end{tabular}
% \end{center}

For the ease of exposure, we will ignore all constants and define $\hat{\mathbb{V}}^*_{h, k}=\mathbb{V}\Sp{\tilde{P}^k_{x_h^k}, V^*_{h}}$, $\hat{\overline{\mathbb{V}}}_{h, k}=\mathbb{V}\Sp{\tilde{P}^k_{x_h^k}, \overline{V}_{h, k}}$. Then, by using Cauchy-Schwartz inequality and equation \eqref{equ:inequ_2}, we can get
\fontsize{9}{9}
\begin{align}
    U&\overset{\mathrm{def}}{=}\sum_{k, h}\ogk\sqrt{\frac{L}{n_k(x_h^k)+1}}\Sp{\sqrt{\hat{\bbV}^*_{h, k}}+\sqrt{\hat{\overline{\bbV}}_{h, k}}}\leq \sqrt{\widetilde{O}\Sp{HSA}\sum_{k, h}\ogk\Sp{\hat{\bbV}^*_{h, k}+\hat{\overline{\bbV}}_{h, k}}} \label{equ:simple_U}
\end{align}
\normalsize
% As our detailed analysis in appendix will show, $U$ is basically the quantity we need to bound in order to get an explicit bound for equation \eqref{equ:reg} under Bernstein-type noise.

Here, note that $U\approx \sum_{k, h}\sigma^k_\Be(x_h^k)$. Then, after some steps of algebra, it is possible to show that
\begin{align*}
	\sum_{k=1}^K\sum_{h=1}^H\ogk\Sp{\hat{\bbV}^*_{h, k}+\hat{\overline{\bbV}}_{h, k}}
% 	&\leq\widetilde{O}\Sp{HT+H^6S^2A+\sqrt{H^5T}+H^2U}\\
	&\leq\widetilde{O}\Sp{HT+H^2U}\tag{When $T\geq\widetilde{\Omega}\Sp{H^5S^2A}$}\\
	\implies U\leq\widetilde{O}\Sp{\sqrt{HSA\Sp{HT+H^2U}}}&\leq\widetilde{O}\Sp{H\sqrt{SAT}+H^{1.5}\sqrt{U}}.\tag{By using equation \eqref{equ:simple_U}}
\end{align*}
% Then, by plugging this result into equation \eqref{equ:simple_U}, we can obtain
% \begin{align*}
%     U&\leq\widetilde{O}\Sp{\sqrt{HSA\Sp{HT+H^2U}}}\leq\widetilde{O}\Sp{H\sqrt{SAT}+H^{1.5}\sqrt{U}}.
% \end{align*}
Now, we can see that $\sum_{k, h}\sigma^k_\Be(x_h^k)\approx U\leq\widetilde{O}\Sp{H\sqrt{SAT}}$ satisfies this inequality. 
% Therefore, when $T\geq\widetilde{\Omega}\Sp{H^5S^2A}$, it holds with high probability that $U\leq\widetilde{O}\Sp{H\sqrt{SAT}}$. The full proof is given in Appendix under Lemma \ref{theo:variance_sum}.
Finally, by plugging this result back into equation \eqref{equ:reg}, we can have
$$\mathrm{Reg}\Sp{M, K, \mathsf{\algoname}_\Be}\leq\widetilde{O}\Sp{H\sqrt{SAT}+H^4S^2A},$$
which matches the known lower bound when $T\geq\widetilde{\Omega}\Sp{H^6S^3A}$.
\section{Conclusion}
\label{sec:conclusion}
We gave a new algorithm with randomized exploration, \algoname, for tabular MDP, which enjoys a near-optimal $\widetilde{O}\left(H\sqrt{SAT}\right)$ regret bound in the time-homogeneous model.
Previously, near-optimal regret bounds can only be achieved by optimistic algorithms.
Our result also highlights the importance of using a single random seed for the entire episode and using the variance information in tuning the magnitude of noise (cf. Bernstein's inequality).

One important open problem is whether randomized exploration can a achieve a horizon-free regret bound in the time-homogeneous model where the transition is the same at different levels~\citep{zanette2019tighter,wang2020long,zhang2020reinforcement}. Another possible future direction is to consider whether the sub-optimal lower order terms $\widetilde{O}\Sp{H^4S^2A}$ can be further improved to relax the current requirement $T\geq\widetilde{\Omega}\Sp{H^6S^3A}$ for being near-optimal.
% Another interesting direction is to see whether our analysis can be generalized to function approximation settings.

\section*{Acknowledgements} 
We sincerely thank Jinglin Chen and Chao Qin for pointing out mistakes in the initial draft of this paper. We are grateful for their careful reading and insightful discussion.
This work was supported in part by NSF TRIPODS II-DMS 2023166, NSF CCF 2007036, NSF IIS 2110170, NSF DMS 2134106, NSF CCF 2212261, NSF IIS 2143493, NSF CCF 2019844.

\medskip
\bibliography{References,simonduref}
\bibliographystyle{plainnat}

\newpage
\appendix
\tableofcontents
\newpage

\section{Table of Notations}
\label{sec:notation}
\setcounter{theorem}{0}
% \setcounter{lemma}{0}
% \setcounter{definition}{0}

% \begin{table}[h]
%	\caption{Classification accuracies for naive Bayes and flexible
%		Bayes on various data sets.}
% 	\label{table:notation}
	\vskip 0.15in
% 	\begin{center}
		\begin{longtable}{ll}
			\toprule
			\textbf{Symbol} & \textbf{Meaning} \\
			\midrule
			$\S$ & The state space \\
			$\A$ & The action space \\
			$S$ & Size of state space \\
			$A$ & Size of action space \\
			$H$ & The length of horizon \\
			$K$ & The total number of episodes \\
			$T$ & The total number of steps, $T=HK$\\
			$\pi^k$ & The greedy policy generated in Algorithm~\ref{algo:cub-rlsvi} at episode $k$\\
			$R_{h, s, a}$ & Expected reward function at $(h, s, a)$\\
			${P}_{h, s, a}\Sp{s'}$ & Transition probability\\
			$M$ & Underlying true MDP, $M=(H, \S, \A, R, P, s_1)$\\
			$n_k(h,s,a)$ & $\sum_{k'=1}^{k-1} \mathds{1}\{(s_h^{k'},a_h^{k'})=(s,a)\}$ \\
			$\hat{R}^k_{h, s, a}$ & Estimated reward function, $\frac{1}{n_k\Sp{h, s, a}+1}\sum_{l=1}^{k-1}\mathds{1}\{(s_h^l, a_h^l)=(s, a)\}r^l_{h, s_h^l, a_h^l}$\\
			$\hat{P}^k_{h, s, a}\Sp{s'}$ & Estimated transition kernel,  $\frac{1}{n_k\Sp{h, s, a}+1}\sum_{l=1}^{k-1}\mathds{1}\{(s_h^l, a_h^l, s_{h+1}^l)=(s, a, s')\}$ \\
			$\tilde{P}^k_{h, s, a}\Sp{s'} $ & Estimated transition probability with a slightly different\\
			& denominator, $\frac{1}{\max\Bp{n_k\Sp{h, s, a}, 1}}\sum_{l=1}^{k-1}\mathds{1}\{(s_h^l, a_h^l, s_{h+1}^l)=(s, a, s')\}$\\
			$\hat{M}^k$ & Estimated MDP, $\hat{M}^k=(H, \S, \A, \hat{P}, \hat{R}, s_1)$\\
			$\gamma^k_\ty(h, s, a)$ & $\sigma^k_\ty(h, s, a)\sqrt{\log(40k^4)}$\\
			$\hat{z}_k$ & Perturbation's single random source during episode $k$\\
			& from a standard Gaussian, $\hat{z}_k\sim\mathcal{N}(0, 1)$\\
			$w^k_\ty(h, s, a)$ & Noise of type ``ty'', $w^k_\ty(h, s, a)=\sigma^k_\ty(h, s, a)\hat{z}_k$\\
			$\uw^k_\ty(h, s, a)$ & $-\gamma^k_\ty\hsa$\\
			$\overline{w}^k_\ty(h, s, a)$ & $\gamma^k_\ty\hsa$\\
			$\overline{M}^k_\ty$ & Perturbed estimated MDP with ty-type noise, $\overline{M}^k=(H, \S, \A, \hat{P}, \hat{R}+w^k_\ty, s_1)$\\
% 			$M'^k$ & A copy of $\overline{M}^k$ with noise independent to $\overline{M}^k$\\
			$\underline{M}^k_\ty$ & Negatively perturbed MDP, $\underline{M}^k_\ty=(H, \S, \A, \hat{P}, \hat{R}+\uw^k_\ty, s_1)$\\
			$\overline{\overline{M}}^k_\ty$ & Positively perturbed MDP, $\underline{M}^k_\ty=(H, \S, \A, \hat{P}, \hat{R}+\overline{w}^k_\ty, s_1)$\\
			$V^*_h$ / $V^*_{h, k}$ & Optimal value function at step $h$ for true MDP $M$\\
			$V^\pik_{h}$ / $V^\pik_{h, k}$ & Value function at step $h$ by running policy $\pik$ on true MDP $M$\\
			$\overline{Q}_{h, k}$ & $Q$-value function obtained by running Algorithm \ref{algo:cub-rlsvi}\\
			$\overline{V}_{h, k}$ & Value function obtained by running policy $\pik$ on $\overline{M}^k$\\
			& with a clipping of threshold $2(H-h+1)$\\
			$\underline{V}_{h, k}$ & Value function obtained by running policy $\pik$ on $\underline{M}^k$\\
			& with a clipping of threshold $2(H-h+1)$\\
% 			$\underline{Q}_{h, k}$ & Value function obtained by running policy $\pik$ on $\underline{M}^k$\\
			$\overline{\overline{V}}_{h, k}$ & Value function obtained by running policy $\pik$ on $\overline{\overline{M}}^k$\\
			& with a clipping of threshold $2(H-h+1)$\\
% 			$\overline{\overline{Q}}_{h, k}$ & Value function obtained by running policy $\pik$ on $\overline{\overline{M}}^k$\\
% 			$\ov'$ & Optimal value function of $M'$\\
			$\Rr^k_{h, s, a}$ & $\hat{R}^k_{h, s, a} - R_{h, s, a}$\\
			$\mathcal{P}^k_{h, s, a}$ & $\inner{\hat{P}^k_{h, s, a}-P_{h, s, a}, V^*_{h+1}}$\\
			$\mathcal{H}^{k}_h$ & The historical observations and actions till time $h$ in episode $k$,\\
			  & $\Bp{(s_l^j, a_l^j, r_l^j):j\leq k\text{ and }l\leq H\text{ if }j<k, \text{ else }l\leq h}$ \\
			$\overline{\mathcal{H}}^k_h$ & The historical observations and actions till time $h$ and episode $k$,\\
			 & plus the randomness in episode $k$, $\mathcal{H}_h^k\cup\Bp{ \hat{z}_k}$\\
			$\bbV\Sp{P, V}$ & Variance of $V\in\R^S$ under distribution $P\in\Delta^S$, $\sum_{s\in\S}P(s)(V(s)-\inner{P, V})^2$\\
			$\alpha_k$ & $200H^2\log(2HSAk^2)\log(40k^4)$\\
			$\sigma^k_{\mathrm{ty}}\Sp{h, s, a}$ & Magnitude of perturbation. $\ty\in\Bp{\Ho, \Be}$\\
			$\ty$ & Reserved subscript for denoting perturbation type, $\ty\in\Bp{\Ho, \Be}$,\\
			 & where ``Ho'' denotes Hoeffding-type and ``Be'' denotes Bernstein-type\\
			$\sigma^k_{\mathrm{Ho}}\Sp{h, s, a}$ & $H\sqrt{\frac{\log\Sp{2HSAk^2}}{n_k\Sp{h, s, a}+1}}+\frac{H}{n_k\Sp{h, s, a}+1}$\\
			$\sigma^k_{\mathrm{Be}}\Sp{h, s, a}$ & $\sqrt{\frac{16\mathbb{V}\Sp{\tilde{P}^k_{h, s, a}, \overline{V}_{k, h+1}}\log\Sp{2HSAk^2}}{n_k\Sp{h, s, a}+1}}+\frac{65H\log\Sp{2HSAk^2}}{n_k\Sp{h, s, a}+1}+\sqrt{\frac{\log\Sp{2HSAk^2}}{n_k\Sp{h, s, a}+1}}$\\
			$\sqrt{e^k_{\mathrm{Ho}}(h,s,a)}$ & $H\sqrt{\frac{\log\Sp{2HSAk^2}}{n_k\Sp{h, s, a}+1}}+\frac{H}{n_k\hsa+1}$\\
			$\sqrt{e^k_{\mathrm{Be}}(h,s,a)}$ & $\sqrt{\frac{6\mathbb{V}\Sp{\tilde{P}^k_{h, s, a}, V^*_{h+1}}\log\Sp{2HSAk^2}}{n_k\Sp{h, s, a}+1}}+\frac{9H\log\Sp{2HSAk^2}}{n_k\Sp{h, s, a}+1}+\sqrt{\frac{\log\Sp{2HSAk^2}}{n_k\Sp{h, s, a}+1}}$\\
% 			$\ud_h(s_h)$ &  $V^*_h(s_h) - \uv_h(s_h)$\\
%             $\od_h(s_h)$ &  $ V^*_h(s_h) - \bv_h(s_h)$ \\
            % $\ud^\pi_h(s_h)$ & $ \uv_h(s_h) -V_h^\pi(s_h)$ \\
            % $\od_h^\pi(s_h)$  & $ \bv_h(s_h) - V_h^\pi(s_h)$\\
            % $\delta^\pi_h(s_h)$ & $V^*_h(s_h) - V^\pi_h(s_h)$\\
            $C_1$ & $\frac{1}{\Phi(1.5)-\Phi(1)}$\\
 			%$\H_h^k$ & \\
			\bottomrule
		\end{longtable}
% 	\end{center}
	\vskip -0.1in
% \end{table}
\section{Good Events}
\label{sec:good_events}
% \simon{I don't think we have defined $\tilde{P}$? Actually, why do we need $\tilde{P}$ in the proof? Can we replace it by $\hat{P}$?}
\begin{definition}
	Let $M'=\Sp{H, \S, \A, P', R', s_1}$. We define the following confidence sets for both Bernstein-type and Hoeffding-type noise
	\fontsize{9}{9}
	\begin{equation}
	    \label{def:confi_set}
	    \mathcal{M}^k_{\mathrm{ty}}=\Bp{M':\forall\Sp{h, s, a}, \abs{\Sp{R'_{h, s, a}-R_{h, s, a}}+\inner{P'_{h, s, a}-P_{h, s, a}, V^{*}_{h+1}}}\leq\sqrt{e^k_{\mathrm{ty}}\Sp{h, s, a}}},
	\end{equation}
	\normalsize
	where the confidence widths are set as
	\fontsize{10}{10}
	\begin{equation}
	    \label{def:confi_width_be}
	    \begin{split}
	        \sqrt{e^k_{\mathrm{Be}}\Sp{h, s, a}}=&\sqrt{\frac{6\mathbb{V}\Sp{\tilde{P}^k_{h, s, a}, V^*_{h+1}}\log\Sp{2HSAk^2}}{n_k\Sp{h, s, a}+1}}\\
	        &\qquad+\frac{9H\log\Sp{2HSAk^2}}{n_k\Sp{h, s, a}+1}+\sqrt{\frac{\log\Sp{2HSAk^2}}{n_k\Sp{h, s, a}+1}},
	    \end{split}
	\end{equation}
	\normalsize
	\begin{equation}
	    \label{def:confi_width_ho}
	    \sqrt{e^k_{\mathrm{Ho}}\Sp{h, s, a}}=H\sqrt{\frac{\log\Sp{2HSAk^2}}{n_k\Sp{h, s, a}+1}}+\frac{H}{n_k\hsa+1}.
	\end{equation}
	We also define two events $\mathcal{E}^1_k$ and $\mathcal{E}^2_k$ as the following:
	\begin{equation}
	    \label{def:event_e1}
	    \mathcal{E}^1_k=\Bp{\abs{\hat{R}^k_{h, s, a}-R_{h, s, a}}\leq\sqrt{\frac{\log\Sp{2HSAk^2}}{n_k\Sp{h, s, a}+1}}+\frac{1}{n_k\hsa+1},\ \forall\Sp{h, s, a}},
	\end{equation}
% 	\fontsize{8}{8}
	\begin{equation}
	    \label{def:event_e2}
	    \begin{split}
	        \mathcal{E}^2_k=\Bigg\lbrace\abs{\inner{\hat{P}^k_{h, s, a}-P_{h, s, a}, V^*_{h+1}}}\leq& \sqrt{\frac{6\mathbb{V}\Sp{\tilde{P}^k_{h, s, a}, V^*_{h+1}}\log\Sp{2HSAk^2}}{n_k\Sp{h, s, a}+1}}\\
	        &\qquad +\frac{8H\log\Sp{2HSAk^2}}{n_k\Sp{h, s, a}+1},\forall\hsa\Bigg\rbrace.
	    \end{split}
	\end{equation}
% 	\normalsize
\end{definition}

We have the following lemmas about concentration of events.
\begin{lemma}
	\label{lemma:r_concen}
	For fixed $\Sp{k, h, s, a}$, let $n=n_k\hsa$. Then, if $n\geq 1$, for any fixed $\delta>0$, we have
	$$\P\Sp{|\hat{R}^k_{h, s, a}-R_{h, s, a}|\geq\sqrt{\frac{\log\Sp{2/\delta}}{n+1}}+\frac{1}{n+1}}\leq\delta.$$
\end{lemma}
\begin{proof}
	Let $\hat{R}^k_{h, s, a}=\frac{1}{n+1}\sum_{i=1}^{n}r_{\hsa, i}$, where $r_{\hsa, i}\sim\mathscr{R}_{h, s, a}$ are i.i.d. samples. By definition of the MDP, we have $\E\Mp{r_{\hsa, i}}=R_{h, s, a}$. Then, notice that
	$$\hat{R}^k_{h, s, a}=\frac{1}{n+1}\sum_{i=1}^{n}r_{\hsa, i}=\frac{1}{n}\sum_{i=1}^{n}r_{\hsa, i}-\frac{1}{n\Sp{n+1}}\sum_{i=1}^{n}r_{\hsa, i}.$$
	Since the reward is assumed to be bounded in $\Mp{0, 1}$, we have $\frac{1}{n\Sp{n+1}}\sum_{i=1}^{n}r_{\hsa, i}\leq\frac{1}{n+1}$. Then, for fixed $\delta>0$, we have
	\begin{align*}
		&\P\Sp{|\hat{R}^k_{h, s, a}-R_{h, s, a}|\geq\sqrt{\frac{\log\Sp{2/\delta}}{n+1}}+\frac{1}{n+1}}\\
		=&\P\Sp{\abs{\frac{1}{n}\sum_{i=1}^{n}r_{\hsa, i}-R_{h, s, a}-\frac{1}{n\Sp{n+1}}\sum_{i=1}^{n}r_{\hsa, i}}\geq\sqrt{\frac{\log\Sp{2/\delta}}{n+1}}+\frac{1}{n+1}}\\
		\leq&\P\Sp{\abs{\frac{1}{n}\sum_{i=1}^{n}r_{\hsa, i}-R_{h, s, a}}+\frac{1}{n+1}\geq\sqrt{\frac{\log\Sp{2/\delta}}{n+1}}+\frac{1}{n+1}}\tag{By triangle inequality}\\
		\leq&\P\Sp{\abs{\frac{1}{n}\sum_{i=1}^{n}r_{\hsa, i}-R_{h, s, a}}\geq\sqrt{\frac{\log\Sp{2/\delta}}{2n}}}\tag{Since $n+1\leq 2n$ for $n\geq 1$}\\
		\leq&\delta.\tag{By standard Hoeffding's inequality}
	\end{align*}
\end{proof}
\begin{lemma}
	\label{lemma:v_concen}
	For fixed $\Sp{k, h, s, a}$, let $n=n_k\Sp{h, s, a}$ and $V\in\R^S$ be some non-negative value function such that $\Norm{V}_{\infty}\leq H$. Then, if $n\geq 1$, for any fixed $\delta>0$, we have
	\begin{equation}
		\label{equ:v_concen_ho}
		\P\Sp{\abs{\inner{\hat{P}^k_{h, s, a}-P_{h, s, a}, V}}\geq H\sqrt{\frac{\log\Sp{2/\delta}}{n+1}}+\frac{H}{n+1}} \leq\delta,
	\end{equation}
	\begin{equation}
		\label{equ:v_concen_be}
		\P\Sp{\abs{\inner{\hat{P}^k_{h, s, a}-P_{h, s, a}, V}}\geq \sqrt{\frac{6\bbV\Sp{\tilde{P}^k_{h, s, a}, V}\log\Sp{2/\delta}}{n+1}}+\frac{8H\log\Sp{2/\delta}}{n+1}}\leq\delta.
	\end{equation}
\end{lemma}
\begin{proof}
	For fixed $\hsa$, we generate $n$ i.i.d. samples of $s_{\hsa, i}\sim P_{h, s, a}$ and consider $V\Sp{s_{\hsa, i}}$. Then, by taking $n_k\Sp{h, s, a}=n$, we have
	$$\inner{\hat{P}^k_{h, s, a}, V}=\frac{1}{n}\sum_{i=1}^{n}V\Sp{s_{\hsa, i}}-\frac{1}{n\Sp{n+1}}\sum_{i=1}^{n}V\Sp{s_{\hsa, i}}.$$
	The first result in equation (\ref{equ:v_concen_ho}) can be proved very similarly as Lemma \ref{lemma:r_concen} using Hoeffding's inequality by simply replacing the upper bound of 1 in reward by $H$.
	
	Then, for second result, we first consider $n\geq 2$. For some $\delta>0$, define
	$$b_{\hsa, n}=\sqrt{\frac{2\mathbb{V}\Sp{\tilde{P}^k_{h, s, a}, V}\log\Sp{2/\delta}}{n-1}}+\frac{7H\log\Sp{2/\delta}}{3\Sp{n-1}}+\frac{H}{n+1}.$$
	By noticing that $F\Sp{s}\leq H$ and applying similar technique in proof of Lemma \ref{lemma:r_concen}, we have
	\begin{align*}
		&\P\Sp{\abs{\inner{\hat{P}^k_{h, s, a}-P_{h, s, a}, V}}\geq b_{\hsa, n}}\\
		\leq&\P\Sp{\abs{\inner{\tilde{P}^k_{h, s, a}-P_{h, s, a}, V}}\geq\sqrt{\frac{2\mathbb{V}\Sp{\tilde{P}^k_{h, s, a}, V}\log\Sp{2/\delta}}{n-1}}+\frac{7H\log\Sp{2/\delta}}{3\Sp{n-1}}}\\
		\leq&\delta.\tag{By Lemma \ref{lemma:empirical_bernstein}, the empirical Bernstein's inequality}
	\end{align*}
	Then, since $3\Sp{n-1}\geq n+1$ when $n\geq 2$, we can easily check that
	$$b_{\hsa, n}\leq\sqrt{\frac{6\bbV\Sp{\tilde{P}^k_{h, s, a}, V}\log\Sp{2/\delta}}{n+1}}+\frac{8H\log\Sp{2/\delta}}{n+1}.$$
	Finally, since $\Norm{V}_\infty\leq H$, when $n=1$, we trivially have
	$$\abs{\inner{\hat{P}^k_{h, s, a}-P_{h, s, a}, V}}\leq H\leq \sqrt{\frac{6\bbV\Sp{\tilde{P}^k_{h, s, a}, V}\log\Sp{2/\delta}}{n+1}}+\frac{8H\log\Sp{2/\delta}}{n+1}.$$
	Therefore, we can conclude that
	$$\P\Sp{\abs{\inner{\hat{P}^k_{h, s, a}-P_{h, s, a}, V}}\geq \sqrt{\frac{6\bbV\Sp{\tilde{P}^k_{h, s, a}, V}\log\Sp{2/\delta}}{n+1}}+\frac{8H\log\Sp{2/\delta}}{n+1}}\leq\delta.$$
\end{proof}
\begin{lemma}
	\label{lemma:confi_set1}
	$\sum_{k=1}^{\infty}\P\Sp{\Sp{\mathcal{E}_k^{1}}^c}\leq\frac{\pi^2}{6}.$
\end{lemma}
\begin{proof}
	Let $n=n_k\hsa$. Then, for some fixed $\Sp{h, s, a}$, $n\geq 1$ and $\delta_n>0$, by Lemma \ref{lemma:r_concen}, we have

	$$\P\Sp{|\hat{R}^k_{h, s, a}-R_{h, s, a}|\geq\sqrt{\frac{\log\Sp{2/\delta_n}}{n+1}}+\frac{1}{n+1}}\leq\delta_n.$$

	Therefore, by taking $\delta_n=\frac{1}{HSAn^2}$, a union bound will give us
	$$\sum_{n=1}^{\infty}\sum_{h,s,a}\P\Sp{\Bp{|\hat{R}^k_{h, s, a}-R_{h, s, a}|\geq\sqrt{\frac{\log\Sp{2HSAn^2}}{n+1}}+\frac{1}{n+1}}}\leq\sum_{n=1}^{\infty}\frac{1}{n^2}=\frac{\pi^2}{6}.$$
	Therefore, we have
	\fontsize{8}{8}
	\begin{align*}
		&\sum_{k=1}^{\infty}\P\Sp{\exists\Sp{h, s, a}:n_k\Sp{h, s, a}>0, |\hat{R}^k_{h, s, a}-R_{h, s, a}|\geq\sqrt{\frac{\log\Sp{2HSAn_k\Sp{h, s, a}^2}}{n_k\Sp{h, s, a}+1}}+\frac{1}{n_k\Sp{h, s, a}+1}}\\
		&\leq\frac{\pi^2}{6}.
	\end{align*}
	\normalsize
	Since the MDP is time-inhomogeneous, each $\hsa$ can only be visited at most once during one episode, which implies $n_k\Sp{h, s, a}\leq k$. Therefore, we have 
	$$\sqrt{\frac{\log\Sp{2HSAn_k\Sp{h, s, a}}}{n_k\Sp{h, s, a}+1}}+\frac{1}{n_k\Sp{h, s, a}+1}\leq\sqrt{\frac{\log\Sp{2HSAk^2}}{n_k\Sp{h, s, a}+1}}+\frac{1}{n_k\Sp{h, s, a}+1}$$
	and thus the proof is complete.
\end{proof}

\begin{lemma}
	\label{lemma:confi_set2}
	$\sum_{k=1}^{\infty}\P\Sp{\Sp{\mathcal{E}_k^{2}}^c}\leq\frac{\pi^2}{6}.$
\end{lemma}
\begin{proof}
	This proof will be very similar to proof of Lemma \ref{lemma:confi_set1}. In specific, for fixed $\hsa$, let $n=n_k\Sp{h, s, a}\geq 1$. Then, for any $\delta_n>0$, since $\Norm{V^*_{h+1}}_\infty\leq H$, by Lemma \ref{lemma:v_concen}, we have
	$$\P\Sp{\abs{\inner{\hat{P}^k_{h, s, a}-P_{h, s, a}, V^*_{h+1}}}\geq \sqrt{\frac{6\bbV\Sp{\tilde{P}^k_{h, s, a}, V^*_{h+1}}\log\Sp{2/\delta_n}}{n+1}}+\frac{8H\log\Sp{2/\delta_n}}{n+1}}\leq\delta_n.$$
	Therefore, by taking $\delta_n=\frac{1}{HSAn^2}$ and applying a similar union bound argument used in the proof of Lemma \ref{lemma:confi_set1}, we can conclude $\sum_{k=1}^{\infty}\P\Sp{\Sp{\mathcal{E}^2_k}^c}\leq\frac{\pi^2}{6}$.
\end{proof}

We further define the event $\mathcal{C}^k_\ty=\Bp{\hat{M}^k\in\M^k_\ty}$. With what we have proved above, it will be straightforward to show the following results about $\mathcal{C}^k_\ty$.
\begin{lemma}
    \label{lemma:C_be}
	$\sum_{k=1}^{\infty}\P\Sp{\Sp{\mathcal{C}^k_\Be}^c}=\sum_{k=1}^{\infty}\P\Sp{\hat{M}^k\notin\mathcal{M}^k_\Be}\leq\frac{\pi^2}{3}$
\end{lemma}
\begin{proof}
	We can easily notice $\mathcal{E}_k^1\cap\mathcal{E}_k^2\implies\hat{M}^k\in\mathcal{M}^k_\Be$, which implies $\hat{M}^k\notin\mathcal{M}^k_\Be\implies\Sp{\mathcal{E}_k^1}^c\cup\Sp{\mathcal{E}^2_k}^c$. The first result then follows straightforwardly by applying Lemma \ref{lemma:confi_set1} and Lemma \ref{lemma:confi_set2}.
\end{proof}
\begin{lemma}
    \label{lemma:C_ho}
    $\sum_{k=1}^{\infty}\P\Sp{\Sp{\mathcal{C}^k_\Ho}^c}=\sum_{k=1}^{\infty}\P\Sp{\hat{M}^k\notin\mathcal{M}_{\mathrm{Ho}}^k}\leq\frac{\pi^2}{3}$.
\end{lemma}
\begin{proof}
    Similarly, for fixed $\hsa$, we generate $n$ i.i.d. samples $s_{\hsa, i}\sim P_{h, s, a}$ and $r_{\hsa, i}\sim\mathscr{R}_{h, s, a}$ for $i=1, \dots, n$ respectively. Define $Y_{\hsa, i}=r_{\hsa, i}+V^*_{h+1}(s_{\hsa, i})$ and we have $\E\Mp{Y_{\hsa, i}}=R_{h, s, a}+\inner{P_{h, s, a}, V^*_{h+1}}$.
    
    By definition of MDP, we know that $Y_{\hsa, i}\leq H$. Thus, we can use an argument similar to the proof of Lemma \ref{lemma:r_concen}. In specific, let $n=n_k(h, s, a)$ and for $\delta_n>0$, we have
    \begin{align*}
    	&\P\Sp{\abs{\frac{1}{n+1}\sum_{i=1}^{n}Y_{\hsa, i}-\E\Mp{Y_{\hsa, i}}}\geq H\sqrt{\frac{\log\Sp{2/\delta_n}}{n+1}}+\frac{H}{n+1}}\\
    	=&\P\Sp{\abs{\Sp{\hat{R}^k_{h, s, a}-R_{h, s, a}}+\inner{\hat{P}^k_{h, s, a}-P_{h, s, a}, V^*_{h+1}}}\geq H\sqrt{\frac{\log\Sp{2/\delta_n}}{n+1}}+\frac{H}{n+1}}\\
    	\leq&\delta_n.
    \end{align*}
	Then, we can take $\delta_n=\frac{1}{HSAn^2}$ and apply a similar union bound argument in used in the proof of Lemma \ref{lemma:confi_set1}. As a result, we can obtain
	$$\sum_{k=1}^{\infty}\P\Sp{\hat{M}^k\notin\mathcal{M}^k_\Ho}\leq\frac{\pi^2}{6}\leq\frac{\pi^2}{3}.$$
\end{proof}

We can also have well-behaved bounds on magnitude of noise and estimated value functions.
\begin{definition}
\label{def:E^w_k}
We define $w^k_\ty(h,s,a) = \sigma^k_\ty(h,s,a)\hat{z}_k$ and $\gamma^k_\ty(h,s,a) = \sigma^k_\ty(h,s,a)\sqrt{\log(40k^4)}$. We define the event $\mathcal{E}_k^w$ as 
\begin{align*}
    \mathcal{E}_k^w = \Bp{\forall (h,s,a), |w^k_\ty(h,s,a)|\leq \gamma^k_\ty(h,s,a)}.
\end{align*}
\end{definition}
\begin{lemma}
    \label{lemma:z_hat_bound}
	$\sum_{k=1}^{K}\P\Sp{\Sp{\mathcal{E}^w_{k}}^c}\leq\frac{\pi^2}{3}$ regardless the type of noise we choose.
\end{lemma}
\begin{proof}
    For any $k\in[K]$, by the tail bound of Gaussian distribution,
    \begin{align*}
        \P\Sp{|\hat{z}_k|\geq \sqrt{\log\Sp{40k^4}}} \leq 2\exp \Sp{-\frac{\log\Sp{40k^4}}{2}} \leq \frac{2}{k^2}.
    \end{align*}
    Summing over $k\in[K]$,
    \begin{align*}
    \sum_{k=1}^{K}\P\Sp{\Sp{\mathcal{E}_w^{k}}^c}= \sum_{k=1}^{K}\P\Sp{|\hat{z}_k|\geq \sqrt{\log\Sp{40k^4}}} \leq \sum_{k=1}^{\infty}\frac{2}{k^2}\leq\frac{\pi^2}{3}.
    \end{align*}
    Note that this result does not depend on the type of noise we choose.
\end{proof}
% \begin{definition}
% We define the events $\mathcal{E}_{h,k}^{\overline{Q}}$ as and $\mathcal{E}^{th}_k$ as
% \begin{align*}
%     \mathcal{E}_{h,k}^{\overline{Q}} = \Bp{\forall(s,a),|\Sp{\overline{Q}_{h,k} - Q^*_{h,k}}(s,a)| \leq H-h+1}.
% \end{align*}
% \end{definition}

Now, we define the following good events that hold with high probability and will be used throughout the whole proof.

% \qiwen{
\begin{definition}[Good events $\mathcal{G}_k$]
	Let $\mathcal{G}_{k,\ty} = \Bp{   {\mathcal{C}_\ty^{k} \cap \mathcal{E}_{k}^w }}$. 
\end{definition}
The subscript ``ty'' will be ignored later since it is clear from the context.

% \qiwen{
% \begin{definition}
% We define the event $\mathcal{E}^{th}_k$ as
% \begin{equation}
%     \label{def:event_no_clip}
%     \mathcal{E}^{th}_k=\Bp{\bigcap_{h=1}^{H}\Bp{n_k(h, s_h^k, a_h^k)\geq\alpha_k}}.
% \end{equation}
% \end{definition}

% We will show that under events $\mathcal{E}_k^w$, $\mathcal{E}^{th}_k$ and $\hat{M}^k\in\mathcal{M}^k_\ty$, no clipping happens on $s_h^k$ for all $h\in[H]$.
% }

\begin{definition}
With $\alpha_k=200H^2\log(2HSAk^2)\log(40k^4)$, we define events $\mathcal{E}^{th}_{h, k}$ and $\mathcal{E}^{cum}_{h, k}$ as
\begin{equation}
    \label{def:event_no_clip}
    \mathcal{E}^{th}_{h, k}=\Bp{n_k(h, s_h^k, a_h^k)\geq\alpha_k},\quad\mathcal{E}^{cum}_{h, k}=\bigcap_{i=1}^{h}\mathcal{E}^{th}_{i, k}.
\end{equation}
\end{definition}

We will show that under events $\mathcal{E}_k^w$, $\mathcal{E}^{th}_{h, k}$ and $\hat{M}^k\in\mathcal{M}^k_\ty$, no clipping happens on $s_h^k$.

% \begin{definition}
% We define the event $\mathcal{E}^{th}_k$ as
% \begin{equation}
%     \label{def:event_no_clip}
%     \mathcal{E}^{th}_k=\Bp{\bigcap_{h=1}^{H}\Bp{n_k(h, s_h^k, a_h^k)\geq\alpha_k}},
% \end{equation}
% which means that no clipping happens during the episode $k$.
% \end{definition}

% \qiwen{
\begin{lemma}
    \label{lemma:q_bar_bound}
    Assume that $\mathcal{E}_k^w$, $\mathcal{E}^{th}_{h, k}$ and $\hat{M}^k\in\mathcal{M}^k_\ty$ hold. Then, regardless the type of noise we choose, it holds that
    $$|\overline{Q}_{h, k}(s_h^k, a_h^k)|\leq 2(H-h+1),$$
    which immediately tells us that no clipping is triggered for any $(s_h^k, a_h^k)$.
\end{lemma}
\begin{proof}
    We have that
    \begin{align*}
        \overline{Q}_{h,k}(s_h^k,a_h^k)=&\hat{R}^k_\hsahk+\inner{\hat{P}^k_\hsahk,\ov_{h+1,k}}+\sigma_\ty^k(\hsahk)\hat{z}_k.
    \end{align*}
    As we have $\abs{\ov_{h+1,k}}\leq2(H-h)$ by clipping and $\hat{R}^k_\hsahk\in[0,1]$, we only need to show that $\sigma_\ty^k(\hsahk)\hat{z}_k\leq 1$. Under event $\mathcal{E}_w^k$, we have $\abs{\sigma_\ty^k(\hsahk)\hat{z}_k}$ is bounded by $\gamma^k_\ty(\hsahk)=\sigma^k_\ty(\hsahk)\sqrt{\log(40k^4)}$. Note that we have $\ov_{h+1, k}(s)\in [-2H, 2H]$ by clipping for any $s\in\S$. Thus, by Lemma \ref{lemma:standard_var_bound}, we have $\bbV\Sp{\tilde{P}^k_{h, s, a}, \overline{V}_{h+1, k}}\leq 4H^2$ for any $(h, s, a)$.
    
    By taking $\alpha_k=200H^2\log(2HSAk^2)\log(40k^4)$ and referring to the definitions of $\sigma^k_\Be(h, s, a)$ in Equation \eqref{equ:def_be_noise}, we can check that
    % \fontsize{8}{8}
    \begin{align*}
        &\gamma^k_\Be(\hsahk)\\
        =&\sigma^k_\Be(\hsahk)\sqrt{\log(40k^4)}\\
        =& \Sp{\sqrt{\frac{16\mathbb{V}\Sp{\tilde{P}^k_{\hsahk}, \overline{V}_{k, h+1}}\log\Sp{2HSAk^2}}{n_k\Sp{\hsahk}+1}}+\frac{65H\log\Sp{2HSAk^2}}{n_k\Sp{\hsahk}+1}}\sqrt{\log(40k^4)}\\
        &\qquad + \sqrt{\frac{\log\Sp{2HSAk^2}}{n_k\Sp{\hsahk}+1}}\cdot\sqrt{\log(40k^4)}\\
        \leq&\Sp{\sqrt{\frac{64H^2\log\Sp{2HSAk^2}}{\alpha_k}}+\frac{65H\log\Sp{2HSAk^2}}{\alpha_k}+\sqrt{\frac{\log\Sp{2HSAk^2}}{\alpha_k}}}\sqrt{\log(40k^4)}\tag{Event $\mathcal{E}_k^{th}$ implies $n_k(\hsahk)\geq\alpha_k$}\\
        \leq& \sqrt{\frac{64}{200}}+\frac{65}{200H}+\sqrt{\frac{1}{200H^2}}\leq 1.
    \end{align*}
    % \normalsize
    Thus, we have $\gamma^k_\Be\hsa\leq 1$ and we can similarly check that $\gamma^k_\Ho\hsa\leq 1$. As a result, we have
    \begin{align*}
        \abs{\overline{Q}_{h,k}(s_h^k,a_h^k)}\leq2(H-h+1),
    \end{align*}
    which completes the proof.
\end{proof}
\section{Optimism}
\label{sec:optimism}

Let $\mathcal{H}^k_h$ denote the history trajectory, which is defined as
\begin{equation}
	\label{equ:def_his}
	\mathcal{H}^k_h=\Bp{(s_l^j, a_l^j, r_l^j):j\leq k\text{ and }l\leq H\text{ if }j<k, \text{ else }l\leq h}.
\end{equation}
We will prove that for both types of noise, $\overline{V}_{h, k}$ is optimistic with constant probability under certain conditions.
\subsection{Hoeffding-type Noise}

% \qiwen{
\begin{lemma}
    \label{lemma:ho_optimism}
    Condition on history $\mathcal{H}^{k-1}_H$, if $\G_{k, \Ho}$ holds and Hoeffding-based noise is applied, then $\ov_{h, k}$ is optimisitic with constant probability for any $h\in[H]$. Specifically, we have
	$$\P\Sp{\overline{V}_{h, k}(s)\geq V^*_h(s),\forall h\in[H], s\in\S\mid\mathcal{H}^{k-1}_H, \G_{k, \Ho}}\geq\Phi(1.9)-\Phi(1):=C_\Ho.$$
\end{lemma}

\begin{proof}
    We will show that if $\hat{z}_k\geq 1$, then for all $h\in[H]$ and $s\in\S$, we have $\ov_{h,k}(s)\geq V^*_h(s)$. The proof will use induction and the argument is true for $h=H+1$ as $\ov_{H+1,k}(s)=V_{H+1}^*(s)=0$. Suppose the argument is true for timestep $h+1$ and for timestep $h$ we have
    
    \begin{align*}
        \ov_{h,k}(s)=&\clip_{2(H-h+1)}\Sp{\max_{a\in\A}\overline{Q}_{h,k}(s,a)}\\
        \geq & \min \Bp{2(H-h+1),\max_{a\in\A}\overline{Q}_{h,k}(s,a)}\\
        \geq & \min \Bp{(H-h+1),\overline{Q}_{h,k}(s,\pi_h^*(s))}\\
        \geq & \min \Bp{(H-h+1),\hat{R}_{h,s,\pi_h^*(s)}^k+\inner{\hat{P}_{h,s,\pi_h^*(s)}^k,\ov_{h+1,k}}+\sigma_\ty^k(h,s,\pi_h^*(s))\hat{z}_k}\\
        \geq & \min \Bp{(H-h+1),\hat{R}_{h,s,\pi_h^*(s)}^k+\inner{\hat{P}_{h,s,\pi_h^*(s)}^k,V^*_{h+1,k}}+\sigma_\ty^k(h,s,\pi_h^*(s))\hat{z}_k} \tag{Inductive hypothesis}\\
        \geq & \min \Bp{(H-h+1),R_{h,s,\pi_h^*(s)}^k+\inner{P_{h,s,\pi_h^*(s)}^k,V^*_{h+1,k}}}   \tag{Since $\hat{M}^k\in\mathcal{M}^k_\Ho$ inferred by $\G_{k, \Ho}$ and $\hat{z}_k\geq1$}\\
        \geq & \min \Bp{(H-h+1),Q^*_h(s,\pi_h^*(s))}\\
        \geq & V_h^*(s).
    \end{align*}
    Then by induction we have that the optimism is achieved for all $h\in[H]$ and $s\in\S$ simultaneously. Meanwhile, as stated in Definition \ref{def:E^w_k}, we have $\hat{z}_k\leq\sqrt{\log\Sp{40k^4}}$ under event $\mathcal{E}^w_k$ and numerically, $\sqrt{\log\Sp{40k^4}}\geq 1.9$. Therefore, the probability that $\hat{z}_k\geq 1$ under $\mathcal{E}^w_k$, inferred by $\G_{k, \Ho}$, is at least
	$$\P\Sp{\hat{z}_k\geq 1\mid \mathcal{H}^{k-1}_H, \G_{k, \Ho}}=\frac{\Phi(1.9)-\Phi(1)}{\Phi(1.9)-\Phi(-1.9)}\geq\Phi(1.9)-\Phi(1):=C_\Ho.$$
	Thus, we can conclude that
	$$\P\Sp{\overline{V}_{h, k}(s)\geq V^*_h(s),\forall h\in[H], s\in\S\mid\mathcal{H}^{k-1}_H, \G_{k, \Ho}}\geq C_\Ho.$$
\end{proof}

\subsection{Bernstein-type Noise}
The following proof of optimism applies some techniques used in \cite{zhang2020reinforcement}. We first present a technical lemma.
\begin{lemma}
	\label{lemma:mono}
	Let $f_z:\Delta^S\times\R_+^S\times\R\times\R\mapsto\R$ with $f_z\Sp{p, v, n, L}=\frac{n}{n+1}\inner{p, v}+\max\Bp{4\sqrt{\frac{\mathbb{V}\Sp{p, v}L}{n+1}}, \frac{64HL}{n+1}}\cdot z$ for some constant $H>0$ and $z\in\R$. Then, $f_z$ satisfies
	\begin{enumerate}[(i)]
		\item $f_z\Sp{p, v, n, L}$ is non-decreasing in $v\Sp{s}$ for all $p\in\Delta^S$, $\Norm{v}_{\infty}\leq 2H$, $L>0$, $n\geq 3$ and $z\in\Mp{-1.5, 1.5}$
		
		\item $f_z\Sp{p, v, n, L}\geq\frac{n}{n+1}\inner{p, v}+\Sp{3\sqrt{\frac{\mathbb{V}\Sp{p, v}L}{n+1}}+\frac{8HL}{n+1}}\cdot z$ for $z\in\Mp{1, 1.5}$.
		
		\item $f_z\Sp{p, v, n, L}\leq\frac{n}{n+1}\inner{p, v}+\Sp{3\sqrt{\frac{\mathbb{V}\Sp{p, v}L}{n+1}}+\frac{8HL}{n+1}}\cdot z$ for $z\in\Mp{-1.5, -1}$.
	\end{enumerate}
\end{lemma}
\begin{proof}
	It is obvious that $f_z\Sp{p, v, n, L}$ is continuous in $v\Sp{s}$ and not differentiable at only one point where $4\sqrt{\frac{\mathbb{V}\Sp{p, v}L}{n+1}}=\frac{64HL}{n+1}$. Therefore, to prove statement (i), we only need to show that $\frac{\partial f_z\Sp{p, v, n, L}}{\partial v\Sp{s}}\geq 0$. Specifically, we have
	\begin{align*}
		\frac{\partial f_z\Sp{p, v, n, L}}{\partial v\Sp{s}}&=\frac{n}{n+1}\cdot p\Sp{s}+\mathds{1}\Bp{4\sqrt{\frac{\mathbb{V}\Sp{p, v}L}{n+1}}\geq\frac{64HL}{n+1}}\frac{4p\Sp{s}\Sp{v\Sp{s}-\inner{p, v}}L}{\sqrt{\Sp{n+1}\mathbb{V}\Sp{p, v}L}}\cdot z\\
		&\overset{\text{(a)}}{\geq} \frac{n}{n+1}\cdot p\Sp{s}+\mathds{1}\Bp{4\sqrt{\frac{\mathbb{V}\Sp{p, v}L}{n+1}}\geq\frac{64HL}{n+1}}\frac{-8HL}{\sqrt{\Sp{n+1}\mathbb{V}\Sp{p, v}L}}\cdot z\\
% 		&\overset{\text{(a)}}{\geq}\min\Bp{\frac{n}{n+1}p\Sp{s}+\frac{p\Sp{s}\Sp{v(s)-\inner{p, v}}}{4H}\cdot z, \frac{n}{n+1}p(s)}\\
		&\overset{\text{(b)}}{\geq} p(s)\Sp{\frac{n}{n+1}-\frac{z}{2}}\\
		&\geq 0.
	\end{align*}
	Here, The inequality (a) holds because $\Norm{v}_{\infty}\leq 2H$ and $v$ is non-negative, which means to have $v\Sp{s}-\inner{p, v}\geq -2H$. The inequality (b) above holds because when the condition inside indicator $\mathds{1}\Bp{\cdot}$ holds, we will have $\sqrt{\Sp{n+1}\mathbb{V}\Sp{p, v}L}\geq 16HL$. The last inequality holds because we have $n\geq 3$ and $z\leq 1.5$. Therefore, $f_z\Sp{p, v, n, L}$ is non-decreasing in $v\Sp{s}$.
	
	For the statement (ii), we consider two cases. First, when $4\sqrt{\frac{\mathbb{V}\Sp{p, v}L}{n+1}}\geq \frac{64HL}{n+1}$ holds, we have $\frac{8HL}{n+1}\leq \frac{1}{2}\sqrt{\frac{\mathbb{V}\Sp{p, v}L}{n+1}}$, which means to have
	$$\frac{n}{n+1}\inner{p, v}+\Sp{3\sqrt{\frac{\mathbb{V}\Sp{p, v}L}{n+1}}+\frac{8HL}{n+1}}\cdot z\leq\frac{n}{n+1}\inner{p, v}+\frac{7z}{2}\sqrt{\frac{\mathbb{V}\Sp{p, v}L}{n+1}}\leq f_z\Sp{p, v, n, L}.$$
	When $4\sqrt{\frac{\mathbb{V}\Sp{p, v}L}{n+1}}\leq \frac{64HL}{n+1}$ holds, we have $3\sqrt{\frac{\mathbb{V}\Sp{p, v}L}{n+1}}\leq\frac{48HL}{n+1}$, which similarly leads to
	$$\frac{n}{n+1}\inner{p, v}+\Sp{3\sqrt{\frac{\mathbb{V}\Sp{p, v}L}{n+1}}+\frac{8HL}{n+1}}\cdot z\leq f_z\Sp{p, v, n, L}.$$
	
	The state (iii) can be shown similarly and thus the proof is complete.
\end{proof}

% \qiwen{
\begin{lemma}
    \label{lemma:be_optimism}
	Condition on history $\mathcal{H}^{k-1}_H$, if $\G_{k, \Ho}$ holds and Bernstein-based noise is applied, then $\ov_{h, k}$ is optimisitic with constant probability for any $h\in[H]$. Specifically, we have
	$$\P\Sp{\overline{V}_{h, k}\Sp{s}\geq V^*_h\Sp{s}, \forall h\in[H], s\in\S\mid\mathcal{H}^{k-1}_H, \G_{k, \Be}}\geq\Phi\Sp{1.5}-\Phi\Sp{1}:=C_\Be$$
\end{lemma}
\begin{proof}
	Similar to what we have discussed in the proof of Lemma \ref{lemma:ho_optimism}, under event $\mathcal{E}^w_k$, we have $\hat{z}_k\in\Mp{1, 1.5}$ with probability at least $\Phi\Sp{1.5}-\Phi\Sp{1}=C_\Be$. Then, we will show that $\overline{Q}_{h, k}\Sp{s, a}\geq Q^*_h\Sp{s, a}$ for any $h$ with arbitrary $s, a$ and $\hat{z}_k\in\Mp{1, 1.5}$. The proof will use induction. For simplicity, let $L=\log\Sp{2HSAk^2}$.
	
	For $h=H+1$, the inequality holds trivially because both sides are 0. Then, by assuming $\overline{Q}_{h+1, k}\Sp{s, a}\geq Q^*_{h}\Sp{s, a}$ for any $\Sp{s, a}$ such that $n_k(h, s, a)\geq 3$, we have
	\fontsize{9}{9}
	\begin{align*}
		\overline{Q}_{h, k}\Sp{s, a}&=\hat{R}^k_{h, s, a}+\inner{\hat{P}^k_{h, s, a}, \overline{V}_{h+1, k}}+\sigma^k_\Be\Sp{h, s, a}\hat{z}_k\\
		&\geq R_{h, s, a}+\inner{\hat{P}^k_{h, s, a}, \overline{V}_{h+1, k}}+\Sp{4\sqrt{\frac{\mathbb{V}\Sp{\tilde{P}^k_{h, s, a}, \overline{V}_{h+1, k}}L}{n_k\Sp{h, s, a}+1}}+\frac{64HL}{n_k\Sp{h, s, a}+1}}\cdot\hat{z}_k\tag{Replace $\hat{R}_{h, s, a}$ by $R_{h, s, a}$ through applying event $\mathcal{E}^1_k$ defined in (\ref{def:event_e1})}\\
		&\geq R_{h, s, a}+\inner{\hat{P}^k_{h, s, a}, \overline{V}_{h+1, k}}+\max\Bp{4\sqrt{\frac{\mathbb{V}\Sp{\tilde{P}^k_{h, s, a}, \overline{V}_{h+1, k}}L}{n_k\Sp{h, s, a}+1}}, \frac{64HL}{n_k\Sp{h, s, a}+1}}\cdot\hat{z}_k\\
		&\overset{\text{(a)}}{\geq} R_{h, s, a}+\inner{\hat{P}^k_{h, s, a}, V^*_{h+1}}+\max\Bp{4\sqrt{\frac{\mathbb{V}\Sp{\tilde{P}^k_{h, s, a}, V^*_{h+1}}L}{n_k\Sp{h, s, a}+1}}, \frac{64HL}{n_k\Sp{h, s, a}+1}}\cdot\hat{z}_k\\
		&\geq R_{h, s, a}+\inner{\hat{P}^k_{h, s, a}, V^*_{h+1}}+\Sp{3\sqrt{\frac{\mathbb{V}\Sp{\tilde{P}^k_{h, s, a}, V^*_{h+1}}L}{n_k\Sp{h, s, a}+1}}+\frac{8HL}{n_k\Sp{h, s, a}+1}}\cdot\hat{z}_k\tag{By applying statement (ii) of Lemma \ref{lemma:mono}}\\
		&\geq R_{h, s, a}+\inner{\hat{P}^k_{h, s, a}, V^*_{h+1}}+\sqrt{\frac{6\mathbb{V}\Sp{\tilde{P}^k_{h, s, a}, V^*_{h+1}}}{n_k\Sp{h, s, a}+1}}+\frac{8HL}{n_k\hsa+1}\tag{Since $\hat{z}_k\geq 1$}\\
		&\geq R_{h, s, a}+\inner{P_{h, s, a}, V^*_{h+1}}\tag{By applying event $\mathcal{E}_k^2$ defined in (\ref{def:event_e2})}\\
		&=Q^*_h\Sp{s, a}.
	\end{align*}
	\normalsize
	Here, the above inequality (a) holds by applying inductive hypothesis and statement (i) in Lemma \ref{lemma:mono}. It is applicable because when $\mathcal{E}^w_k$ holds, and by the clipping function, $\Norm{\overline{V}_{h+1, k}}_{\infty}\leq 2H$. When $n_k(h, s, a)<3$, $\overline{Q}_{h, k}(s, a)\geq Q^*_h(s, a)$ holds trivially because $Q^*_h(s, a)\leq H$ by definition. Therefore, the induction is complete.
	
	Now, for arbitrary $(k, h, s)$, set $a=\argmax_{a\in\A}\overline{Q}_{h,k}(s,a)$ and we have
	\begin{align*}
	    \ov_{h,k}(s)=&\clip_{2(H-h+1)}\Sp{\max_{a\in\A}\overline{Q}_{h,k}(s,a)}\\
        \geq & \min \Bp{2(H-h+1),\max_{a\in\A}\overline{Q}_{h,k}(s,a)}\\
        \geq & \min \Bp{(H-h+1),\overline{Q}_{h,k}(s,\pi_h^*(s))}\\
        \geq & \min \Bp{(H-h+1),Q^*_{h,k}(s,\pi_h^*(s))}\\
        \geq & V^*_{h,k}(s)
	\end{align*}
\end{proof}
\section{Pessimism}
\label{sec:pessimism}

Similar to what we have proved in Section \ref{sec:optimism}, in this section we will prove that for both types of noise, $\overline{V}_{h, k}$ is pessimistic with constant probability under certain conditions.
\subsection{Hoeffding-type Noise}

% \qiwen{
\begin{lemma}
    \label{lemma:ho_pessimism}
    Condition on history $\mathcal{H}^{k-1}_H$, if $\G_{k, \Ho}$ holds and Hoeffding-based noise is applied, then $\ov_{h, k}$ is optimisitic with constant probability for any $h\in[H]$. Specifically, we have
	$$\P\Sp{\overline{V}_{h, k}(s)\leq V^*_h(s),\forall h\in[H], s\in\S\mid\mathcal{H}^{k-1}_H, \G_{k, \Ho}}\geq\Phi(1.9)-\Phi(1):=C_\Ho.$$
% 	\simon{$\geq \rightarrow \leq$?}
\end{lemma}

\begin{proof}
    We will show that if $\hat{z}_k\leq -1$, then for all $h\in[H]$ and $s\in\S$, we have $\ov_{h,k}(s)\leq V^*_h(s)$. The proof will use induction and the argument is true for $h=H+1$ as $\ov_{H+1,k}(s)=V_{H+1}^*(s)=0$. Suppose the argument is true for timestep $h+1$ and we consider timestep $h$. Set $a=\argmax_{a\in\A}\overline{Q}_{h,k}(s,a)$.
    
    \begin{align*}
        \ov_{h,k}(s)=&\clip_{2(H-h+1)}\Sp{\overline{Q}_{h,k}(s,a)}\\
        \leq & \max \Bp{-2(H-h+1),\overline{Q}_{h,k}(s,a)}\\
        \leq & \max \Bp{-(H-h+1),\overline{Q}_{h,k}(s,a)}\\
        \leq & \max \Bp{-(H-h+1),\hat{R}_{h,s,a}^k+\inner{\hat{P}_{h,s,a}^k,\ov_{h+1,k}}+\sigma_\ty^k(h,s,a)\hat{z}_k}\\
        \leq & \max \Bp{-(H-h+1),\hat{R}_{h,s,a}^k+\inner{\hat{P}_{h,s,a}^k,V^*_{h+1,k}}+\sigma_\ty^k(h,s,a)\hat{z}_k} \tag{Induction Hypothesis}\\
        \leq & \max \Bp{-(H-h+1),R_{h,s,a}^k+\inner{P_{h,s,a}^k,V^*_{h+1,k}}}\tag{Since $\hat{M}^k\in\mathcal{M}^k_\Ho$ and $\hat{z}_k\leq-1$}\\
        \leq & \max \Bp{-(H-h+1),Q^*_h(s,a)}\\
        \leq & \max \Bp{-(H-h+1),\max_{a\in\A}Q^*_h(s,a)}\\
        \leq & V_h^*(s).
    \end{align*}
    Then by induction we have that the optimism is achieved for all $h\in[H]$ and $s\in\S$ simultaneously. By using argument similar to the proof of Lemma \ref{lemma:ho_optimism}, we can see that when $\hat{z}_k\leq -1$, we have $\overline{V}_{h, k}(s)\leq V^*_{h}(s)$ an this hold simultaneously for any $h\in[H]$, $s\in\S$. Furtherm as stated in Definition \ref{def:E^w_k}, we have $|\hat{z}_k|\leq\sqrt{\log\Sp{40k^4}}$ under event $\mathcal{E}^w_k$ and numerically, $\sqrt{\log\Sp{40k^4}}\geq 1.9$. Therefore, the probability that $\hat{z}_k\leq -1$ under $\mathcal{E}^w_k$ is at least
	$$\P\Sp{\hat{z}_k\leq -1\mid\mathcal{H}^{k-1}_H, \G_{k, \Ho}}=\frac{\Phi(1.9)-\Phi(1)}{\Phi(1.9)-\Phi(-1.9)}\geq\Phi(1.9)-\Phi(1)=C_\Ho.$$
	Thus, we can conclude that
	$$\P\Sp{\overline{V}_{h, k}(s)\leq V^*_h(s),\forall h\in[H], s\in\S\mid\mathcal{H}^{k-1}_H, \G_{k, \Ho}}\geq C_\Ho.$$
\end{proof}

\subsection{Bernstein-type Noise}

% \qiwen{
\begin{lemma}
    \label{lemma:be_pessimism}
	Condition on history $\mathcal{H}^{k-1}_H$, if $\G_{k, \Ho}$ holds and Bernstein-based noise is applied, then $\ov_{h, k}$ is pessimistic with constant probability for any $h\in[H]$. Specifically, we have
	$$\P\Sp{\overline{V}_{h, k}\Sp{s}\leq V^*_h\Sp{s}, \forall h\in[H], s\in\S\mid\mathcal{H}^{k-1}_H, \G_{k, \Be}}\geq C_\Be$$
\end{lemma}
\begin{proof}
	Similar to what we have discussed in the proof of Lemma \ref{lemma:ho_pessimism}, under event $\mathcal{E}^w_k$, we have $\hat{z}_k\in\Mp{-1.5, -1}$ with probability at least $\Phi\Sp{1.5}-\Phi\Sp{1}=C_\Be$. Then, we will show that $\overline{Q}_{h, k}\Sp{s, a}\leq Q^*_h\Sp{s, a}$ for any $h$ with arbitrary $s, a$ and $\hat{z}_k\in\Mp{-1.5, -1}$. The proof will go by induction. For simplicity, let $L=\log\Sp{2HSAk^2}$.
	
	For $h=H+1$, the inequality holds trivially because both sides are 0. Then, by assuming $\overline{Q}_{h+1, k}\Sp{s, a}\leq Q^*_{h}\Sp{s, a}$ for any $\Sp{s, a}$ such that $n_k(h, s, a)\geq 3$, we have
	\begin{align*}
		\overline{Q}_{h, k}\Sp{s, a}&=\hat{R}^k_{h, s, a}+\inner{\hat{P}^k_{h, s, a}, \overline{V}_{h+1, k}}+\sigma^k_\Be\Sp{h, s, a}\hat{z}_k\\
		&\leq R_{h, s, a}+\inner{\hat{P}^k_{h, s, a}, \overline{V}_{h+1, k}}-\Sp{4\sqrt{\frac{\mathbb{V}\Sp{\tilde{P}^k_{h, s, a}, \overline{V}_{h+1, k}}L}{n_k\Sp{h, s, a}+1}}+\frac{64HL}{n_k\Sp{h, s, a}+1}}\tag{Replace $\hat{R}_{h, s, a}$ by $R_{h, s, a}$ through applying event $\mathcal{E}^1_k$ defined in (\ref{def:event_e1})}\\
		&\overset{\text{(a)}}{\leq} R_{h, s, a}+\inner{\hat{P}^k_{h, s, a}, V^*_{h+1}}-\max\Bp{4\sqrt{\frac{\mathbb{V}\Sp{\tilde{P}^k_{h, s, a}, V^*_{h+1}}L}{n_k\Sp{h, s, a}+1}}, \frac{64HL}{n_k\Sp{h, s, a}+1}}\\
		&\leq R_{h, s, a}+\inner{\hat{P}^k_{h, s, a}, V^*_{h+1}}-\Sp{3\sqrt{\frac{\mathbb{V}\Sp{\tilde{P}^k_{h, s, a}, V^*_{h+1}}L}{n_k\Sp{h, s, a}+1}}+\frac{8HL}{n_k\Sp{h, s, a}+1}}\tag{By applying statement (iii) of Lemma \ref{lemma:mono}}\\
		&\leq R_{h, s, a}+\inner{\hat{P}^k_{h, s, a}, V^*_{h+1}}-\sqrt{\frac{6\mathbb{V}\Sp{\tilde{P}^k_{h, s, a}, V^*_{h+1}}}{n_k\Sp{h, s, a}+1}}-\frac{8HL}{n_k\hsa+1}\\
		&\leq R_{h, s, a}+\inner{P_{h, s, a}, V^*_{h+1}}\tag{By applying event $\mathcal{E}_k^2$ defined in (\ref{def:event_e2})}\\
		&=Q^*_h\Sp{s, a}.
	\end{align*}
	Here, the above inequality (a) holds by applying inductive hypothesis and statement (i) in Lemma \ref{lemma:mono}. It is applicable because when $\mathcal{E}^w_k$ holds, by the clipping function, $\Norm{\overline{V}_{h+1, k}}_{\infty}\leq 2H$. When $n_k(h, s, a)<3$, $\overline{Q}_{h, k}(s, a)\leq Q^*_h(s, a)$ holds trivially because $0\leq Q^*_h(s, a)\leq H$ by definition. Therefore, the induction is complete.
	
	Now, for arbitrary $(k, h, s)$, set $a=\argmax_{a\in\A}\overline{Q}_{h,k}(s,a)$ and we have
	\begin{align*}
	    \ov_{h,k}(s)=&\clip_{2(H-h+1)}\Sp{\overline{Q}_{h,k}(s,a)}\\
        \leq & \max \Bp{-2(H-h+1),\overline{Q}_{h,k}(s,a)}\\
        \leq & \max \Bp{-(H-h+1),\overline{Q}_{h,k}(s,a)}\\
        \leq & \max \Bp{-(H-h+1),Q^*_{h,k}(s,a)}\\
        \leq & \max \Bp{-(H-h+1),Q^*_{h,k}(s,\pi^*_h(s))}\\
        \leq & V^*_{h,k}(s).
	\end{align*}
\end{proof}
\section{Regret Decomposition}
\label{sec:regret_decomp}

In this section, we prove the multiple lemmas necessary for bounding the regret. The regret is mainly composed of two terms, the pessimism term and the estimation error term. The pessimism term, $V^*_{1,k}(s_1^k) - \overline{V}_{1,k}(s_1^k)$, measures how much regret is due to the value the algorithm uses, $\overline{V}_{1,k}$, is smaller than the true value, $V^*_{1,k}(s_1^k)$.  The estimation error term, $\bV_{1,k}(s_1^k) - V_{1,k}^{\pi^k}(s_1^k)$ measure how much regret is due to the value, $\overline{V}_{1,k}$, does not estimate $V_{1,k}^{\pi^k}(s_1^k)$, the true value of the policy $\pi^k$ accurately. 

We first introduce a few definitions key to this section. In this section, we omit $k$ if it is clear from the context. Let $a_h^k = \pi^k_h(s_h^k)$ unless specified otherwise.
 
\begin{definition}
\label{def:curly_PR}
Let $\Pp^k_{h, s, a} = \inner{\hP^k_{h, s, a} - P_{h, s, a}, V^*_{h+1}}$ and $\Rr^k_{h, s, a} = \hR^k_{h, s, a} - R_{h, s, a} $.
\end{definition}

% \begin{definition}[$M'^k_\ty$ and $\bV'_{h,k}$] Given history $\mathcal{H}_H^{k-1}$, $\hat{P}^k$ and $\hR^k$, we define independent Gaussian noise term $w'^k_\ty \hsa |\mathcal{H}_H^{k-1} = \sigma^k_\ty\hsa \hat{z}'_k$ where $\hat{z}'_k\sim \mathcal{N}(0, 1)$, empirical MDP $M'^k_\ty = (H, \S, \A,\hP^k, \hR^k + w'^k_\ty, s_1^k)$, and $\ov'_{h,k}$ the optimal value of $M'^k$. Note that $w'^k_\ty\hsa$ can be different from the Gaussian noise generated in the algorithm $w^k\hsa$, but $w^k_\ty\hsa$ and $w'^k_\ty\hsa$ follow the same distribution. We define $w'^k_\ty\hsa$ for the purpose of studying its statistical properties. Let $\mathcal{G}'_k$ denote the counterpart of $\mathcal{G}_k$ about MDP $M'^k_\ty$.
% \end{definition}

\begin{definition}[$\um^k_\ty$ and $\uv_{h,k}$]
\label{def:V_under}
Given history $\mathcal{H}_H^{k-1}$ (defined in equation \eqref{equ:def_his}), $\hat{P}^k$ and $\hR^k$, we define $\uw^k_\ty\hsa = -\gamma^k_\ty\hsa$ and $\uv_{h,k}$ be the value function obtained by running policy $\pi^k$ on the MDP $\um^k_\ty = (H, \S, \A,\hP^k, \hR^k + \uw^k_\ty, s_1^k)$ plus a magnitude clipping with threshold $2(H-h+1)$.

\end{definition}

\begin{definition}[$\upm^k_\ty$ and $\upv_{h,k}$]
\label{def:V_up}
Given history $\mathcal{H}_H^{k-1}$ (defined in equation \eqref{equ:def_his}), $\hat{P}^k$ and $\hR^k$, we define $\upw^k_\ty\hsa = \gamma^k_\ty\hsa$ and $\upv_{h,k}$ be the value function obtained by running policy $\pi^k$ on the MDP $\upm^k_\ty = (H, \S, \A,\hP^k, \hR^k + \upw^k_\ty, s_1^k)$ plus a magnitude clipping with threshold $2(H-h+1)$.

Similar to Lemma \ref{lemma:q_bar_bound}, we can also show that under good event $\G_k$ and $\mathcal{E}^{th}_{h, k}$, no clipping happens on $s_h^k$ for $\uv_{h, k}(s_h^k)$ and $\upv_{h, k}(s_h^k)$.

% In the following part, we will ignore the episode $k$ when it is clear in the context. From the proof of Lemma \ref{lemma:ho_optimism}, \ref{lemma:be_optimism}, \ref{lemma:be_pessimism}, \ref{lemma:ho_pessimism}, we can easily derive that for all $h\in[H]$ and $s_h\in\S$, we have $\uv_h(s_h)\leq \ov_h(s_h)\leq \upv_h(s_h)$ under the good event $\mathcal{G}_k$.
\end{definition}

% \qiwen{
\begin{lemma}
\label{lemma:V_bar_upper_lower_bound}
Under the good event $\mathcal{G}_k$, we have $\uv_{h, k}(s)\leq \ov_{h, k}(s)\leq \upv_{h, k}(s)$ for all $h\in[H]$, $s\in\S$.
\end{lemma}
\begin{proof}
This is an immediate result by noticing that under good event $\G_k$, we have $\uw^k_\ty(h, s, a)\leq w^k_\ty(h, s, a)\leq \upw^k_\ty(h, s, a)$ for all $h\in[H]$ and $s\in\S$.
\end{proof}

% \begin{proof}
% We prove that $\uv_{h, k}(s)\leq \ov_{h, k}(s)$ and $\ov_{h, k}(s)\leq \upv_{h, k}(s)$ can be derived similarly. By the definition of $\mathcal{G}_k$, we have $\uw_\ty^k(h,s,a)\leq \min \{0,w_\ty^k(h,s,a)\}$ for all $(h,s,a)\in[H]\times\S\times\A$. Now we prove that $\uv_{h, }(s)\leq \ov_{h, k}(s)$ for all $h\in[H+1]$, $s\in\S$ by induction. First, we have $\uv_{H+1}(s_{H+1})= \ov_{H+1}(s_{H+1})=\bm{0}$ for all $s\in\S$. Suppose we have $\uv_{h+1, k}(s)\leq \ov_{h+1, k}(s)$ for all $s\in\S$. Then for all $s\in\S$, we have
% \begin{align*}
%     \uv_{h, k}(s)=&\hat{R}_{h,s,\pi^k(s)}^k+\inner{\widehat{P}_{h,s,\pi^k(s)}^k,\uv_{h+1,k}}+\uw_\ty^k(h,s,\pi^k(s))\\
%     % \leq& \max_{a\in\A}\Bp{\hat{R}_{h,s,a}^k+\inner{\widehat{P}_{h,s,a}^k,\uv_{h+1,k}}+\uw_\ty^k(h,s,a)}\\
%     \leq& \hat{R}_{h,s,\pi^k(s)}^k+\inner{\widehat{P}_{h,s,\pi^k(s)}^k,\ov_{h+1,k}}+\min\{w_\ty^k(h,s,\pi^k(s)),0\} \tag{Induction hypothesis}\\
%     \leq& \min\Bp{\overline{Q}_{h,k}(s,\pi^k(s)),2(H-h)+1}\tag{Since $\ov_{h+1, k}$ is clipped}\\
%     \leq & \min\Bp{\overline{Q}_{h,k}(s,\pi^k(s)),2(H-h+1)}\\
%     \leq & \clip_{2(H-h+1)}\Sp{\max_{a\in\A}\Bp{\overline{Q}_{h,k}(s,a)}}\\
%     =& \ov_{h, k}(s).
% \end{align*}
% By induction we complete the proof.
% \end{proof}
% }

\begin{definition}
Define $\ud^\pi_h(s_h)$, $\od^\pi_h(s_h)$, $\upd^\pi_h(s_h)$, $\delta^\pi_h(s_h)$ ,$\ud_h(s_h)$, $\od_h(s_h)$ and $\upd_h(s_h)$ as
\begin{align*}
    % \ud_h(s_h) &  = V^*_h(s_h) - \uv_h(s_h), \;\;
    % \od_h(s_h)   = V^*_h(s_h) - \bv_h(s_h), \\
    \ud^\pi_h(s_h) & = \uv_h(s_h) -V_h^\pi(s_h),\\
    \od_h^\pi(s_h)  &= \bv_h(s_h) - V_h^\pi(s_h),\\
    \upd^\pi_h(s_h) &= \upv(s_h) - V_h^\pi(s_h),\\
    \delta^\pi_h(s_h) &=V^*_h(s_h) - V^\pi_h(s_h),\\
    \ud_h(s_h)&= \uv_h(s_h)-V_h^*(s_h),\\
    \od_h(s_h)&= \ov_h(s_h)-V_h^*(s_h),\\
    \upd_h(s_h)&= \upv_h(s_h)-V_h^*(s_h).
\end{align*}
\end{definition}

\begin{definition}
We denote the history trajectory $\oh_h^k =\mathcal{H}_h^k\cup\Bp{ \hat{z}_k}$. With filtration sets $\Bp{\oh_h^k}_{h,k}$, we define the following sequences:
\begin{align*}
    \M_{\delta_h(s_h)} & = \mathds{1}\{\G_k\cap\mathcal{E}^{cum}_{h, k}\} \Mp{\inner{P_{h, s_h, a_h}, \delta_{h+1}}-{\delta_{h+1}(s_{h+1})}},
\end{align*}
where $\delta\in\{\ud^\pi,\od^\pi,\upd^\pi,\delta^\pi,\ud,\od,\upd\}$. We will show the sequences are martingales in Lemma~\ref{lem:martingale_bd}.
\end{definition}

% \begin{definition}
% \label{def:M_w}
% With the filtration $\Bp{\mathcal{H}_H^{k-1}}_k$, we define the following martingale difference sequence,
% \begin{align*}
%     \M_{h, k}^w = \mathds{1}\{\G_k\}\Mp{\E_{\hat{z}'}\Mp{V'_\hk}(s_h^k)- \ov_\hk(s_h^k)}.
% \end{align*}
% $\E_{\hat{z}'}\Mp{\cdot}$ is the expectation condition on good event $\mathcal{G}_k'$ and taken over $\hat{z}'$ with other randomness fixed.
% \end{definition}

Finally, the regret can be decomposed as
\begin{align*}
	&\mathrm{Regret}\Sp{M, K, \mathsf{SSR}_\ty}\\=&\sum_{k=1}^{K}\Sp{V^*_1(s_1^k)-V^{\pik}_{1, k}(s_1^k)}\\
	=&\sum_{k=1}^{K}\mathds{1}\Sp{\mathcal{C}^k_\ty}\Sp{V^*_1(s_1^k)-V^{\pik}_{1, k}(s_1^k)}+\sum_{k=1}^{K}\mathds{1}\Sp{\Sp{\mathcal{C}^k_\ty}^c}\Sp{V^*_1(s_1^k)-V^{\pik}_{1, k}(s_1^k)}\\
	=&\sum_{k=1}^{K}\mathds{1}\Bp{\mathcal{C}^k_\ty}\Biggl(\underbrace{V^*_{1, k}(s_1^k)-\ov_{1,k}(s_1^k)}_{\text{pessimism term}=-\od_{1, k}(s_1^k)}+\underbrace{\ov_{1,k}(s_1^k)-V^\pik_{1, k}(s_1^k)}_{\text{estimation error term}=\od^{\pik}_{1, k}(s_1^k)}\Biggr)\\
	&\qquad+\underbrace{\sum_{k=1}^{K}\mathds{1}\Sp{\Sp{\mathcal{C}^k_\ty}^c}\Sp{V^*_1(s_1^k)-V^{\pik}_{1, k}(s_1^k)}}_{\text{(a)}}.
\end{align*}
By Lemma \ref{lemma:C_be} and \ref{lemma:C_ho}, we know that
$$\E\Mp{\sum_{k=1}^{K}\mathds{1}\Sp{\Sp{\mathcal{C}^k_\ty}^c}}=\sum_{k=1}^{K}\P\Sp{\Sp{\mathcal{C}^k_\ty}^c}\leq \sum_{k=1}^{\infty}\P\Sp{\Sp{\mathcal{C}^k_\ty}^c}\leq\frac{\pi^2}{3}.$$
Therefore, by standard Hoeffding's inequality, it holds with probability at least $1-\delta$ that
$$\sum_{k=1}^{K}\mathds{1}\Sp{\Sp{\mathcal{C}^k_\ty}^c}\leq\frac{\pi^2}{3}+\sqrt{\frac{\log(1/\delta)}{2K}}.$$
Since the value functions of true MDP is bounded in $\Mp{0, H}$, with probability at least $1-\delta$, we have
$$\text{(a)}\leq H\sum_{k=1}^{K}\mathds{1}\Sp{\Sp{\mathcal{C}^k_\ty}^c}\leq\frac{\pi^2H}{3}+H\sqrt{\frac{\log(1/\delta)}{2K}}=\widetilde{O}\Sp{H}.$$
% $$\E\Mp{\text{(a)}}\leq \sum_{k=1}^{K}\P\Sp{\Sp{\mathcal{C}^k_\ty}^c}\Sp{V^*_1(s_1^k)-V^{\pik}_{1, k}(s_1^k)}\leq H\sum_{k=1}^{\infty}\P\Sp{\Sp{\mathcal{C}^k_\ty}^c}\leq\frac{\pi^2H}{3},$$

Further, notice that the good event $\mathcal{G}_k=\mathcal{C}^k_\ty\cap\mathcal{E}^w_k$ and by Lemma \ref{lemma:z_hat_bound}, we have $\sum_{k=1}^{\infty}\P\Sp{\Sp{\mathcal{E}^w_k}^c}\leq\frac{\pi^2}{3}$. Therefore, we can similarly address the regret incurred by $\Sp{\mathcal{E}^w_k}^c$ as the bound for term (a). 
% Meanwhile, The warm-up term can be bounded by $\widetilde{O}\Sp{H^4SA}$ as shown in Lemma \ref{lemma:warm_up}. 
As a result, it will be sufficient to only consider $\ogk(V^*_{1, k}(s_1^k)-V^\pik_{1, k}(s_1^k))$ when bounding pessimism and estimation error terms. That is, with probability at least $1-\delta$, it holds that
\begin{align}
	\mathrm{Regret}\Sp{M, K, \mathsf{SSR}_\ty}&\leq\sum_{k=1}^{K}\ogk\Sp{\abs{\od_{1,k}(s_1^k)}+\abs{\od_{1,k}^\pik(s_1^k)}}+\widetilde{O}\Sp{H}.\label{equ:regret_decomp1}
\end{align}

% Similarly, we can prove that
% \begin{align}
% 	\mathrm{Regret}\Sp{M, K, \mathsf{SSR}_\ty}&\leq\sum_{k=1}^{K}\mathbf{1}\{\mathcal{G}'_k\}\Sp{\abs{\od_{1,k}(s_1^k)}+\abs{\od_{1,k}^\pik(s_1^k)}}+\widetilde{O}\Sp{H^4SA},
% \end{align}
% where $\mathcal{G}'_k$ is the good event in the i.i.d. sampled MDP $\mathcal{M}^{'k}_\ty$. Combining these two argument, we have
% \begin{align}
% 	\mathrm{Regret}\Sp{M, K, \mathsf{SSR}_\ty}&\leq\sum_{k=1}^{K}\ogk\mathbf{1}\{\mathcal{G}'_k\}\Sp{\abs{\od_{1,k}(s_1^k)}+\abs{\od_{1,k}^\pik(s_1^k)}}+\widetilde{O}\Sp{H^4SA}.
% \end{align}
% In the following sections, we will bound the regret when good events $\mathcal{G}_k$ and $\mathcal{G}'_k$ happen.

Then, we decompose the estimation error term in Section \ref{subsec:estimation}. We decompose the pessimism term in Section \ref{subsec:pessimism}. We combine the decomposition of the pessimism term and the estimation error term in Section \ref{subsec:decomp}.

\subsection{Pessimism Term}
\label{subsec:pessimism}
\begin{lemma}
	\label{lem:pessimism}
	Let $C_1 = \max\Bp{\frac{1}{\Phi\Sp{1.9}-\Phi(1)}, \frac{1}{\Phi\Sp{1.5}-\Phi\Sp{1}}}=\frac{1}{\Phi\Sp{1.5}-\Phi\Sp{1}}\approx 10.9$. Then, for any $h,k,s_h^k$ and the type of noise we used, under the good event $\mathcal{G}_k$, the following bound holds,
	\begin{align}
		%\mathds{1}\{\G_k\}\Sp{V^*_{h,k}(s_h^k) - \bv_{h,k}(s_h^k)} & \leq C_1\mathds{1}\{\G_k\}\Sp{\bv_\hk(s_h^k) -V^{\pik}_{h, k}(s_h^k)+V^\pik_{h, k}(s_h^k)-\uv_\hk(s_h^k) +\M_\hk^w}, \label{eq:pes1}\\
		%\mathds{1}\{\G_k\}\Sp{V^*_{h,k}(s_h^k) - \bv_{h,k}(s_h^k)} & \geq C_1\mathds{1}\{\G_k\}\Sp{\bv_\hk(s_h^k) -V^{\pik}_{h, k}(s_h^k)+V^\pik_{h, k}(s_h^k)-\upv_\hk(s_h^k) +\M_\hk^w},\label{eq:pes2}\\
		\ogk\abs{\od_{h,k}(s_h^k)} & \leq \ogk C_1\Sp{\abs{\upd^\pik_\hk(s_h^k)}+\abs{\ud^\pik_\hk(s_h^k)}}.
	\end{align}
\end{lemma}

\begin{proof}
	Let $\Oo_{k}$ be the event that $\ov_{h, k}(s) \geq V^*_{h}(s)$ simultaneously for all $s\in\S$ and $h\in [H]$. By Lemma \ref{lemma:ho_optimism} and \ref{lemma:be_optimism}, we know that $\P\Sp{\Oo_k\mid\mathcal{H}^{k-1}_H, \mathcal{G}_k}\geq \min\Bp{\Phi(1.9)-\Phi(1), \Phi(1.5)-\Phi(1)}=\Phi(1.5)-\Phi(1)$, which means $\frac{1}{\P\Sp{\Oo_k}}\leq C_1$ regardless the type of noise used. 
	
	The definition of $\Oo_k$ implies $V^*_h\leq \E\Mp{\ov_{h, k}\mid \Oo_k, \mathcal{H}^{k-1}_H, \mathcal{G}_k}$. Meanwhile, notice that
	\begin{align*}
		&\ogk\Sp{\E\Mp{\ov_{h, k}\mid\mathcal{H}^{k-1}_H, \mathcal{G}_k}-\uv_{h, k}}\\
		=&\ogk\P\Sp{\Oo_k\mid \mathcal{H}_H^{k-1}, \G_k}\Sp{\E\Mp{\ov_{h, k}\mid \Oo_k, \mathcal{H}^{k-1}_H, \mathcal{G}_k}-\uv_{h, k}}\\
		&\qquad+\ogk\P\Sp{\Sp{\Oo_k}^c\mid \mathcal{H}^{k-1}_H, \mathcal{G}_k}\underbrace{\Sp{\E\Mp{\ov_{h, k}\mid \Sp{\Oo_k}^c, \mathcal{H}^{k-1}_H, \mathcal{G}_k}-\uv_{h, k}}}_{\text{(a)}\geq \bm{0}}\\
		\geq& \ogk\P\Sp{\Oo_k\mid \mathcal{H}_H^{k-1}, \G_k}\Sp{\E\Mp{\ov_{h, k}\mid \Oo_k, \mathcal{H}^{k-1}_H, \mathcal{G}_k}-\uv_{h, k}}
	\end{align*}
	$$\implies \ogk\Sp{\E\Mp{\ov_{h, k}\mid \Oo_k, \mathcal{H}^{k-1}_H, \mathcal{G}_k}-\uv_{h, k}}\leq \ogk C_1\Sp{\E\Mp{\ov_{h, k}\mid\mathcal{H}^{k-1}_H, \mathcal{G}_k}-\uv_{h, k}}.$$
	Here, we have term $\text{(a)}\geq \bm{0}$ since $\uv_{h,k}\leq\ov_{h, k}$ under event $\mathcal{G}_k$, by Lemma \ref{lemma:V_bar_upper_lower_bound}.
	
	Therefore, we have
	\begin{align}
		\mathds{1}\{\G_k\}\Sp{V^*_{h}(s_h^k) - \ov_{h, k}(s_h^k)} & \leq \ogk\Sp{\E\Mp{\ov_{h, k}\mid \Oo_k, \mathcal{H}^{k-1}_H, \G_k}(s_h^k)-\ov_{h, k}(s_h^k)}\nonumber\\
		&\leq \ogk\Sp{\E\Mp{\ov_{h, k}\mid \Oo_k, \mathcal{H}^{k-1}_H, \G_k}(s_h^k)-\uv_{h, k}(s_h^k)}\nonumber\\
		&\leq\ogk C_1\Sp{\E\Mp{\ov_{h, k}\mid\mathcal{H}^{k-1}_H, \mathcal{G}_k}(s_h^k)-\uv_{h, k}(s_h^k)}.\label{equ:delta_upperbound}
	\end{align}
	We can similarly use constant probability pessimism shown in Lemma \ref{lemma:ho_pessimism} and \ref{lemma:be_pessimism}. In particular, let $\mathcal{N}_k$ be the event that $\ov_{h, k}(s)\leq V^*_h(s)$ for all $s\in\S$ and $h\in[H]$. Then, we have
	\begin{align*}
		&\ogk\Sp{\E\Mp{\ov_{h, k}\mid\mathcal{H}^{k-1}_H, \mathcal{G}_k}-\upv_{h, k}}\\
		=&\ogk\P\Sp{\mathcal{N}_k\mid \mathcal{H}_H^{k-1}, \G_k}\Sp{\E\Mp{\ov_{h, k}\mid \mathcal{N}_k, \mathcal{H}^{k-1}_H, \mathcal{G}_k}-\upv_{h, k}}\\
		&\qquad+\ogk\P\Sp{\Sp{\mathcal{N}_k}^c\mid \mathcal{H}^{k-1}_H, \mathcal{G}_k}\underbrace{\Sp{\E\Mp{\ov_{h, k}\mid \Sp{\mathcal{N}_k}^c, \mathcal{H}^{k-1}_H, \mathcal{G}_k}-\upv_{h, k}}}_{\text{(b)}\leq \bm{0}}\\
		\leq& \ogk\P\Sp{\mathcal{N}_k\mid \mathcal{H}_H^{k-1}, \G_k}\Sp{\E\Mp{\ov_{h, k}\mid \mathcal{N}_k, \mathcal{H}^{k-1}_H, \mathcal{G}_k}-\upv_{h, k}}
	\end{align*}
	$$\implies \ogk\Sp{\E\Mp{\ov_{h, k}\mid \mathcal{N}_k, \mathcal{H}^{k-1}_H, \mathcal{G}_k}-\upv_{h, k}}\geq \ogk C_1\Sp{\E\Mp{\ov_{h, k}\mid\mathcal{H}^{k-1}_H, \mathcal{G}_k}-\upv_{h, k}}.$$
	Thus, we have
	\begin{align}
		\mathds{1}\{\G_k\}\Sp{V^*_{h}(s_h^k) - \ov_{h, k}(s_h^k)} & \geq \ogk\Sp{\E\Mp{\ov_{h, k}\mid \mathcal{N}_k, \mathcal{H}^{k-1}_H, \G_k}(s_h^k)-\upv_{h, k}(s_h^k)}\nonumber\\
		&\geq\ogk C_1\Sp{\E\Mp{\ov_{h, k}\mid\mathcal{H}^{k-1}_H, \mathcal{G}_k}(s_h^k)-\upv_{h, k}(s_h^k)}.\label{equ:delta_lowerbound}
	\end{align}
	Since good event $\G_k$ implies $\uv_{h, k}\leq\ov_{h, k}\leq\upv_{h, k}$ by Lemma \ref{lemma:V_bar_upper_lower_bound}, the RHS of \eqref{equ:delta_upperbound} is non-negative and the RHS of \eqref{equ:delta_lowerbound} is non-positive. Therefore, we can then conclude
	\begin{align*}
		&\ogk\abs{V^*_{h}(s_h^k) - \ov_{h, k}(s_h^k)}\\
		\leq& \ogk C_1\Sp{\Sp{\E\Mp{\ov_{h, k}\mid\mathcal{H}^{k-1}_H, \mathcal{G}_k}(s_h^k)-\uv_{h, k}(s_h^k)}-\Sp{\E\Mp{\ov_{h, k}\mid\mathcal{H}^{k-1}_H, \mathcal{G}_k}(s_h^k)-\upv_{h, k}(s_h^k)}}\\
		=&\ogk C_1\Sp{\upv_{h, k}(s_h^k)-\uv_{h, k}(s_h^k)}\\
		\leq &\ogk C_1\Sp{\abs{\upd^\pik_\hk(s_h^k)}+\abs{\ud^\pik_\hk(s_h^k)}}.
	\end{align*}
\end{proof}

\subsection{Estimation Error Term}
\label{subsec:estimation}

We first bound the estimation error of $\upv$, which can be regarded as the optimistic estimate used in UCB-type algorithms. For convenience, we will ignore notation $\ogk$ in this section since all statements are proved under the good event $\mathcal{G}_k$.

\begin{lemma}
\label{lem:upd^pi}
With probability at least $1-\delta$, for all $(k, h, s_i^k)$, under the good event $\mathcal{G}_k$ it holds that
\fontsize{9}{9}
\begin{align*}
    &\mathds{1}\Bp{\mathcal{E}^{cum}_{h, k}}\abs{\upd_{h,k}^\pik(s_h^k)}\\
    &\leq \mathds{1}\Bp{\mathcal{E}^{cum}_{h, k}}\Sp{\abs{\mathcal{P}^k_{h, s_h^k, a_h^k}+\mathcal{R}^k_{h, s_h^k, a_h^k}+\upw_\ty^k(\hsahk)} + \M_{\abs{\upd^{\pik}_{h, k}(s_{h}^k)}}+\M_{\abs{\upd_{h,k}(s_{h}^k)}}+\frac{2SH^2L}{n_k(h,s_h^k,a_h^k)}}\\
	&+\mathds{1}\Bp{\mathcal{E}^{cum}_{h+1, k}}\Sp{\frac{C_1}{H}\abs{\ud_{h+1,k}^{\pi^k}(s_{h+1}^k)}+\frac{H+1+C_1}{H}\abs{\upd_{h+1,k}^{\pi^k}(s_{h+1}^k)}+\frac{1}{H}\abs{\od_{h+1,k}^{\pi^k}(s_{h+1}^k)}}\\
	&+\mathds{1}\Bp{\mathcal{E}^{cum}_{h, k}\cap\Sp{\mathcal{E}^{th}_{h+1, k}}^c}\Sp{\frac{C_1}{H}\abs{\ud_{h+1,k}^{\pi^k}(s_{h+1}^k)}+\frac{H+1+C_1}{H}\abs{\upd_{h+1,k}^{\pi^k}(s_{h+1}^k)}+\frac{1}{H}\abs{\od_{h+1,k}^{\pi^k}(s_{h+1}^k)}},
\end{align*}
\normalsize
where $L=\log(2HS^2AK/\delta)$.
% where for $b_{\mathrm{ty}}(h', s_{h'}, a_{h'})$, there are two following choices
% \begin{equation}
% 	\label{equ:b_1}
% 	\ob_{\Ho}(\hsah)=3H\sqrt{\frac{\log\Sp{2HSAK/\delta}}{n(\hsah)+1}}+\frac{3H}{n(\hsah)+1},
% \end{equation}
% \begin{equation}
% 	\label{equ:b_2}
% 	\ob_\Be(\hsah)=\sqrt{\frac{6\bbV\Sp{\tilde{P}_{\hsah}, \overline{V}_{h+1}-V^*_{h+1}}\log\Sp{2HSAK/\delta}}{n(\hsah)+1}}+\frac{24H\log\Sp{2HSAK/\delta}}{n(\hsah)+1}.
% \end{equation}
\end{lemma}
\begin{proof}
	Since both $\upv$ and $V^\pik$ are obtained by choosing actions based on policy $\pi^k$ under event $\mathcal{G}_k$, we have
% 	\fontsize{9}{9}
	\begin{align}
	    &\mathds{1}\Bp{\mathcal{E}^{cum}_{h, k}}\abs{\upd_{h,k}^\pik(s_h^k)}\nonumber\\
	    =&\mathds{1}\Bp{\mathcal{E}^{cum}_{h, k}}\abs{\upv_{h, k}(s_h^k) - V_{h, k}^{\pik}(s_h^k)}\nonumber\\
		=&\mathds{1}\Bp{\mathcal{E}^{cum}_{h, k}}\abs{\overline{\overline{Q}}_{h, k}(s_h^k, a_h^k)-Q^\pik_{h, k}(s_h^k, a_h^k)}\tag{Since no clipping under $\mathcal{E}^{cum}_{h, k}$ for $\upv_{h, k}(s_h^k)$}\\
		=&\mathds{1}\Bp{\mathcal{E}^{cum}_{h, k}}\abs{\hat{R}^k_{\hsahk}-R_{\hsahk}+\upw_\ty^k(\hsahk)+\inner{\hat{P}^k_{\hsahk}, \upv_{h+1, k}}-\inner{P^k_{\hsahk}, V^\pik_{h+1, k}}}\nonumber\\
		=&\mathds{1}\Bp{\mathcal{E}^{cum}_{h, k}}\left|\hat{R}^k_{\hsahk}-R_{\hsahk}+\upw_\ty^k(\hsahk)+\inner{\hat{P}^k_{\hsahk}, \upv_{h+1, k}}-\inner{P_{\hsahk}, V^\pik_{h+1, k}}\right.\nonumber\\
		&\qquad+\left.\inner{\hat{P}^k_{\hsahk}-P_{\hsahk}, V^*_{h+1}}-\inner{\hat{P}^k_{\hsahk}-P_\hsahk, V^*_{h+1}}\right|\nonumber\\
		\leq& \mathds{1}\Bp{\mathcal{E}^{cum}_{h, k}}\Sp{\abs{\mathcal{P}^k_\hsahk+\mathcal{R}^k_\hsahk+\upw_\ty^k(\hsahk)}+\inner{P_{\hsahk}, \abs{\upv_{h+1, k}-V^\pik_{h+1, k}}}}\label{equ:bifur}\nonumber\\
		&\qquad+\mathds{1}\Bp{\mathcal{E}^{cum}_{h, k}}\abs{\inner{\hat{P}^k_\hsahk-P_{\hsahk}, \upv_{h+1, k}-V^*_{h+1}}}\nonumber\\
		=& \mathds{1}\Bp{\mathcal{E}^{cum}_{h, k}}\Sp{\abs{\Pp^k_{\hsahk} + \Rr^k_{\hsahk} + \upw_\ty^k(\hsahk)}+\abs{\upd^\pik_{h+1, k}(s_{h+1}^k)} + \M_{\abs{\upd^{\pik}_{h, k}(s_{h}^k)}}}\nonumber\\
		&\qquad+\mathds{1}\Bp{\mathcal{E}^{cum}_{h, k}}\abs{\inner{\hat{P}^k_\hsahk-P_{\hsahk}, \upv_{h+1, k}-V^*_{h+1}}}.\nonumber
% 		\overset{\text{(i)}}{\leq}&\ogk\Mp{\mathcal{P}_\hsah+\mathcal{R}_\hsah+w_\ty(\hsah)+\od^\pi_{h+1}(s_{h+1})+\M_{\od^\pi_h(s_h)}+\ob_{\mathrm{ty}}(h, s_h, a_h)}.\nonumber
	\end{align}
	\normalsize
	For the last term, we use Lemma \ref{lemma:p_bound} and then for $L=\log(2HS^2AK/\delta)$, with probability at least $1-\delta$, we have
	\fontsize{9.5}{9.5}
	\begin{align*}
	    &\abs{\inner{\hat{P}^k_\hsahk-P_{\hsahk}, \upv_{h+1, k}-V^*_{h+1}}}\\
	    \leq& \sum_{s_{h+1}\in\S}\abs{\hat{P}^k_\hsahk(s_{h+1})-P_\hsahk(s_{h+1})}\abs{\upv_{h+1,k}(s_h+1)-V_{h+1}^*(s_{h+1})}\\
	    \leq& \sum_{s_{h+1}\in\S}\Sp{2\sqrt{\frac{P_{\hsahk}(s_{h+1})L}{n_k(h,s_h^k,a_h^k)}}+\frac{4L}{3n_k(h,s_h^k,a_h^k)}}\abs{\upd_{h+1,k}(s_{h+1})}\\
	    =& \sum_{s_{h+1}:P_{\hsahk}(s_{h+1})n_k(h,s_h^k,a_h^k)\geq 4LH^2}2P_{\hsahk}(s_{h+1})\sqrt{\frac{L}{P_{\hsahk}(s_{h+1})n_k(h,s_h^k,a_h^k)}}\abs{\upd_{h+1,k}(s_{h+1})}\\
	    &+\sum_{s_{h+1}:P_{\hsahk}(s_{h+1})n_k(h,s_h^k,a_h^k)< 4LH^2}2\sqrt{\frac{LP_{\hsahk}(s_{h+1})n_k(h,s_h^k,a_h^k)}{n_k(h,s_h^k,a_h^k)^2}}\abs{\upd_{h+1,k}(s_{h+1})}\\
	    &+\frac{4SHL}{3n_k(h,s_h^k,a_h^k)}\\
	    \leq& \sum_{s_{h+1}\in\S}P_{\hsahk}(s_{h+1})\frac{1}{H}\abs{\upd_{h+1,k}(s_{h+1})}+\frac{4SHL+2SH^2\sqrt{L}}{3n_k(h,s_h^k,a_h^k)}\\
	    \leq& \frac{1}{H}\abs{\upd_{h+1,k}(s_{h+1}^k)}+\M_{\abs{\upd_{h,k}(s_{h}^k)}}+\frac{2SH^2L}{n_k(h,s_h^k,a_h^k)}\\
	    \leq& \frac{1}{H}\abs{\upd_{h+1,k}^{\pi^k}(s_{h+1}^k)}+\frac{1}{H}\abs{\od_{h+1,k}^{\pi^k}(s_{h+1}^k)}+\frac{1}{H}\abs{\od_{h+1,k}(s_{h+1}^k)}+\M_{\abs{\upd_{h,k}(s_{h}^k)}}+\frac{2SH^2L}{n_k(h,s_h^k,a_h^k)}\tag{By triangle inequality}\\
	    \leq& \frac{1+C_1}{H}\abs{\upd_{h+1,k}^{\pi^k}(s_{h+1}^k)}+\frac{1}{H}\abs{\od_{h+1,k}^{\pi^k}(s_{h+1}^k)}+\frac{C_1}{H}\abs{\ud_{h+1,k}^{\pi^k}(s_{h+1}^k)}+\M_{\abs{\upd_{h,k}(s_{h}^k)}}+\frac{2SH^2L}{n_k(h,s_h^k,a_h^k)}.\tag{By using Lemma \ref{lem:pessimism}}
	\end{align*}
	\normalsize
	Combining the above two arguments, we can prove the argument:
	\fontsize{9}{9}
	\begin{align*}
	    &\mathds{1}\Bp{\mathcal{E}^{cum}_{h, k}}\abs{\upd_{h,k}^\pik(s_h^k)}\\
	    \leq&\mathds{1}\Bp{\mathcal{E}^{cum}_{h, k}}\Sp{ \abs{\Pp^k_{\hsahk} + \Rr^k_{\hsahk} + \upw_\ty^k(\hsahk)} + \M_{\abs{\upd^{\pik}_{h, k}(s_{h}^k)}}+\M_{\abs{\upd_{h,k}(s_{h}^k)}}+\frac{2SH^2L}{n_k(h,s_h^k,a_h^k)}}\nonumber\\
		&\quad+\mathds{1}\Bp{\mathcal{E}^{cum}_{h, k}}\Sp{\frac{H+1+C_1}{H}\abs{\upd_{h+1,k}^{\pi^k}(s_{h+1}^k)}+\frac{1}{H}\abs{\od_{h+1,k}^{\pi^k}(s_{h+1}^k)}+\frac{C_1}{H}\abs{\ud_{h+1,k}^{\pi^k}(s_{h+1}^k)}}.
	\end{align*}
	\normalsize
	Then, the proof is complete by noticing that $\mathcal{E}^{cum}_{h+1, k}=\mathcal{E}^{cum}_{h, k}\cap\mathcal{E}^{th}_{h+1, k}$.
\end{proof}

\begin{lemma}
\label{lem:ud^pi}
    With probability at least $1-\delta$, for all $(k, h, s_h^k)$, under good event $\mathcal{G}_k$ it holds that
    \fontsize{9}{9}
    \begin{align*}
	&\mathds{1}\Bp{\mathcal{E}^{cum}_{h, k}}\abs{\ud_{h,k}^\pik(s_h^k)}\\
	\leq& \mathds{1}\Bp{\mathcal{E}^{cum}_{h, k}}\Sp{\abs{\Pp^k_{\hsahk} + \Rr^k_{\hsahk} + \uw_\ty^k(\hsahk)}+\M_{\abs{\ud^{\pik}_{h, k}(s_{h}^k)}}+\M_{\abs{\ud_{h,k}(s_{h}^k)}}+\frac{2SH^2L}{n_k(h,s_h^k,a_h^k)}}\\
    &+\mathds{1}\Bp{\mathcal{E}^{cum}_{h+1, k}}\Sp{\frac{C_1}{H}\abs{\upd_{h+1,k}^{\pi^k}(s_{h+1}^k)}+\frac{H+1+C_1}{H}\abs{\ud_{h+1,k}^{\pi^k}(s_{h+1}^k)}+\frac{1}{H}\abs{\od_{h+1,k}^{\pi^k}(s_{h+1}^k)}}\\
    &+\mathds{1}\Bp{\mathcal{E}^{cum}_{h, k}\cap\Sp{\mathcal{E}^{th}_{h+1, k}}^c}\Sp{\frac{C_1}{H}\abs{\upd_{h+1,k}^{\pi^k}(s_{h+1}^k)}+\frac{H+1+C_1}{H}\abs{\ud_{h+1,k}^{\pi^k}(s_{h+1}^k)}+\frac{1}{H}\abs{\od_{h+1,k}^{\pi^k}(s_{h+1}^k)}}.
	\end{align*}
	\normalsize
\end{lemma}

\begin{proof}
    The proof exactly follows the proof of Lemma \ref{lem:upd^pi}.
\end{proof}

\begin{lemma}
\label{lem:od^pi}
    With probability at least $1-\delta$, for all $(k, h, s_h^k)$, under good event $\mathcal{G}_k$ it holds that
    \fontsize{9}{9}
    \begin{align*}
	&\mathds{1}\Bp{\mathcal{E}^{cum}_{h, k}}\abs{\od_{h,k}^\pik(s_h^k)}\\
	\leq& \mathds{1}\Bp{\mathcal{E}^{cum}_{h, k}}\Sp{\abs{\Pp^k_{\hsahk} + \Rr^k_{\hsahk} + w_\ty^k(\hsahk)}+\M_{\abs{\od^{\pik}_{h, k}(s_{h}^k)}}+\M_{\abs{\od_{h,k}(s_{h}^k)}}+\frac{2SH^2L}{n_k(h,s_h^k,a_h^k)}}\\
    &+\mathds{1}\Bp{\mathcal{E}^{cum}_{h+1, k}}\Sp{\frac{C_1}{H}\abs{\upd_{h+1,k}^{\pi^k}(s_{h+1}^k)}+\frac{C_1}{H}\abs{\ud_{h+1,k}^{\pi^k}(s_{h+1}^k)}+\frac{H+1}{H}\abs{\od_{h+1,k}^{\pi^k}(s_{h+1}^k)}}\\
    &+\mathds{1}\Bp{\mathcal{E}^{cum}_{h, k}\cap\Sp{\mathcal{E}^{th}_{h+1, k}}^c}\Sp{\frac{C_1}{H}\abs{\upd_{h+1,k}^{\pi^k}(s_{h+1}^k)}+\frac{C_1}{H}\abs{\ud_{h+1,k}^{\pi^k}(s_{h+1}^k)}+\frac{H+1}{H}\abs{\od_{h+1,k}^{\pi^k}(s_{h+1}^k)}}.
	\end{align*}
	\normalsize
\end{lemma}

\begin{proof}
    The proof exactly follows the proof of Lemma \ref{lem:upd^pi}.
\end{proof}

\begin{lemma}
\label{lem:sum of delta}
With probability at least $1-\delta$, for all $(k, i, s_i^k)$, under good event $\mathcal{G}_k$ it holds that
    \begin{align*}
        &\mathds{1}\Bp{\mathcal{E}^{cum}_{i, k}}\Sp{\abs{\od_{i,k}^\pik(s_i^k)}+\abs{\upd_{i,k}^\pik(s_i^k)}+\abs{\ud_{i,k}^\pik(s_i^k)}}\\
        \leq& 3e^{3C_1}\Sp{\sum_{h=i}^H \sqrt{e_\ty^k(\hsahk)}+\sum_{h=i}^H \gamma_\ty^k(\hsahk)+\sum_{h=i}^H\frac{SH^2L}{n_k(h,s_h^k,a_h^k)}}\\
        &\quad+e^{3C_1}\sum_{h=i+1}^{H}\mathds{1}\Bp{\Sp{\mathcal{E}^{th}_{i+1, k}}^c}\Sp{\abs{\od_{i+1,k}^\pik(s_{i+1}^k)}+\abs{\upd_{i+1,k}^\pik(s_{i+1}^k)}+\abs{\ud_{i+1,k}^\pik(s_{i+1}^k)}}\\
        &\quad+\sum_{h=i}^H (1+\frac{3C_1}{H})^{h-1}\mathds{1}\Bp{\mathcal{E}^{cum}_{h, k}}\M_{h,k},
    \end{align*}
    \begin{align*}
        \text{where }\M_{h,k}=&\M_{\abs{\upd^{\pik}_{h, k}(s_{h}^k)}}+\M_{\abs{\upd_{h,k}(s_{h}^k)}}+\M_{\abs{\ud^{\pik}_{h, k}(s_{h}^k)}}\\
        &\qquad+\M_{\abs{\ud_{h,k}(s_{h}^k)}}+\M_{\abs{\od^{\pik}_{h, k}(s_{h}^k)}}+\M_{\abs{\od_{h,k}(s_{h}^k)}}.
    \end{align*}
    % \begin{align*}
    %     \Sp{\abs{\od_{i,k}^\pik(s_i^k)}+\abs{\upd_{i,k}^\pik(s_i^k)}+\abs{\ud_{i,k}^\pik(s_i^k)}}\leq &3e^{3C_1}\Sp{\sum_{h=i}^H \sqrt{e_\ty^k(\hsahk)}+\sum_{h=i}^H \gamma_\ty^k(\hsahk)}\\&+\widetilde{O}(\sqrt{HT}+H^3S^2A).
    % \end{align*}
\end{lemma}

\begin{proof}
    % First, we show that
    % \begin{align*}
    %     \Sp{\abs{\od_{i,k}^\pik(s_i^k)}+\abs{\upd_{i,k}^\pik(s_i^k)}+\abs{\ud_{i,k}^\pik(s_i^k)}}\leq& 3e^{3C_1}\Sp{\sum_{h=i}^H \sqrt{e_\ty^k(\hsahk)}+\sum_{h=i}^H \gamma_\ty^k(\hsahk)+\sum_{h=i}^H\frac{SH^2L}{n_k(h,s_h^k,a_h^k)}}
    %     \\&+\sum_{h=i}^H (1+\frac{3C_1}{H})^{h-1}\M_{h,k},
    % \end{align*}
    % where $\M_{h,k}=\M_{\abs{\upd^{\pik}_{h, k}(s_{h}^k)}}+\M_{\abs{\upd_{h,k}(s_{h}^k)}}+\M_{\abs{\ud^{\pik}_{h, k}(s_{h}^k)}}+\M_{\abs{\ud_{h,k}(s_{h}^k)}}+\M_{\abs{\od^{\pik}_{h, k}(s_{h}^k)}}+\M_{\abs{\od_{h,k}(s_{h}^k)}}$.
    By summing results in Lemma \ref{lem:upd^pi}, Lemma \ref{lem:ud^pi} and Lemma \ref{lem:od^pi}, we have
    \begin{align*}
        &\mathds{1}\Bp{\mathcal{E}^{cum}_{h, k}}\Sp{\abs{\od_{h,k}^\pik(s_h^k)}+\abs{\upd_{h,k}^\pik(s_h^k)}+\abs{\ud_{h,k}^\pik(s_h^k)}}\\
        \leq &\mathds{1}\Bp{\mathcal{E}^{cum}_{h+1, k}}\Sp{1+\frac{3C_1}{H}}\Sp{\abs{\od_{h+1,k}^\pik(s_{h+1}^k)}+\abs{\upd_{h+1,k}^\pik(s_{h+1}^k)}+\abs{\ud_{h+1,k}^\pik(s_{h+1}^k)}} \\
        &+\abs{w_\ty^k(\hsahk)}+\abs{\upw_\ty^k(\hsahk)}+\abs{\uw_\ty^k(\hsahk)}+\frac{6SH^2L}{n_k(h,s_h^k,a_h^k)}+\mathds{1}\Bp{\mathcal{E}^{cum}_{h, k}}\M_{h,k}\\
        &+\mathds{1}\Bp{\mathcal{E}^{cum}_{h, k}\cap\Sp{\mathcal{E}^{th}_{h+1, k}}^c}\Sp{1+\frac{3C_1}{H}}\Sp{\abs{\od_{h+1,k}^\pik(s_{h+1}^k)}+\abs{\upd_{h+1,k}^\pik(s_{h+1}^k)}+\abs{\ud_{h+1,k}^\pik(s_{h+1}^k)}}\\
        &+ 3\abs{\Pp^k_{\hsahk}+\Rr^k_{\hsahk}}\\
        \overset{(\text{i})}{\leq} & \mathds{1}\Bp{\mathcal{E}^{cum}_{h+1, k}}\Sp{1+\frac{3C_1}{H}}\Sp{\abs{\od_{h+1,k}^\pik(s_{h+1}^k)}+\abs{\upd_{h+1,k}^\pik(s_{h+1}^k)}+\abs{\ud_{h+1,k}^\pik(s_{h+1}^k)}}\\
        &+ 3\sqrt{e^k_\ty(\hsahk)}+3\gamma_\ty^k(\hsahk)+\frac{6SH^2L}{n_k(h,s_h^k,a_h^k)}+\mathds{1}\Bp{\mathcal{E}^{cum}_{h, k}}\M_{h,k}\\
        &+ \mathds{1}\Bp{\Sp{\mathcal{E}^{th}_{h+1, k}}^c}\Sp{1+\frac{3C_1}{H}}\Sp{\abs{\od_{h+1,k}^\pik(s_{h+1}^k)}+\abs{\upd_{h+1,k}^\pik(s_{h+1}^k)}+\abs{\ud_{h+1,k}^\pik(s_{h+1}^k)}}
    \end{align*}
    Here, the inequality (i) above holds because of two reasons. Firstly, under event $\mathcal{G}_k$, we have $|w^k_\ty(\hsahk)|\leq|\uw^k_\ty(\hsahk)|$ and
	\begin{align*}
		\abs{\mathcal{P}^k_{h, s_h^k, a_h^k}+\mathcal{R}^k_{h, s_h^k, a_h^k}}&=\abs{\inner{\hP^k_\hsahk - P_\hsahk, V^*_{h+1}}+\Sp{\hat{R}^k_\hsahk-R_\hsahk}}\tag{By Definition \ref{def:curly_PR}}\\
		&\leq\sqrt{e^k_\ty(\hsahk)}.\tag{Under event $\mathcal{G}_k$, $\hat{M}\in\M^k_\ty$}
	\end{align*}
    Then, the proof is complete by using this recursion from $h=i$ to $h=H$ and utilizing the fact that $(1+\frac{3C_1}{H})^H\leq e^{3C_1}$. 
    % Then we can plug in Lemma \ref{lem:martingale_bd} and Lemma \ref{lemma:sum1/n} to complete the proof.
\end{proof}

\subsection{Combining Estimation and Pessimism Terms}
\begin{lemma}
\label{lem:reg decomp}
With probability at least $1-\delta$, it holds that
\begin{align*}
  \mathrm{Regret}\Sp{M, K, \mathsf{SSR}_\ty}\leq& \ogk 3C_1e^{3C_1}\sum_{k=1}^{K}\sum_{h=i}^H \Sp{\sqrt{e_\ty^k(\hsahk)}+ \gamma_\ty^k(\hsahk)}\\
  &\quad+\widetilde{O}\Sp{H^4S^2A+H\sqrt{T}}.  
\end{align*}
\end{lemma}
\begin{proof}
Recall in equation \eqref{equ:regret_decomp1}, with probability at least $1-\delta$, we have
\begin{align*}
    &\mathrm{Regret}\Sp{M, K, \mathsf{SSR}_\ty}\\
    \leq&\sum_{k=1}^{K}\ogk\Sp{\abs{\od_{1,k}(s_1^k)}+\abs{\od_{1,k}^\pik(s_1^k)}}+\widetilde{O}\Sp{H}\\
    \leq&\sum_{k=1}^{K}\ogk C_1\Sp{\abs{\od^\pik_{1, k}(s_1^k)}+\abs{\upd^\pik_{1,k}(s_1^k)}+\abs{\ud^\pik_{1,k}(s_1^k)}}+\widetilde{O}\Sp{H}\tag{By using Lemma \ref{lem:pessimism}}\\
    =&\sum_{k=1}^{K}\mathds{1}\Bp{\G_k\cap\mathcal{E}^{cum}_{1, k}} C_1\Sp{\abs{\od^\pik_{1,k}(s_1^k)}+\abs{\upd^\pik_{1,k}(s_1^k)}+\abs{\ud^\pik_{1,k}(s_1^k)}}+\widetilde{O}\Sp{H}\\
    &\qquad +\sum_{k=1}^{K}\mathds{1}\Bp{\Sp{\mathcal{E}^{th}_{1, k}}^c}\Sp{\abs{\od^\pik_{1,k}(s_1^k)}+\abs{\upd^\pik_{1,k}(s_1^k)}+\abs{\ud^\pik_{1,k}(s_1^k)}}\\
    \leq& \ogk 3C_1e^{3C_1}\Sp{\sum_{k=1}^{K}\sum_{h=i}^H \sqrt{e_\ty^k(\hsahk)}+\sum_{k=1}^{K}\sum_{h=i}^H \gamma_\ty^k(\hsahk)+\sum_{k=1}^{K}\sum_{h=i}^H\frac{SH^2L}{n_k(h,s_h^k,a_h^k)}}\\
    &\qquad+\sum_{k=1}^{K}\sum_{h=i}^H (1+\frac{3C_1}{H})^{h-1}\mathds{1}\Bp{\G_k\cap\mathcal{E}^{cum}_{h, k}}\M_{h,k}+\widetilde{O}\Sp{H}\\
    &\qquad+ \sum_{k=1}^K\sum_{h=1}^H\mathds{1}\Bp{\Sp{\mathcal{E}^{th}_{h, k}}^c}\Sp{\abs{\od^\pik_{h, k}(s_1^k)}+\abs{\upd^\pik_\hk(s_h^k)}+\abs{\ud^\pik_\hk(s_h^k)}}\tag{By using Lemma \ref{lem:sum of delta}}\\
    \overset{\text{(i)}}{\leq}& \ogk 3C_1e^{3C_1}\sum_{k=1}^{K}\sum_{h=i}^H \Sp{\sqrt{e_\ty^k(\hsahk)}+ \gamma_\ty^k(\hsahk)}+\widetilde{O}\Sp{H^3S^2A+H\sqrt{T}}\\
    &\qquad +\widetilde{O}\Sp{H}\sum_{k=1}^K\sum_{h=1}^H\mathds{1}\Bp{\Sp{\mathcal{E}^{th}_{h, k}}^c}\\
    \leq &\ogk 3C_1e^{3C_1}\sum_{k=1}^{K}\sum_{h=i}^H \Sp{\sqrt{e_\ty^k(\hsahk)}+ \gamma_\ty^k(\hsahk)}+\widetilde{O}\Sp{H^4S^2A+H\sqrt{T}}.\tag{By using Lemma \ref{lemma:warm_up}}
\end{align*}
The inequality (i) above holds for two reasons. First, it uses Lemma \ref{lem:martingale_bd} and \ref{lemma:sum1/n}. Second, by our clipping threshold, we know that $\Sp{\abs{\od^\pik_{h, k}(s_1^k)}+\abs{\upd^\pik_\hk(s_h^k)}+\abs{\ud^\pik_\hk(s_h^k)}}\leq \widehat{O}\Sp{H}$.
\end{proof}

\begin{lemma}[Lemma 20 in \cite{agrawal2020improved}]
	\label{lemma:warm_up}
	$$\sum_{k=1}^{K}\sum_{h=1}^H\mathds{1}\Bp{\Sp{\mathcal{E}^{th}_{h, k}}^c}\leq\widetilde{O}\Sp{H^3SA}.$$
\end{lemma}
\begin{proof}
	It holds that
	\begin{align*}
		\sum_{k=1}^{K}\sum_{h=1}^H\mathds{1}\Bp{\Sp{\mathcal{E}^{th}_{h, k}}^c}&= \sum_{k=1}^{K}\sum_{h=1}^{H}\mathds{1}\Bp{n_k(\hsahk)\leq\alpha_k}\\
		&\leq \sum_{s\in\S}\sum_{a\in\A}\sum_{h=1}^{H}\alpha_k\\
		&\leq 200H^3SA\log\Sp{2HSAK^2}\log\Sp{40K^4}\tag{By our choice of $\alpha_k$}\\
		&=\widetilde{O}\Sp{H^3SA}.
	\end{align*}
\end{proof}

% \begin{lemma}
% \label{lem:reg decomp}
% For all $(k, h, s_h^k)$, under good event $\mathcal{G}_k$ it holds that
% \begin{align*}
%     \delta^\pik_{k,h}(s_h^k)\leq C_1\Sp{\abs{\od_{h,k}^\pik(s_h^k)}+\abs{\upd_{h,k}^\pik(s_h^k)}+\abs{\ud_{h,k}^\pik(s_h^k)}}.
% \end{align*}
% \end{lemma}

% \begin{proof}
% The proof is derived by using triangle inequality
% \begin{align*}
%     \delta^\pik_{k,h}(s_h^k)=&\od_{k,h}(s_h^k)+\od^\pik_{k,h}(s_h^k)\\
%     \leq&\abs{\od_{k,h}(s_h^k)}+\abs{\od^\pik_{k,h}(s_h^k)}\\
%     \leq&C_1\Sp{\abs{\upd^\pik_{k,h}(s_h^k)}+\abs{\ud^\pik_{k,h}(s_h^k)}}+\abs{\od^\pik_{k,h}(s_h^k)}\\
%     \leq&C_1\Sp{\abs{\upd^\pik_{k,h}(s_h^k)}+\abs{\ud^\pik_{k,h}(s_h^k)}+\abs{\od^\pik_{k,h}(s_h^k)}},
% \end{align*}
% where the second inequality is due to Lemma \ref{lem:pessimism}
% \end{proof}

\label{subsec:decomp}

\section{Bounds on Individual Terms}
\label{sec:sumnoise}

\subsection{Bounds on Martingale Difference}
\label{subsec:martigalediff}
\begin{lemma}
\label{lem:martingale_bd}
For $i\in[H]$, the sequences starting from $0$ and with difference between two consecutive terms given by $\mathds{1}\{\G_k\}\M_{h,k}$ for $h = i,...,H$, $k = 1,...K$ are martingales with respect to filtration $\Bp{\oh_h^k}_{\substack{h=i,...,H, \\ k = 1,...,K}}$. Moreover, for any $\delta'>0$, with probability at least $1-\delta'$, for any $i\in[H]$, the following hold,
\begin{align*}
    \abs{\sum_{k=1}^K \sum_{h=i}^H \Sp{1+\frac{3C_1}{H}}^h\mathds{1}\{\G_k\cap\mathcal{E}^{cum}_{h,k}\}\M_{h,k}} & = \Tilde{O}\Sp{H\sqrt{T}}.
    %\abs{\sum_{k=1}^K \sum_{h=i}^H \mathds{1}\{\G_k\}\M_{{\od_{h,k}^\pik(s_h^k)}}} & = \tilde{O}\Sp{H\sqrt{T}}, \\
    %\abs{\sum_{k=1}^K \sum_{h=i}^H \mathds{1}\{\G_k\}\M_{{\ud_{h,k}^\pik(s_h^k)}}} &= \tilde{O}\Sp{H\sqrt{T}}, \\
    %\abs{\sum_{k=1}^K \sum_{h=i}^H \mathds{1}\{\G_k\}\M_{\delta^\pik_{h, k}(s_h^k)}} &= \tilde{O}\Sp{H\sqrt{T}}, \\
    %\abs{\sum_{k=1}^K \sum_{h=i}^H \mathds{1}\{\G_k\}\M_{|\overline{\delta}^\pik_{h, k}(s_h^k)|}} &= \tilde{O}\Sp{H\sqrt{T}}.
    % \abs{\sum_{k=1}^K \sum_{h=i}^H \mathds{1}\{\G_k\}\M_{|\overline{\delta}^\pik_{h, k}(s_h^k)|}} &= \tilde{O}\Sp{H\sqrt{T}}.
\end{align*}
\end{lemma}
\begin{proof}
We first show the sequence starting from $0$ and with difference between two consecutive terms given by $\mathds{1}\{\G_k\cap\mathcal{E}^{cum}_{h,k}\}\Sp{1+\frac{3C_1}{H}}^h\M_{|\overline{\delta}^\pik_{h, k}(s_h^k)|}$ is a martingale sequence.
	For any $h\in\{i,...,H\}$ and $k\in[K]$,
	\begin{align*}
	    &\E\Mp{\mathds{1}\{\G_k\cap\mathcal{E}^{cum}_{h,k}\} \M_{|\overline{\delta}^\pik_{h, k}(s_h^k)| } \mid {\oh_{h}^{k}}}\\
	    =&\E\Mp{ \mathds{1}\{\G_k\cap\mathcal{E}^{cum}_{h,k}\}\Sp{\inner{P_{\hsahk}, \abs{\od^\pik_{h+1, k}(s^k_{h+1})} }-\abs{\od^\pik_{h+1, k}(s_{h+1}^k)}} \mid {\oh_{h}^{k}}} =0.
	\end{align*}
	Similarly, we have $\mathds{1}\{\G_k\cap\mathcal{E}^{cum}_{h,k}\}\M_{|\upd^\pik_{h, k}(s_h^k)|}$, $\mathds{1}\{\G_k\cap\mathcal{E}^{cum}_{h,k}\}\M_{|\ud^\pik_{h, k}(s_h^k)|}$, $\mathds{1}\{\G_k\cap\mathcal{E}^{cum}_{h,k}\}\M_{|\od_{h, k}(s_h^k)|}$, $\mathds{1}\{\G_k\cap\mathcal{E}^{cum}_{h,k}\}\M_{|\upd_{h, k}(s_h^k)|}$, $\mathds{1}\{\G_k\cap\mathcal{E}^{cum}_{h,k}\}\M_{|\ud_{h, k}(s_h^k)|}$ are martingale difference sequences. As $\mathds{1}\{\G_k\cap\mathcal{E}^{cum}_{h,k}\}\M_{h, k}$ is the sum of several martingale difference sequences, it is a martingale difference sequence.
	
	Next, we bound $\abs{\mathds{1}\{\G_k\cap\mathcal{E}^{cum}_{h,k}\}\M_{|\overline{\delta}^\pik_{h, k}(s_h^k)|}}$. When $h = H$, $\abs{\mathds{1}\{\G_k\cap\mathcal{E}^{cum}_{h,k}\}\M_{|\overline{\delta}^\pik_{h, k}(s_h^k)|}}$ = 0. When ${\G_k}$ holds, for $h<H$ and any state $x$,
	\begin{align*}
	    \abs{\od^\pik_{h+1, k}(x)}= \abs{\ov_{h+1}(x) -V_{h+1}^\pi(x)}  & = \abs{\inner{P_{h+2,x,\pi(x)}, \ov_{h+2} -V_{h+2}^\pi} + w^k_\ty(h+1, x ,\pi(x))} \\
	    &\leq \inner{P_{h+2,x,\pi(x)}, \abs{\ov_{h+2} -V_{h+2}^\pi}} + \abs{w^k_\ty(h+1, x ,\pi(x))}
	\end{align*}
	By our choice of $\alpha_k$, when $\G_k$ holds, $\abs{w^k_\ty(h,s,a)} \leq \gamma^k_\ty(h,s,a)\leq 1$ for all $k, h,s,a$ as shown in Lemma \ref{lemma:q_bar_bound}. Then, by expanding $\abs{\od^\pik_{h+1, k}(x)}= \abs{\ov_{h+1}(x) -V_{h+1}^\pi(x)}$ recursively from $h+1$ to $H$, we have
	\begin{align*}
	    \abs{\mathds{1}\{\G_k\cap\mathcal{E}^{cum}_{h,k}\}\M_{|\overline{\delta}^\pik_{h, k}(s_h^k)|}} \leq 2H\gamma^k_\ty(h,s,a) \leq 2H.
	\end{align*}
	Similarly, we have the bound on $\mathds{1}\{\G_k\cap\mathcal{E}^{cum}_{h,k}\}\M_{|\upd^\pik_{h, k}(s_h^k)|}$, $\mathds{1}\{\G_k\cap\mathcal{E}^{cum}_{h,k}\}\M_{|\ud^\pik_{h, k}(s_h^k)|}$, $\mathds{1}\{\G_k\cap\mathcal{E}^{cum}_{h,k}\}\M_{|\od_{h, k}(s_h^k)|}$, $\mathds{1}\{\G_k\cap\mathcal{E}^{cum}_{h,k}\}\M_{|\upd_{h, k}(s_h^k)|}$ and $\mathds{1}\{\G_k\cap\mathcal{E}^{cum}_{h,k}\}\M_{|\ud_{h, k}(s_h^k)|}$.
	
	As a result, $\abs{\mathds{1}\{\G_k\cap\mathcal{E}^{cum}_{h,k}\}\Sp{1+\frac{3C_1}{H}}^h\M_{h, k}}$ is bounded by $12e^{3C_1}H$.
	By Azuma-Hoeffding inequality, with probability at least $1-\delta'$, we have
    \begin{align*}
    \abs{\sum_{k=1}^K \sum_{h=i}^H \Sp{1+\frac{3C_1}{H}}^h\mathds{1}\{\G_k\cap\mathcal{E}^{cum}_{h,k}\}\M_{h,k}} \leq\widetilde{O}\Sp{\sqrt{\sum_{k=1}^K \sum_{h=i}^H H^2}}=\widetilde{O}\Sp{H\sqrt{T}}. 
\end{align*}
\end{proof}

\subsection{Bounds on Lower-order Terms}
The following two lemmas are standard results in literature and we present their proofs here for completeness.
\begin{lemma}
	\label{lemma:sumsqrt1/n}
	$$\sum_{k=1}^{K}\sum_{h=1}^{H}\sqrt{\frac{\log\Sp{2HSAk^2}}{n_k\Sp{\hsahk}+1}}\leq\widetilde{O}\Sp{\sqrt{HSAT}}.$$
\end{lemma}
\begin{proof}
	Let $L=\log\Sp{2HSAK^2}$. Then, it can be bounded as
	\begin{align*}
		\sum_{k=1}^{K}\sum_{h=1}^{H}\sqrt{\frac{\log\Sp{2HSAk^2}}{n_k\Sp{\hsahk}+1}}&\leq\sqrt{L}\sum_{k=1}^{K}\sum_{h=1}^{H}\sqrt{\frac{1}{n_k(\hsahk)+1}}\\
		&=\sqrt{L}\sum_{h, s, a}\sum_{n=1}^{n_K(h, s, a)}\sqrt{\frac{1}{n+1}}\\
		&\leq\sqrt{L}\sum_{h, s, a}\int_{0}^{n_K\hsa}\sqrt{\frac{1}{x}}\mathrm{d}x\\
		&\leq 2\sqrt{L}\sum_{h, s, a}\sqrt{n_K\hsa}\\
		&\leq 2\sqrt{L}\cdot\sqrt{HSA\sum_{h, s, a}n_K\hsa}\tag{By Cauchy-Schwartz inequality}\\
		&=\widetilde{O}\Sp{\sqrt{HSAT}}.\tag{Since $\sum_{h, s, a}n_K\hsa=T$}
	\end{align*}
\end{proof}
\begin{lemma}
	\label{lemma:sum1/n}
	$$\sum_{k=1}^{K}\sum_{h=1}^{H}\frac{\log\Sp{2HSAk^2}}{n_k\Sp{\hsahk}+1}\leq\widetilde{O}\Sp{HSA}.$$
\end{lemma}
\begin{proof}
	Let $L=\log\Sp{2HSAK^2}$. Then, it can be bounded as
	\begin{align*}
		\sum_{k=1}^{K}\sum_{h=1}^{H}\frac{\log\Sp{2HSAk^2}}{n_k\Sp{\hsahk}+1}&\leq L\sum_{k=1}^{K}\sum_{h=1}^{H}\frac{1}{n_k(\hsahk)+1}\\
		&= L\sum_{h, s, a}\sum_{n=1}^{n_K(h, s, a)}\frac{1}{n+1}\\
		&\leq L\sum_{h, s, a}\log\Sp{n_K(h, s, a)}\tag{Since $\sum_{n=1}^{N}\frac{1}{n}\leq\log(N)+1$}\\
		&\leq LHSA\cdot\max_{h, s, a}\log(n_K(h, s, a))\\
		&\leq LHSA\log(T)\\
		&=\widetilde{O}\Sp{HSA}.
	\end{align*}
\end{proof}

\section{Bounds on Sum of Variance}
\label{sec:sum_variance}
When we use the Bernstein-type noise, the regret analysis needs to bound the sum of variance. This proof applies some techniques developed in \cite{azar2017minimax}. However, since our optimism only holds with constant probability instead of deterministically, the details are quite different. For simplicity, we first define
\begin{align*}
	\hat{\mathbb{V}}^*_{h+1, k}&=\mathbb{V}\Sp{\tilde{P}^k_{h, s_h^k, a_h^k}, V^*_{h+1}},\quad \mathbb{V}^*_{h+1, k}=\mathbb{V}\Sp{P_{h, s_h^k, a_h^k}, V^*_{h+1}},\\
	\hat{\overline{\mathbb{V}}}_{h+1, k}&=\mathbb{V}\Sp{\tilde{P}^k_{h, s_h^k, a_h^k}, \overline{V}_{h+1, k}},\quad \overline{\mathbb{V}}_{h+1, k}=\mathbb{V}\Sp{P_{h, s_h^k, a_h^k}, \overline{V}_{h+1, k}},\\
% 	\hat{\underline{\bbV}}_{h+1, k}&=\bbV\Sp{\tilde{P}^k_{\hsahk}, \underline{V}_{h+1, k}},\quad\underline{\bbV}_{h+1, k}=\bbV\Sp{P_{\hsahk}, \underline{V}_{h+1, k}},\\
	\mathbb{V}^{\pi^k}_{h+1, k}&=\mathbb{V}\Sp{P_{h, s_h^k, a_h^k}, V^{\pi^k}_{h+1, k}},\\
	U_{h, k, 1}&=\sqrt{\frac{\hat{\bbV}^*_{h+1, k}\log\Sp{2HSAK^2}}{n_k(\hsahk)+1}},\quad U_{h, k, 2}=\sqrt{\frac{\hat{\overline{\bbV}}_{h+1, k}\log\Sp{2HSAK^2}}{n_k(\hsahk)+1}},
% 	U_{h, k, 3}&=\sqrt{\frac{\bbV\Sp{\tilde{P}^k_\hsahk, \overline{V}_{h+1, k}-V^*_{h+1}}\log\Sp{2HSAK/\delta}}{n_k(\hsahk)+1}},\\
% 	U_{h, k, 4}&=\sqrt{\frac{\bbV\Sp{\tilde{P}^k_\hsahk, \underline{V}_{h+1, k}-V^*_{h+1}}\log\Sp{2HSAK/\delta}}{n_k(\hsahk)+1}}.
\end{align*}
We will first give a full proof of the bound on sum of variance and then present all the auxiliary lemmas in Section \ref{sec:aux_lemma}.
\begin{lemma}
	\label{theo:variance_sum}
	Let $U_{h, k}=U_{h, k, 1}+U_{h, k, 2}$. For any $\delta > 0$, with probability at least $1-\delta$, when $T\geq\Omega\Sp{H^5S^2A}$, it holds that
	$$\sum_{k=1}^{K}\sum_{h=1}^{H-1}\ogk U_{h, k}\leq\widetilde{O}\Sp{H\sqrt{SAT}}.$$
\end{lemma}
\begin{proof}
% 	By Lemma \ref{lemma:var_sum}, we have
% 	\begin{align*}
% 		\sqrt{\bbV\Sp{\tilde{P}^k_{\hsahk}, \overline{V}_{h+1, k}-V^*_{h+1}}}&\leq \sqrt{2\Sp{\hat{\overline{\bbV}}_{h+1, k}+\hat{\bbV}^*_{h+1, k}}}\leq 2\Sp{\sqrt{\hat{\overline{\bbV}}_{h+1, k}}+\sqrt{\hat{\bbV}^*_{h+1, k}}},\\
% 		\sqrt{\bbV\Sp{\tilde{P}^k_{\hsahk}, \underline{V}_{h+1, k}-V^*_{h+1}}}&\leq \sqrt{2\Sp{\hat{\underline{\bbV}}_{h+1, k}+\hat{\bbV}^*_{h+1, k}}}\leq 2\Sp{\sqrt{\hat{\underline{\bbV}}_{h+1, k}}+\sqrt{\hat{\bbV}^*_{h+1, k}}}.
% 	\end{align*}
	First, we have
	\fontsize{9}{9}
	\begin{align}
		\sum_{k=1}^{K}\sum_{h=1}^{H-1}\ogk U_{h, k}&\leq\sum_{k=1}^{K}\sum_{h=1}^{H-1}\ogk\sqrt{\frac{\log\Sp{2HSAK^2}}{n_k(\hsahk)+1}}\Sp{\sqrt{\hat{\bbV}^*_{h+1, k}}+\sqrt{\hat{\overline{\bbV}}_{h+1, k}}}\nonumber\\
		&\leq\sum_{k=1}^{K}\sum_{h=1}^{H-1}\ogk\sqrt{\frac{\log\Sp{2HSAK^2}}{n_k(\hsahk)+1}}\cdot \sqrt{2}\sqrt{\hat{\bbV}^*_{h+1, k}+\hat{\overline{\bbV}}_{h+1, k}}\tag{Since $\sqrt{a}+\sqrt{b}\leq\sqrt{2(a+b)}$ for $a, b\geq 0$}\\
		&\leq\sqrt{2}\sqrt{\Sp{\sum_{k=1}^{K}\sum_{h=1}^{H-1}\frac{\log\Sp{2HSAK^2}}{n_k(\hsahk)+1}}\Sp{\sum_{k=1}^{K}\sum_{h=1}^{H-1}\ogk\Sp{\hat{\bbV}^*_{h+1, k}+\hat{\overline{\bbV}}_{h+1, k}}}}\tag{By Cauchy-Schwartz inequality}\\
		&\leq\sqrt{\widetilde{O}\Sp{HSA}\Sp{\sum_{k=1}^{K}\sum_{h=1}^{H-1}\ogk\Sp{\hat{\bbV}^*_{h+1, k}+\hat{\overline{\bbV}}_{h+1, k}}}},\label{equ:sum_U}
	\end{align}
	\normalsize
	where the last inequality above applies Lemma \ref{lemma:sum1/n}.
	
	We will then bound the two sums of variance separately. Specifically, by applying Lemma \ref{lemma:sum_vpi} and Lemma \ref{lemma:sumvstar_vpi}, we have with probability at least $1-\delta/3$,
	\begin{align}
		&\sum_{k=1}^{K}\sum_{h=1}^{H-1}\ogk\hat{\bbV}^*_{h+1, k}\\=&\frac{3}{2}\sum_{k=1}^{K}\sum_{h=1}^{H-1}\ogk\bbV^{\pik}_{h+1, k}+\sum_{k=1}^{K}\sum_{h=1}^{H-1}\ogk\Sp{\hat{\bbV}^*_{h+1, k}-\frac{3}{2}\bbV^{\pik}_{h+1, k}}\nonumber\\
		\leq&\widetilde{O}\Sp{HT+H^2\sqrt{T}+H^3+H^3S^2A+H\sum_{k=1}^{K}\sum_{h=1}^{H-1}\ogk\delta^{\pi^k}_{h+1, k}(s_{h+1}^k)}.\label{equ:sum_vhatstar}
	\end{align}
	By similarly applying Lemma \ref{lemma:sum_vpi} and Lemma \ref{lemma:sumvbar_vpi}, we have with probability at least $1-\delta/3$,
	\begin{align}
		&\sumogk\hat{\overline{\bbV}}_{h+1, k}\\=&\frac{3}{2}\sumogk\bbV^\pik_{h+1, k}+\sumogk\Sp{\hat{\overline{\bbV}}_{h+1, k}-\frac{3}{2}\bbV^\pik_{h+1, k}}\nonumber\\
		\leq&\widetilde{O}\Sp{HT+H^2\sqrt{T}+H^3+H^3S^2A+H\sum_{k=1}^{K}\sum_{h=1}^{H-1}\ogk\abs{\overline{\delta}^{\pi^k}_{h+1, k}(s_{h+1}^k)}},\label{equ:sum_vhatbar}
	\end{align}
% 	\begin{align}
% 		\sumogk\hat{\underline{\bbV}}_{h+1, k}&=\frac{3}{2}\sumogk\bbV^\pik_{h+1, k}+\sumogk\Sp{\hat{\underline{\bbV}}_{h+1, k}-\frac{3}{2}\bbV^\pik_{h+1, k}}\nonumber\\
% 		&\leq\widetilde{O}\Sp{HT+H^2\sqrt{T}+H^3+H^3S^2A+H\sum_{k=1}^{K}\sum_{h=1}^{H-1}\ogk\abs{\underline{\delta}^{\pi^k}_{h+1, k}(s_{h+1}^k)}}.\label{equ:sum_vhatline}
% 	\end{align}
	By combining equations \eqref{equ:sum_vhatstar} and \eqref{equ:sum_vhatbar}, we have
	\begin{align}
		&\sum_{k=1}^{K}\sum_{h=1}^{H-1}\ogk\Sp{\hat{\bbV}^*_{h+1, k}+\hat{\overline{\bbV}}_{h+1, k}}\nonumber\\
		\leq&\widetilde{O}\Sp{HT+H^2\sqrt{T}+H^3S^2A+H\sumogk\Sp{\delta^\pik_{h+1, k}(s_{h+1}^k)+\abs{\od^\pik_{h+1, k}(s_{h+1}^k)}}}.\label{equ:sum_deltas}
	\end{align}

	Then, by referring to definitions of $\sqrt{e^k_\Be(\hsahk)}$ and $\gamma^k_\Be(\hsahk)$, with probability at least $1-\delta/3$, we have
	\begin{align}
	    &\sumogk\Sp{\delta^\pik_{h+1, k}(s_{h+1}^k)+\abs{\od^\pik_{h+1, k}(s_{h+1}^k)}}\nonumber\\
	    \leq&\sum_{h=1}^{H-1}\Sp{\sum_{k=1}^{K}\ogk \Sp{\abs{\delta^\pik_{h+1, k}(s_{h+1}^k)}+\abs{\od^\pik_{h+1, k}(s_{h+1}^k)}}}\nonumber\\
	    \leq & \ogk 3C_1e^{3C_1}H\sum_{k=1}^{K}\sum_{h=1}^H \Sp{\sqrt{e_\ty^k(\hsahk)}+ \gamma_\ty^k(\hsahk)}+\widetilde{O}\Sp{H^5S^2A+H^2\sqrt{T}}\tag{By referring to the proof of Lemma \ref{lem:reg decomp}}\\
	    \leq &\widetilde{O}\Sp{H^5S^2A+H^2\sqrt{T}+\sqrt{H^3SAT}+H\sumogk U_{h, k}}\tag{By Lemma \ref{lemma:sumsqrt1/n} and \ref{lemma:sum1/n}}\\
	    \leq &\widetilde{O}\Sp{H^5S^2A+H^2\sqrt{SAT}+H\sumogk U_{h, k}}.\label{equ:sum_odeltas}
	\end{align}
% 	\begin{align}
% 		&\sumogk\abs{\od^\pik_{h+1, k}(s_{h+1}^k)}\nonumber\\
% 		=&\sum_{h=1}^{H-1}\Sp{\sum_{k=1}^{K}\ogk\abs{\od^\pik_{h+1, k}(s_{h+1}^k)}}\nonumber\\
% % 		\leq& H\sum_{k=1}^{K}\ogk\abs{\od^\pik_{2, k}(s_{2}^k)}\nonumber\\
% 		\leq&3e^{3C_1}H\sum_{k=1}^{K}\sum_{h=2}^{H}\ogk\Sp{\sqrt{e^k_\Be(\hsahk)}+|\gamma^k_\Be(\hsahk)|}+\widetilde{O}\Sp{\sqrt{H^3T}+H^4S^2A}\nonumber\\
% 		\leq&\widetilde{O}\Sp{\sqrt{H^3T}+H^4S^2A+H\Sp{\sum_{k=1}^{K}\sum_{h=2}^{H-1}\ogk U_{h, k, 1}+U_{k, h, 2}}}\nonumber\\
% 		\leq&\widetilde{O}\Sp{\sqrt{H^3T}+H^4S^2A+H\sumogk U_{h, k}}.\label{equ:sum_odeltas}
% 	\end{align}
% 	By further applying Lemma \ref{lem:reg decomp}, we can similarly obtain with probability at least $1-\delta/4$,
% 	\begin{align}
% 		\sumogk\abs{\delta^\pik_{h+1, k}(s_{h+1}^k)}&\leq \widetilde{O}\Sp{\sqrt{H^3T}+H^4S^2A+H\sumogk U_{h, k}},\label{equ:sum_udeltas}
% % 		\sumogk\delta^\pik_{h+1, k}(s_{h+1}^k)&\leq \widetilde{O}\Sp{H^3SA+H^{1.5}\sqrt{SAT}+H^2\sqrt{T}+H\sumogk U_{h, k}}.\label{equ:sum_deltapis}
% 	\end{align}
	By plugging equation \eqref{equ:sum_odeltas} into equation \eqref{equ:sum_deltas}, we can have
	\begin{align}
		&\sum_{k=1}^{K}\sum_{h=1}^{H-1}\ogk\Sp{\hat{\bbV}^*_{h+1, k}+\hat{\overline{\bbV}}_{h+1, k}}\nonumber\\
		\leq&\widetilde{O}\Sp{HT+H^2\sqrt{T}+H^3S^2A+H^6S^2A+H^3\sqrt{SAT}+H^2\sumogk U_{h, k}}\nonumber\\
		\leq&\widetilde{O}\Sp{HT+H^3\sqrt{SAT}+H^6S^2A+H^2\sumogk U_{h, k}}\nonumber\\
		\leq&\widetilde{O}\Sp{HT+H^2\sumogk U_{h, k}}\tag{When $T\geq\Omega\Sp{H^5S^2A}$}
	\end{align}
	
	Now, by plugging the above result into equation \eqref{equ:sum_U}, when $T\geq\Omega\Sp{H^5S^2A}$, it holds that
	\begin{align*}
	    \sum_{k=1}^{K}\sum_{h=1}^{H-1}\ogk U_{h, k}&\leq\sqrt{\widetilde{O}\Sp{HSA\Sp{HT+H^2\sumogk U_{h, k}}}}\\
	    &\leq\widetilde{O}\Sp{H\sqrt{SAT}+H^{1.5}\sqrt{\sumogk U_{h, k}}}.
	\end{align*}
	It is easy to check that the above inequality implies $\sumogk U_{h, k}\leq\widetilde{O}\Sp{H\sqrt{SAT}}$ and thus the proof is complete.
\end{proof}

\subsection{Auxiliary Lemmas}
\label{sec:aux_lemma}

The lemmas used for proving Lemma \ref{theo:variance_sum} are presented as the following.
% \begin{lemma}
% 	\label{lemma:var_sum}
% 	For aribtrary random variables $X$ and $Y$ with finite variance, we have $\mathrm{Var}\Sp{X+Y}\leq 2\mathrm{Var}\Sp{X}+2\mathrm{Var}\Sp{Y}$.
% \end{lemma}
% \begin{proof}
% 	The following sequence of computation holds.
% 	\begin{align*}
% 		\mathrm{Var}\Sp{X+Y}&=\mathrm{Var}\Sp{X}+\mathrm{Cov}\Sp{X, Y}+\mathrm{Var}\Sp{Y}\\
% 		&\leq\mathrm{Var}\Sp{X}+\sqrt{\mathrm{Var}\Sp{X}\mathrm{Var}\Sp{Y}}+\mathrm{Var}\Sp{Y}\tag{By Cauchy-Schwartz inequality}\\
% 		&=\Sp{\sqrt{\mathrm{Var}\Sp{X}}+\sqrt{\mathrm{Var}\Sp{Y}}}^2\\
% 		&\leq 2\mathrm{Var}\Sp{X}+2\mathrm{Var}\Sp{Y}.
% 	\end{align*}
% \end{proof}

\begin{lemma}[Lemma 8 in \cite{azar2017minimax}]
	\label{lemma:sum_vpi}
	For any $\delta > 0$, with probability at least $1-\delta$, it holds that
	$$\sum_{k=1}^{K}\sum_{h=1}^{H-1}\mathds{1}\Sp{\mathcal{G}_k}\mathbb{V}^{\pi^k}_{h+1, k}\leq\Tilde{O}\Sp{HT+H^2\sqrt{T}+H^3}.$$
\end{lemma}
\begin{lemma}
	\label{lemma:vhat_vstar}
	For any $\delta > 0$, with probability at least $1-\delta$, for any $k\in\Mp{K}$, $h\in\Mp{H}$, it holds that
	$$\hat{\mathbb{V}}^*_{h+1, k}\leq\frac{3}{2}\mathbb{V}^*_{h+1, k}+\frac{2H^2S\log\Sp{2HS^2AK/\delta}}{n_k\Sp{h, s_h^k, a_h^k}},$$
	$$\hat{\overline{\bbV}}_{h+1, k}\leq\frac{3}{2}\overline{\bbV}_{h+1, k}+\frac{2H^2S\log\Sp{2HS^2AK/\delta}}{n_k\Sp{h, s_h^k, a_h^k}}.$$
% 	$$\hat{\underline{\bbV}}_{h+1, k}\leq\frac{3}{2}\underline{\bbV}_{h+1, k}+\frac{2H^2S\log\Sp{2HS^2AK/\delta}}{n_k\Sp{h, s_h^k, a_h^k}}.$$
\end{lemma}
\begin{proof}
	The proof apply some techniques in \cite{zhang2020reinforcement}. Fix some $\delta>0$ and let $L=\log\Sp{2HS^2AK/\delta}$ for simplicity. First, by Lemma \ref{lemma:bennet}, for some tuple $\Sp{k, h, s, a, s'}$, we have
	\begin{align*}
		&\P\Sp{\tilde{P}^k_{h, s, a}(s')\geq\frac{3}{2}P_{h, s, a}(s')+\frac{2L}{n_k\Sp{h, s, a}}}\\
		\leq&\P\Sp{\tilde{P}^k_{h, s, a}(s')-P_{h, s, a}(s')\geq\sqrt{\frac{2P_{h, s, a}(s')L}{n_k\Sp{h, s, a}}}+\frac{L}{n_k\Sp{h, s, a}}}\tag{Since $a+b\geq 2\sqrt{ab}$ for $a, b\geq 0$}\\
		\leq&\frac{\delta}{HS^2AK}.
	\end{align*}
	Then, a union bound says that its complement holds for any $\Sp{k, h, s, a, s'}$ with probability at least $1-\delta$. Thus, we have
	\begin{align*}
		\hat{\bbV}^*_{h+1, k}&=\sum_{s'\in\S}\tilde{P}^k_{h, s_h^k, a_h^k}(s')\Sp{V^*_{h+1}(s')-\inner{\tilde{P}^k_{h, s_h^k, a_h^k}, V^*_{h+1}}}^2\\
		&\leq \sum_{s'\in\S}\tilde{P}^k_{h, s_h^k, a_h^k}(s')\Sp{V^*_{h+1}(s')-\inner{P_{h, s_h^k, a_h^k}, V^*_{h+1}}}^2\tag{Since $\E\Mp{X}$ is the minimizer of $\min_x\E\Mp{\Sp{X-x}^2}$}\\
		&\leq \sum_{s'\in\S}\Sp{\frac{3}{2}P_{h, s, a}(s')+\frac{2L}{n_k(h, s, a)}}\Sp{V^*_{h+1}(s')-\inner{P_{h, s_h^k, a_h^k}, V^*_{h+1}}}^2\\
		&\leq\frac{3}{2}\bbV^*_{h+1, k}+\frac{2H^2S\log\Sp{2HS^2AK/\delta}}{n_k(h, s_h^k, a_h^k)}.
	\end{align*}
	For $\hat{\overline{\bbV}}_{h+1, k}$, we just need to follow a similar argument and thus the proof is complete.
\end{proof}
\begin{lemma}
	\label{lemma:sumvstar_vpi}
	For any $\delta > 0$, with probability at least $1-\delta$, it holds that 
	\fontsize{9}{9}
	$$\sum_{k=1}^{K}\sum_{h=1}^{H-1}\mathds{1}\Bp{\mathcal{G}_k}\Sp{\hat{\bbV}^*_{h+1, k}-\frac{3}{2}\bbV^{\pi^k}_{h+1, k}}\leq\widetilde{O}\Sp{H\sum_{k=1}^{K}\sum_{h=1}^{H-1}\mathds{1}\Bp{\mathcal{G}_k}\delta^{\pi^k}_{h+1, k}(s_{h+1}^k)+H^2\sqrt{T}+H^3S^2A}.$$
	\normalsize
\end{lemma}
\begin{proof}
	We begin by applying Lemma \ref{lemma:vhat_vstar}. Thus, with probability at least $1-\frac{\delta}{2}$, we have
	\begin{align*}
		&\sum_{k=1}^{K}\sum_{h=1}^{H-1}\mathds{1}\Bp{\mathcal{G}_k}\Sp{\hat{\bbV}^*_{h+1, k}-\frac{3}{2}\bbV^{\pi^k}_{h+1, k}}\\
		\leq&\sum_{k=1}^{K}\sum_{h=1}^{H-1}\ogk\Sp{\frac{3}{2}\bbV^*_{h+1, k}-\frac{3}{2}\bbV^{\pi^k}_{h+1, k}}+\frac{4H^2S\log\Sp{4HS^2AK/\delta}}{3n_k(h, s_h^k, a_h^k)}\tag{By Lemma \ref{lemma:vhat_vstar}}\\
		\leq&\sum_{k=1}^{K}\sum_{h=1}^{H-1}\ogk\Sp{\frac{3}{2}\bbV^*_{h+1, k}-\frac{3}{2}\bbV^{\pi^k}_{h+1, k}}+\widetilde{O}\Sp{H^3S^2A}\tag{By Lemma \ref{lemma:sum1/n}}\\
		\leq&\sum_{k=1}^{K}\sum_{h=1}^{H-1}\ogk\frac{3}{2}\E_{s'\sim P_{h, s_h^k, a_h^k}}\Mp{\Sp{V^*_{h+1}(s')}^2-\Sp{V^{\pi^k}_{h+1, k}(s')}^2}+\widetilde{O}\Sp{H^3S^2A}\tag{Since $V^{\pi^k}_{h+1, k}\leq V^*_{h+1}$}\\
		\leq&3H\sum_{k=1}^{K}\sum_{h=1}^{H-1}\ogk\E_{s'\sim P_{h, s_h^k, a_h^k}}\Mp{\delta^{\pi^k}_{h+1, k}(s')}+\widetilde{O}\Sp{H^3S^2A}\tag{Since $V^{\pi^k}_{h+1, k}\leq V^*_{h+1}\leq H$ and $a^2-b^2=(a+b)(a-b)$}\\
		\leq&\widetilde{O}\Sp{H\sum_{k=1}^{K}\sum_{h=1}^{H-1}\mathds{1}\Bp{\mathcal{G}_k}\delta^{\pi^k}_{h+1, k}(s_{h+1}^k)+H^2\sqrt{T}+H^3S^2A}. \tag{By Lemma \ref{lem:martingale_bd}}
	\end{align*}
	The last line above holds because by Lemma \ref{lem:martingale_bd}, with probability at least $1-\frac{\delta}{2}$, we have
	$$\abs{\sum_{k=1}^{K}\sum_{h=1}^{H-1}\ogk\Sp{\E_{s'\sim P_{h, s_h^k, a_h^k}}\Mp{\delta^{\pi^k}_{h+1, k}(s')}-\delta^{\pi^k}_{h+1, k}(s_{h+1}^k)}}\leq\widetilde{O}\Sp{H\sqrt{T}}.$$
\end{proof}
\begin{lemma}
	\label{lemma:sumvbar_vpi}
	For any $\delta > 0$, w ith probability at least $1-\delta$, it hold that
	\fontsize{9}{9}
	\begin{align*}
		\sum_{k=1}^{K}\sum_{h=1}^{H-1}\ogk\Sp{\hat{\overline{\bbV}}_{h+1, k}-\frac{3}{2}\bbV^{\pi^k}_{h+1, k}}&\leq \widetilde{O}\Sp{H\sum_{k=1}^{K}\sum_{h=1}^{H-1}\ogk\abs{\overline{\delta}^{\pi^k}_{h+1, k}(s_{h+1}^k)}+H^2\sqrt{T}+H^3S^2A}.
% 		\sum_{k=1}^{K}\sum_{h=1}^{H-1}\ogk\Sp{\hat{\underline{\bbV}}_{h+1, k}-\frac{3}{2}\bbV^{\pi^k}_{h+1, k}}&\leq\widetilde{O}\Sp{H\sum_{k=1}^{K}\sum_{h=1}^{H-1}\ogk\abs{\underline{\delta}^{\pi^k}_{h+1, k}(s_{h+1}^k)}+H^2\sqrt{T}+H^3S^2A}.
	\end{align*}
	\normalsize
\end{lemma}
\begin{proof}
	Similarly, we begin by applying Lemma \ref{lemma:vhat_vstar} and with probability at least $1-\frac{\delta}{3}$, we have
	\begin{align*}
		&\sum_{k=1}^{K}\sum_{h=1}^{H-1}\mathds{1}\Bp{\mathcal{G}_k}\Sp{\hat{\overline{\bbV}}_{h+1, k}-\frac{3}{2}\bbV^{\pi^k}_{h+1, k}}\\
		\leq&\sum_{k=1}^{K}\sum_{h=1}^{H-1}\ogk\Sp{\frac{3}{2}\overline{\bbV}_{h+1, k}-\frac{3}{2}\bbV^{\pi^k}_{h+1, k}}+\frac{4H^2S\log\Sp{6HS^2AK/\delta}}{3n_k(h, s_h^k, a_h^k)}\\
		\leq&\sum_{k=1}^{K}\sum_{h=1}^{H-1}\ogk\Sp{\frac{3}{2}\overline{\bbV}_{h+1, k}-\frac{3}{2}\bbV^{\pi^k}_{h+1, k}}+\widetilde{O}\Sp{H^3S^2A}\tag{By Lemma \ref{lemma:sum1/n}}\\
		=&\frac{3}{2}\underbrace{\sum_{k=1}^{K}\sum_{h=1}^{H-1}\ogk\Sp{\inner{P_{h, s_h^k, a_h^k}, \Sp{\overline{V}_{h+1, k}}^2}-\inner{P_{h, s_h^k, a_h^k}, \Sp{V^{\pi^k}_{h+1, k}}^2}}}_{\text{(a)}}\\
		&\qquad+\frac{3}{2}\underbrace{\sum_{k=1}^{K}\sum_{h=1}^{H-1}\ogk\Sp{\inner{P_{h, s_h^k, a_h^k}, V^{\pi^k}_{h+1, k}}^2-\inner{P_{h, s_h^k, a_h^k}, \overline{V}_{h+1, k}}^2}}_{\text{(b)}}+\widetilde{O}\Sp{H^3S^2A}.\tag{By definition of variance}
	\end{align*}
	We will bound (a) and (b) separately. For term (a), with probability at least $1-\frac{\delta}{3}$, we have
	\begin{align*}
		\text{(a)}&=\sum_{k=1}^{K}\sum_{h=1}^{H-1}\ogk\inner{P_\hsahk, \Sp{\overline{V}_{h+1, k}}^2-\Sp{V^{\pi^k}_{h+1, k}}^2}\\
		&\leq \sumogk\inner{P_\hsahk, \abs{\ov_{h+1, k}-V^\pik_{h+1, k}}\abs{\ov_{h+1, k}+V^\pik_{h+1, k}}}\tag{Since $a^2-b^2=(a+b)(a-b)$}\\
		&\leq 3H\sumogk\inner{P_\hsahk, \abs{\ov_{h+1, k}-V^\pik_{h+1, k}}}\tag{Since $\Norm{\ov_{h+1, k}}_\infty\leq 2H$ under $\mathcal{G}_k$}\\
		&\leq 3H\sum_{k=1}^{K}\sum_{h=1}^{H-1}\ogk\abs{\overline{\delta}^{\pi^k}_{h+1, k}(s_{h+1}^k)}+\widetilde{O}\Sp{H^2\sqrt{T}}.\tag{By Lemma \ref{lem:martingale_bd}}
	\end{align*}
	For term (b), with probability at least $1-\frac{\delta}{3}$, we have
	\begin{align*}
		\text{(b)}&=\sum_{k=1}^{K}\sum_{h=1}^{H-1}\ogk\inner{P_{h, s_h^k, a_h^k}, V^{\pi^k}_{h+1, k}+\overline{V}_{h+1, k}}\inner{P_{h, s_h^k, a_h^k}, V^{\pi^k}_{h+1, k}-\overline{V}_{h+1, k}}\\
		&\leq 3H\sum_{k=1}^{K}\sum_{h=1}^{H-1}\ogk\inner{P_{h, s_h^k, a_h^k}, \abs{V^{\pi^k}_{h+1, k}-\overline{V}_{h+1, k}}}\\
		&\leq 3H\sum_{k=1}^{K}\sum_{h=1}^{H-1}\ogk\abs{\overline{\delta}^{\pi^k}_{h+1, k}(s_{h+1}^k)}+\widetilde{O}\Sp{H^2\sqrt{T}}.\tag{By Lemma \ref{lem:martingale_bd}}
	\end{align*}
	Therefore, in summary, we have with probability at least $1-\delta/2$,
	\fontsize{9}{9}
	\begin{align*}
		\sum_{k=1}^{K}\sum_{h=1}^{H-1}\ogk\Sp{\hat{\overline{\bbV}}_{h+1, k}-\frac{3}{2}\bbV^{\pi^k}_{h+1, k}}&\leq \widetilde{O}\Sp{H\sum_{k=1}^{K}\sum_{h=1}^{H-1}\ogk\abs{\overline{\delta}^{\pi^k}_{h+1, k}(s_{h+1}^k)}+H^2\sqrt{T}+H^3S^2A}.
	\end{align*}
	\normalsize
% 	For $\sum_{k=1}^{K}\sum_{h=1}^{H-1}\ogk\Sp{\hat{\underline{\bbV}}_{h+1, k}-\frac{3}{2}\bbV^{\pi^k}_{h+1, k}}$, we just need to follow a very similar argument and thus the proof is complete.
\end{proof}

\section{Proof of the Main Theorems}
\label{sec:proofmain}

In this section, we state and prove our two main theorems.
\begin{theorem}
	\label{theo:appen_theo1}
	If the Hoeffding-type noise is used, then for any MDP $M=\Sp{H, \S, \A, P, R, s_1}$, for any $\delta > 0$, with probability at least $1-\delta$, Algorithm \ref{algo:cub-rlsvi} satisfies
	$$\mathrm{Reg}(M, K, \mathsf{SSR}_\Ho)\leq\widetilde{O}\Sp{H^{1.5}\sqrt{SAT}+H^4S^2A}.$$
	In particular, when $T\geq\widetilde{\Omega}\Sp{H^5S^3A}$, it holds that $\mathrm{Reg}(M, K, \mathsf{SSR}_\Ho)\leq\widetilde{O}\Sp{H^{1.5}\sqrt{SAT}}$.
\end{theorem}
\begin{proof}
	By using the result of Lemma \ref{lem:reg decomp}, under Hoeffding-type noise, with probability at least $1-\delta$, we have
	\begin{align*}
	    &\mathrm{Reg}\Sp{M, K, \mathsf{SSR}_\Ho}\\
	   %& =\sum_{k=1}^K\ogk\delta^\pik_{1,k}(s_1^k)+\widetilde{O}(H^5S^2A)\\
	   % &\leq \sum_{k=1}^K C_1\ogk\Sp{\abs{\od_{h,k}^\pik(s_h^k)}+\abs{\upd_{h,k}^\pik(s_h^k)}+\abs{\ud_{h,k}^\pik(s_h^k)}}+\widetilde{O}(H^5S^2A)\\
	    \leq& \ogk 3C_1e^{3C_1}\sum_{k=1}^{K}\sum_{h=i}^H \Sp{\sqrt{e_\Ho^k(\hsahk)}+ \gamma_\Ho^k(\hsahk)}+\widetilde{O}\Sp{H^4S^2A+H\sqrt{T}}\\
	    \leq& 6C_1e^{3C_1}\sum_{k=1}^{K}\sum_{h=1}^{H-1}\Sp{H\sqrt{\frac{\log(2HSAk^2)}{n_k(h,s,a)+1}}+\frac{H}{n_k(h,s,a)+1}}+\widetilde{O}\Sp{H^4S^2A+H\sqrt{T}}\\
	    =&\widetilde{O}(H^{1.5}\sqrt{SAT}+H^4S^2A).
	\end{align*}
	Here, the second inequality is from the definitions of $\sqrt{e_\Ho^k(\hsahk)}$ and $\gamma_\Ho^k(\hsahk)$, and the last step is from Lemma \ref{lemma:sumsqrt1/n} and \ref{lemma:sum1/n}.

	\begin{comment}
	\begin{align*}
		\mathrm{Reg}\Sp{M, K, \mathsf{SSR}_\Ho}&\leq\underbrace{\sum_{k=1}^{K}\sum_{h=i}^{H}\ogk\Sp{\sqrt{e^k_\Ho(\hsahk)}+\Sp{1+2C_1}|\uw^k_\Ho(\hsahk)|}}_{\text{(a)}}+\widetilde{O}\Sp{H\sqrt{T}}\nonumber\\
		&\qquad+\underbrace{\sum_{k=1}^{K}\sum_{h=1}^{H}\ogk\Sp{\Sp{1+C_1}\ob^k(\hsahk)+C_1\ub^k(\hsahk)}}_{\text{(b)}}+\widetilde{O}\Sp{H^5S^2A}.
	\end{align*}
	By referring to the definitions of $\uw^k_\Ho(\hsahk)$, $\gamma^k_\Ho(\hsahk)$, $\sigma^k_\Ho(\hsahk)$ and $\sqrt{e^k_\Ho(\hsahk)}$ in Definition \ref{def:V_under}, \ref{def:E^w_k} and equations \eqref{equ:def_ho_noise}, \eqref{def:confi_width_ho}, we know that
	$$|\uw^k_\Ho(\hsahk)|=\gamma^k_\Ho(\hsahk)=L_k\sigma^k_\Ho(\hsahk)=L_k\sqrt{e^k_\Ho(\hsahk)},$$
	where $L_k=\sqrt{\log(40k^4)}$. Therefore, we have
	\begin{align*}
		\text{(a)}&=\widetilde{O}\Sp{\sum_{k=1}^{K}\sum_{h=1}^{H}\ogk\Sp{H\sqrt{\frac{1}{n_k(\hsahk)+1}}+\frac{H}{n_k(\hsahk)+1}}}\\
		&\leq\widetilde{O}\Sp{H^{1.5}\sqrt{SAT}+H^2SA}.\tag{By Lemma \ref{lemma:sumsqrt1/n} and \ref{lemma:sum1/n}}
	\end{align*}
	Also, by referring to Lemma \ref{lemma:b_bound}, we can obtain $$\text{(b)}\leq \widetilde{O}\Sp{H\sqrt{SAT}+\Sp{HS}^{1.5}A}.$$
	Therefore, by combining the bounds for term (a) and (b), we can conclude under Hoeffding-type noise, it holds with probability at least $1-\delta$ that
	\begin{align*}
		\mathrm{Reg}\Sp{M, K, \mathsf{SSR}_\Ho}&\leq\widetilde{O}\Sp{H^{1.5}\sqrt{SAT}+H^2SA+(HS)^{1.5}A+H\sqrt{T}+H^5S^2A}\leq\widetilde{O}\Sp{H^{1.5}\sqrt{SAT}+H^5S^2A}.
	\end{align*}
	\end{comment}
	
\end{proof}

\begin{theorem}
	For Bernstein-type noise and $T\geq\widetilde{\Omega}\Sp{H^5S^2A}$, then for any MDP $M=\Sp{H, \S, \A, P, R, s_1}$, for any $\delta > 0$, with probability at least $1-\delta$, Algorithm \ref{algo:cub-rlsvi} satisfies
	$$\mathrm{Reg}(M, K, \mathsf{SSR}_\Be)\leq\widetilde{O}\Sp{H\sqrt{SAT}+H^4S^2A}.$$
	In particular, if we further have $T\geq\widetilde{\Omega}\Sp{H^6S^3A}$, it then holds that $\mathrm{Reg}(M, K, \mathsf{SSR}_\Be)\leq\widetilde{O}\Sp{H\sqrt{SAT}}$.
\end{theorem}
\begin{proof}
	Similar to the proof of Theorem \ref{theo:appen_theo1}, under Bernstein-type noise, it holds with probability at least $1-\frac{\delta}{2}$ that
	\begin{align*}
	    &\mathrm{Reg}\Sp{M, K, \mathsf{SSR}_\Be}\\
	   % =&\sum_{k=1}^K\ogk\delta^\pik_{1,k}(s_1^k)+\widetilde{O}(H^5S^2A)\\
	   % \leq& \sum_{k=1}^K \ogk C_1\Sp{\abs{\od_{h,k}^\pik(s_h^k)}+\abs{\upd_{h,k}^\pik(s_h^k)}+\abs{\ud_{h,k}^\pik(s_h^k)}}+\widetilde{O}(H^5S^2A)\\
	   % \leq& 3C_1e^{3C_1}\Sp{\sum_{k=1}^K\sum_{h=1}^H \ogk\sqrt{e_\ty^k(\hsahk)}+\sum_{k=1}^K\sum_{h=1}^H \ogk\gamma_\ty^k(\hsahk)}+\widetilde{O}(\sqrt{HT}+H^5S^2A)\\
	   \leq & \ogk 3C_1e^{3C_1}\sum_{k=1}^{K}\sum_{h=i}^H \Sp{\sqrt{e_\Be^k(\hsahk)}+ \gamma_\Be^k(\hsahk)}+\widetilde{O}\Sp{H^4S^2A+H\sqrt{T}}\\
	    \leq& \widetilde{O}\Sp{\sum_{k=1}^{K}\sum_{h=1}^{H-1}\Sp{\ogk U_{h,k}+\sqrt{\frac{\log(2HSAk^2)}{n_k(h,s,a)+1}}+\frac{H}{n_k(h,s,a)+1}}}+\widetilde{O}(H^4S^2A+H\sqrt{T})\\
	    =&\widetilde{O}(H\sqrt{SAT}+H^4S^2A),
	\end{align*}
	where the last step is from Lemma \ref{theo:variance_sum}.

\end{proof}

\section{Technical Lemmas}

\begin{lemma}[Bennet's Inequality]
	\label{lemma:bennet}
	Let $Z_1, \dots, Z_n$ be i.i.d. random variables bounded in $\Mp{0, 1}$. Then, for any $\delta >0$, we have
	$$\P\Sp{\abs{\frac{1}{n}\sum_{i=1}^{n}Z_i-\E\Mp{Z}}\geq\sqrt{\frac{2\mathrm{Var}\Sp{Z}\log\Sp{2/\delta}}{n}}+\frac{\log\Sp{2/\delta}}{n}}\leq\delta.$$
\end{lemma}

\begin{lemma}[from \cite{maurer2009empirical}]
	\label{lemma:empirical_bernstein}
	Let $Z_1, \dots, Z_n$ with $n\geq 2$ be i.i.d. random variables bounded in $\Mp{0, H}$. Define $\bar{Z}=\frac{1}{n}\sum_{i=1}^{n}Z_i$ and $\hat{V}_n=\frac{1}{n}\sum_{i=1}^{n}\Sp{Z_i-\bar{Z}}^2$. Then, for any $\delta>0$, we have
	$$\P\Sp{\abs{\E\Mp{\bar{Z}}-\sum_{i=1}^{n}Z_i}\geq\sqrt{\frac{2\hat{V}_n\log\Sp{2/\delta}}{n-1}}+\frac{7\log\Sp{2/\delta}}{3\Sp{n-1}}}\leq\delta.$$
\end{lemma}

\begin{lemma}
    \label{lemma:standard_var_bound}
    Let $X$ be arbitrary random variable bounded in $[a, b]$ for some $a, b\in\R$. Then, we have $\mathrm{Var}(X)\leq\frac{(b-a)^2}{4}$.
\end{lemma}
% \begin{proof}
%     Define $g\Sp{t}=\E\Mp{\Sp{X-t}^2}$. We can immediately notice that $g\Sp{t}$ is a convex function in $t$. Therefore, we can minimize it by taking derivative and setting it to 0, which gives $t^*=\E\Mp{X}$. Thus, the minimum of $g\Sp{t}$ is obtained at $t=\E\Mp{X}$ and we have $g\Sp{\E\Mp{X}}=\mathrm{Var}\Sp{X}$.
    
% 	Since $\mathrm{Var}\Sp{X}$ is the minimum of $g\Sp{t}$, we can have
% 	\begin{align*}
% 	\mathrm{Var}\Sp{X}\leq g\Sp{\frac{a+b}{2}}&=\E\Mp{\Sp{X-\frac{a+b}{2}}^2}\\
% 	&=\frac{1}{4}\E\Mp{\Sp{2X-\Sp{a+b}}^2}\\
% 	&=\frac{1}{4}\E\Mp{\Sp{\Sp{X-a}+\Sp{X-b}}^2}.
% 	\end{align*}
% 	Since $X$ takes values in $\Mp{a, b}$, we have $X-a\geq 0$ and $X-b\leq 0$. Thus, for any $x\in\Mp{a, b}$, we can have 
% 	$$\Sp{\Sp{x-a}+\Sp{x-b}}^2\leq\Sp{\Sp{x-a}-\Sp{x-b}}^2=\Sp{b-a}^2.$$
% 	As a result, we have
% 	$$\E\Mp{\Sp{\Sp{X-a}+\Sp{X-b}}^2}\leq\Sp{b-a}^2\implies\mathrm{Var}\Sp{X}\leq\frac{\Sp{b-a}^2}{4}.$$
% \end{proof}

\begin{lemma}
    \label{lemma:p_bound}
    For any $\delta>0$, with probability at least $1-\delta$, it holds for all $k, h, s, a, s'$ that
    $$\abs{\hat{P}^k_{h, s, a}(s')-P_{h, s, a}(s')}\leq\sqrt{\frac{4P_{h, s, a}(s')(1-P_{h, s, a}(s'))\log(2HS^2AK/\delta)}{n_k(h, s, a)+1}}+\frac{3\log(2HS^2AK/\delta)}{n_k(h, s, a)+1}.$$
\end{lemma}
\begin{proof}
    Let $\delta'=\frac{\delta}{HS^2AK}$ and fix $(k, h, s, a, s')$ such that $n_k(h, s, a)\geq 1$. Then, we have
    % \fontsize{9}{9}
    \begin{align*}
        &\P\Sp{\abs{\hat{P}^k_{h, s, a}(s')-P_{h, s, a}(s')}\geq\sqrt{\frac{4P_{h, s, a}(s')(1-P_{h, s, a}(s'))\log(2/\delta')}{n_k(h, s, a)+1}}+\frac{3\log(2/\delta')}{n_k(h, s, a)+1}}\\
        \leq& \P\left(\abs{\tilde{P}^k_{h, s, a}(s')-P_{h, s, a}(s')}\geq\sqrt{\frac{4P_{h, s, a}(s')(1-P_{h, s, a}(s'))\log(2/\delta')}{n_k(h, s, a)+1}}+\frac{3\log(2/\delta')-1}{n_k(h, s, a)+1}\right)\\
        \leq& \P\Sp{\abs{\tilde{P}^k_{h, s, a}(s')-P_{h, s, a}(s')}\geq \sqrt{\frac{4P_{h, s, a}(s')(1-P_{h, s, a}(s'))\log(2/\delta')}{n_k(h, s, a)+1}}+\frac{2\log(2/\delta')}{n_k(h, s, a)+1}}\\
        \leq& \P\Sp{\abs{\tilde{P}^k_{h, s, a}(s')-P_{h, s, a}(s')}\geq \sqrt{\frac{2P_{h, s, a}(s')(1-P_{h, s, a}(s'))\log(2/\delta')}{n_k(h, s, a)}}+\frac{\log(2/\delta')}{n_k(h, s, a)}}\tag{Since $n+1\leq 2n$ for $n\geq 1$}\\
        \leq& \delta'=\frac{\delta}{HS^2AK}.\tag{By Lemma \ref{lemma:bennet}, the Bennet's inequality}
    \end{align*}
    % \normalsize
    Then, the proof is complete by taking a union bound over all possible $(k, h, s, a, s')$.
\end{proof}
% {\color{red}
\section{Numeric Simulations}
\label{sec:simulation}
In this section, we empirically compare RLSVI \cite{russo2019worst}, UCBVI \cite{azar2017minimax} and our algorithm \algoname\ on the famous deep sea environment, which is a tabular environment frequently used to test an algorithm's ability to do efficient exploration \cite{osband2018randomized, osband2017deep, tan2020parameterized}.

Deep sea, as shown in Figure \ref{fig:deep_sea}, is a grid-like deterministic environment with $N\times N$ cell states, action space $\Bp{0, 1}$ and action mask $M_{ij}\sim\mathsf{Bernoulli}\Sp{0.5}$, $\Sp{i, j}\in\mathcal{S}$, whose values are sampled when initializing the environment. At each cell $\Sp{i, j}$. Action $M_{ij}$ represents going ``right'', which leads the agent to the lower right cell, and $1-M_{ij}$ represents going ``left'', which leads the agent to the lower left cell. An episode of this environment will end after $N$ steps. When going ``left'' or going ``right'' at the off-diagonal, the agent will receive 0 reward; when going ``right'' along the diagonal before reaching the lower right corner, the agent will receive negative reward $-\frac{0.01}{N}$. Finally, when reaching the lower right corner, depending on the environment initialization, the agent will either receive reward $+1$ or $-1$. In our experiment, we set this to $+1$, which results in an obvious optimal policy ``always going right'' with total reward 0.99 per episode.

\begin{wrapfigure}{r}{0.5\linewidth}
    \centering
	\includegraphics[scale=0.25]{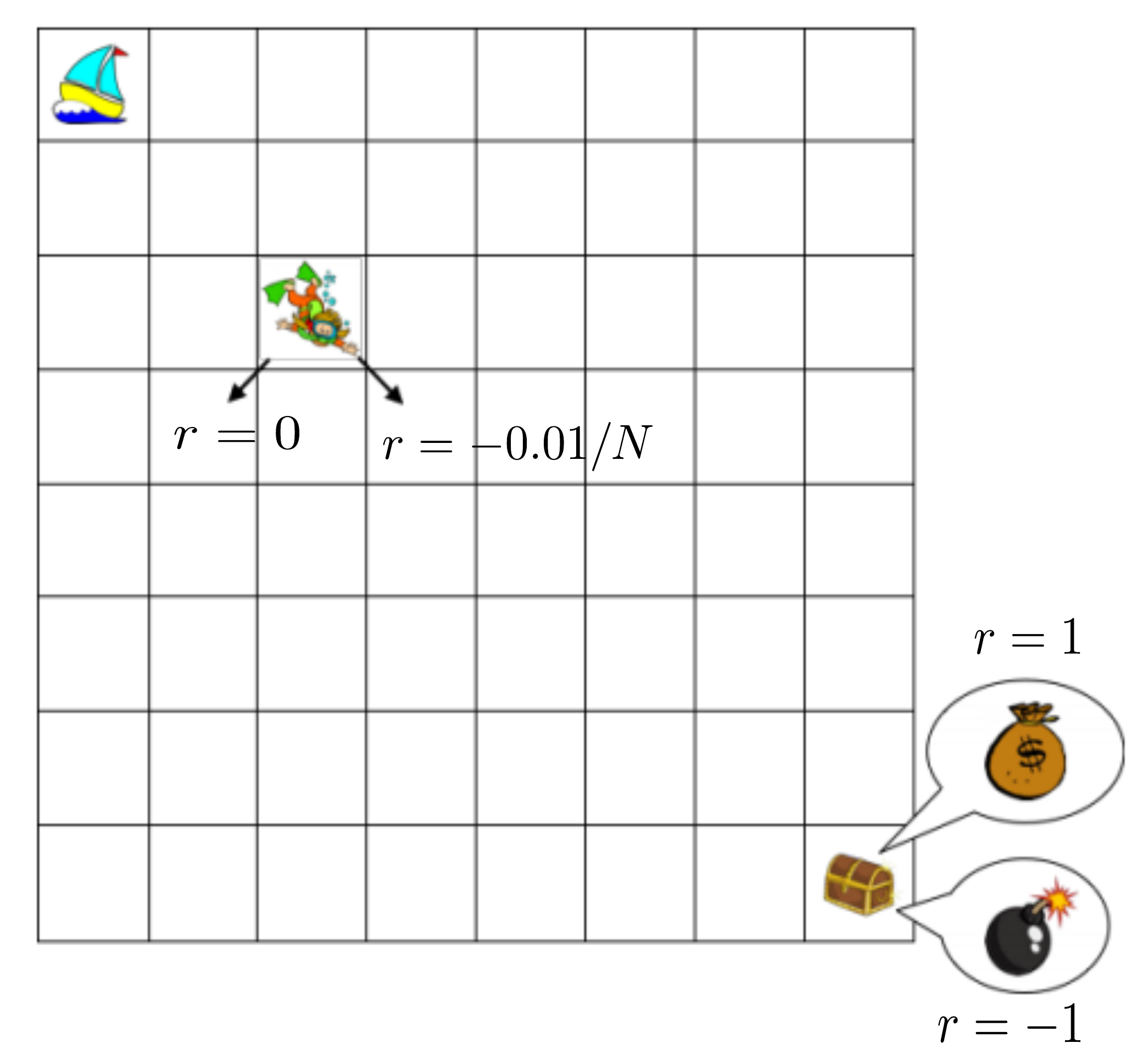}
	\caption{An example deep sea environment with $N=8$ \cite{osband2017deep}.}
	\vspace{-20pt}
	\label{fig:deep_sea}
\end{wrapfigure}

% \begin{figure}[ht]
% 	\centering
% 	\includegraphics[width=100mm]{figs/deep_sea.pdf}
% 	\caption{An example deep sea environment with $N=8$ \cite{osband2017deep}.}
% 	\label{fig:deep_sea}
% \end{figure}

The experiment results are shown in Figure \ref{fig:exp_results}.\footnote{Bonuses for all three algorithms are scaled down from the theoretical values by a factor of $7\times 10^4$ since without scaling, none of them can learn anything even in the deep sea with $N=5$.} From the plots, we can see that in both settings, \algoname\ performs significantly better than RLSVI as predicted by our theory. Specifically, because of the instability incurred by the independent random seeds and large perturbation magnitude, RLSVI almost never reaches the lower right corner in both settings and thus incurs linear regret. On the other hand, \algoname\ obtains a much lower sub-linear regret because it can explore consistently with the single random seed.

% in the small deep sea ($N=5$), the flat cumulative regret curve of \algoname\ shows that it can consistently perform the optimal policy as long as it reaches the lower right corner and obtains the reward. In comparison, because of the instability incurred by the independent random sees in RLSVI, it will still perform in a non-optimal way even after it has reached the lower right corner. Meanwhile, in the large deep sea ($N=20$), , our algorithm \algoname\ can still obtain a much lower sub-linear regret.

% Although in both plots, \algoname\ performs slightly worse than UCBVI, we want to emphasize that our theory mainly indicates an improvement over RLSVI. The worse performance compared to UCBVI may come from loose control in constants and logarithmic terms, which is not the focus of our paper.

Meanwhile, in both settings, \algoname\ performs comparably with the UCBVI, which is expected since both algorithms achieve the minimax lower bound and our analysis does not indicate that one is better than the other. 

\begin{figure}[ht]
    \centering
    \includegraphics[width=\linewidth]{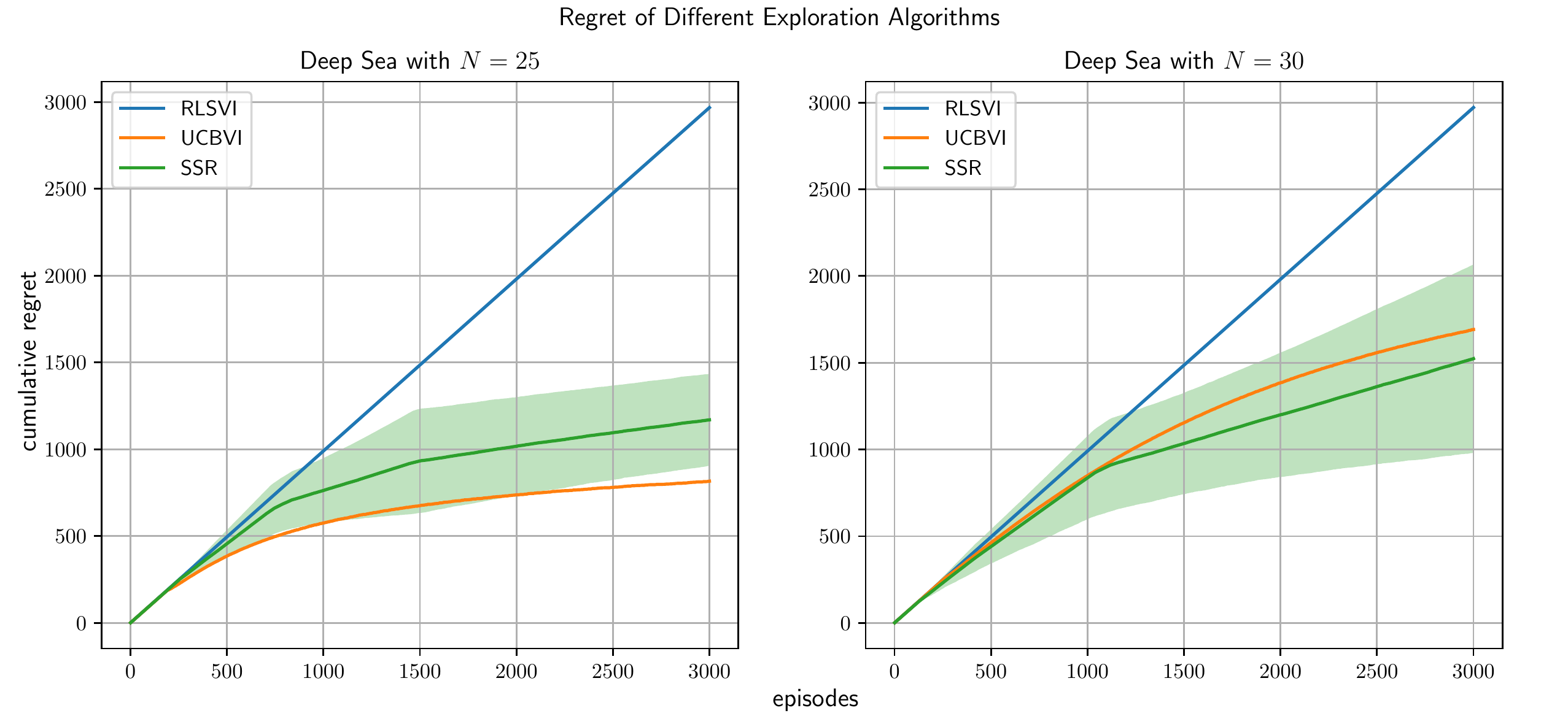}
    \caption{Empirical evaluation of RLSVI, UCBVI and \algoname\ in deep sea environments with $N=25$ and $N=30$. The results are averaged over 10 repeated trials and the shaded area represents the standard deviation. For simplicity, we use Hoeffding-type bonus for both UCBVI and \algoname.}
    \label{fig:exp_results}
\end{figure}
% }

% \begin{wrapfigure}{r}{0.5\linewidth}
%     \centering
%     \includegraphics[width=\linewidth]{figs/ablation.pdf}
%     \caption{Empirical evaluation of RLSVI and \algoname\ in deep sea environments with $N=25$, where both algorithms use the same noise magnitude.}
%     \label{fig:ablation}
% \end{wrapfigure}

Finally, we also do an ablation study to show that the better performance of \algoname\ over the RLSVI indeed comes from the single seed randomization instead of smaller noise magnitude. In particular, we run both algorithms in a deep sea environment with $N=25$ and apply the same noise magnitude, whose results are shown in Figure \ref{fig:ablation}. We can see that although using the same noise magnitude, \algoname\  still significantly outperforms RLSVI.

\begin{figure}[ht]
    \centering
    \includegraphics[width=85mm]{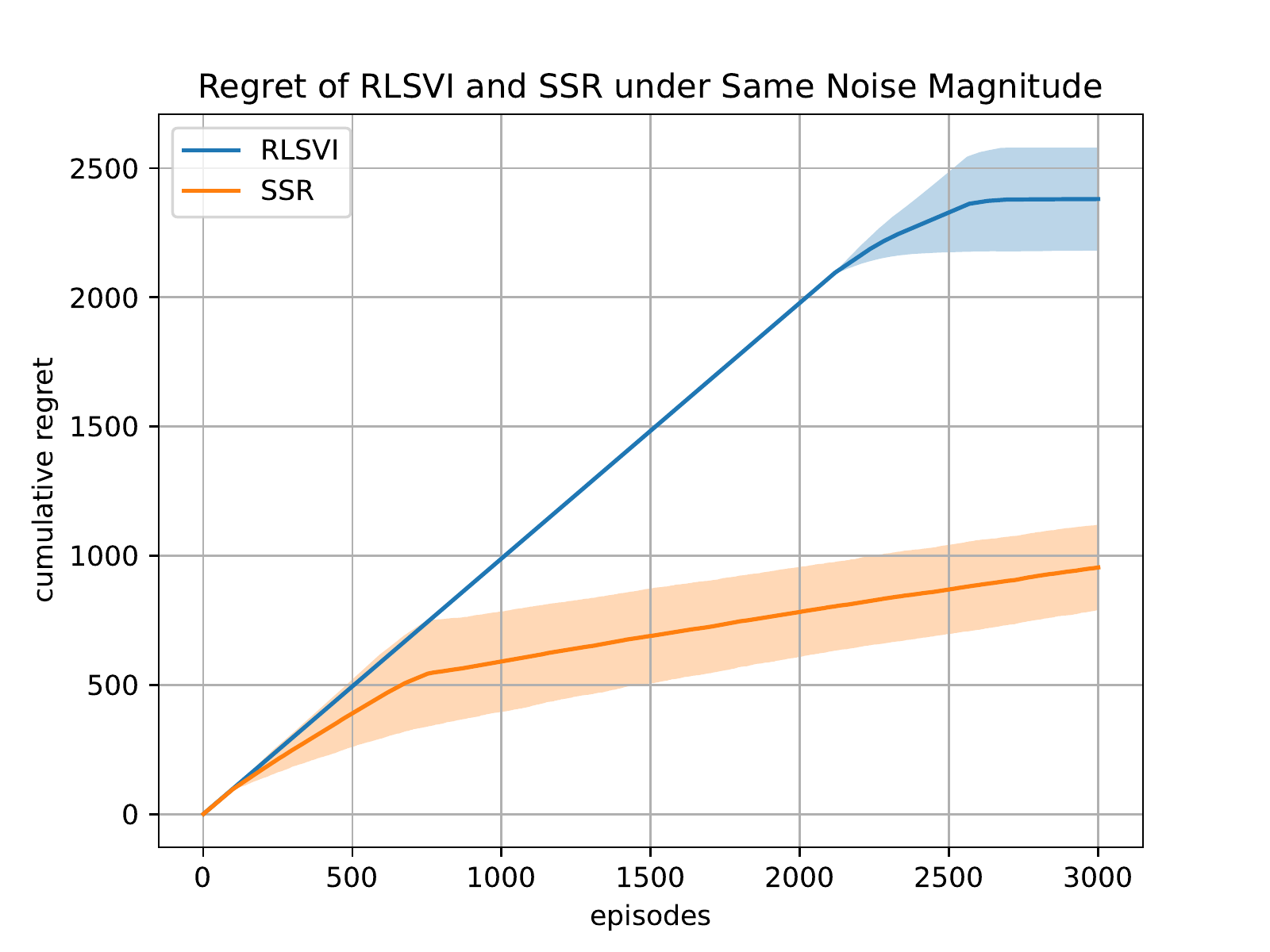}
    \caption{Empirical evaluation of RLSVI and \algoname\ in deep sea environments with $N=25$, where both algorithms use the same noise magnitude.}
    \label{fig:ablation}
\end{figure}

\end{document}